\documentclass{article}

\usepackage{amsmath,amsfonts,bm}

\def\eqref#1{equation~\ref{#1}}

\def\1{\bm{1}}

\def\mB{{\bm{B}}}

\def\mP{{\bm{P}}}

\DeclareMathAlphabet{\mathsfit}{\encodingdefault}{\sfdefault}{m}{sl}
\SetMathAlphabet{\mathsfit}{bold}{\encodingdefault}{\sfdefault}{bx}{n}

\def\gB{{\mathcal{B}}}

\def\gF{{\mathcal{F}}}
\def\gG{{\mathcal{G}}}

\def\gN{{\mathcal{N}}}

\def\gP{{\mathcal{P}}}

\def\gW{{\mathcal{W}}}

\def\sE{{\mathbb{E}}}

\def\sR{{\mathbb{R}}}

\DeclareMathOperator*{\argmin}{arg\,min}

\usepackage{microtype}
\usepackage{graphicx}
\usepackage{subfigure}
\usepackage{tabularx,booktabs} 

\usepackage{algorithm,algorithmic}

\usepackage{epsfig,amsmath,amssymb,amsfonts,amstext,amsthm,mathrsfs}
\usepackage{latexsym,graphics,epsf,epsfig,psfrag}
\usepackage{dsfont,color,epstopdf,fixmath}
\usepackage{enumitem}
\usepackage{dsfont}

\usepackage[utf8]{inputenc} 
\usepackage[T1]{fontenc}    

\usepackage{booktabs}       

\usepackage{hhline}

\usepackage{threeparttable}
\usepackage{booktabs, caption, makecell}

\usepackage{amsfonts}       
\usepackage{nicefrac}       
\usepackage{microtype}      
\usepackage{subfigure}
\usepackage{mathtools}
\usepackage{amssymb}
\usepackage{amsthm}

\usepackage[colorlinks=true,citecolor=blue]{hyperref}

\usepackage{multirow}

\newcommand{\mdsone}{\mathds{1}}

\newcommand{\msP}{\mathsf{P}}

\newcommand{\mcs}{\mathcal S}
\newcommand{\mca}{\mathcal A}

\newcommand{\ltwo}[1]{\left\|#1\right\|_2}
\newcommand{\lone}[1]{\left|#1\right|}

\newcommand{\linf}[1]{\left\|#1\right\|_\infty}

\newcommand{\lmupi}[1]{\left\|#1\right\|_{\mu_{\pi}}}
\newcommand{\ld}[1]{\left\|#1\right\|_{\mathcal{D}}}
\newcommand{\lmuW}[1]{\left\|#1\right\|_{\mu_{\pi_W}}}
\newcommand{\lmutauWt}[1]{\left\|#1\right\|_{\mu_{\pi_{\tau_{t} W_{t}}}}}
\newcommand{\lnutauWt}[1]{\left\|#1\right\|_{\nu_{\pi_{\tau_{t} W_{t}}}}}

\newtheorem{theorem}{Theorem}

\newtheorem{lemma}{Lemma}

\newtheorem{assumption}{Assumption}
\newtheorem{remark}{Remark}

\allowdisplaybreaks[4]

\usepackage{hyperref}
\usepackage{cleveref}
\usepackage{url}

\usepackage[accepted]{icml2021}

\icmltitlerunning{CRPO: A New Approach for Safe Reinforcement Learning with Convergence Guarantee}

\begin{document}

\twocolumn[
\icmltitle{CRPO: A New Approach for Safe Reinforcement Learning with Convergence Guarantee}




\begin{icmlauthorlist}
\icmlauthor{Tengyu Xu}{to}
\icmlauthor{Yingbin Lang}{to}
\icmlauthor{Guanghui Lan}{ed}
\end{icmlauthorlist}

\icmlaffiliation{to}{Department of Electrical and Computer Engineering, The Ohio State University, OH, United States}
\icmlaffiliation{ed}{Industrial and Systems Engineering, Georgia Institute of Technology, GA, United States}

\icmlcorrespondingauthor{Tengyu Xu}{xu.3260@osu.edu}

\icmlkeywords{Reinforcement Learning, Constrained Markov Decision Process, Global Convergence}

\vskip 0.3in
]



\printAffiliationsAndNotice{}  

\begin{abstract}
In safe reinforcement learning (SRL) problems, an agent explores the environment to maximize an expected total reward and meanwhile avoids violation of certain constraints on a number of expected total costs. In general, such SRL problems have nonconvex objective functions subject to multiple nonconvex constraints, and hence are very challenging to solve, particularly to provide a globally optimal policy. Many popular SRL algorithms adopt a primal-dual structure which utilizes the updating of dual variables for satisfying the constraints. In contrast, we propose a primal approach, called constraint-rectified policy optimization (CRPO), which updates the policy alternatingly between objective improvement and constraint satisfaction. CRPO provides a primal-type algorithmic framework to solve SRL problems, where each policy update can take any variant of policy optimization step. To demonstrate the theoretical performance of CRPO, we adopt natural policy gradient (NPG) for each policy update step and show that CRPO achieves an $\mathcal{O}(1/\sqrt{T})$ convergence rate to the global optimal policy in the constrained policy set and an $\mathcal{O}(1/\sqrt{T})$ error bound on constraint satisfaction. This is the first finite-time analysis of primal SRL algorithms with global optimality guarantee. Our empirical results demonstrate that CRPO can outperform the existing primal-dual baseline algorithms significantly.
\end{abstract}

\vspace{-0.5cm}
\section{Introduction}
\vspace{-0.2cm}
Reinforcement learning (RL) has achieved great success in solving complex sequential decision-making and control problems such as Go \cite{silver2017mastering}, StarCraft \cite{deepmind2019mastering} and recommendation system \cite{zheng2018drn}, etc. In these settings, the agent is allowed to explore the entire state and action space to maximize the expected total reward. However, in safe RL (SRL), in addition to maximizing the reward, an agent needs to satisfy certain constraints. Examples include self-driving cars \cite{fisac2018general}, cellular network \cite{julian2002qos}, and robot control \cite{levine2016end}. 
The global optimal policy in SRL is the one that maximizes the reward and at the same time satisfies the cost constraints. 

The current safe RL algorithms can be generally categorized into the {\bf primal} and {\bf primal-dual} approaches. The {\bf primal-dual} approaches \cite{tessler2018reward,ding2020provably,stooke2020responsive,yu2019convergent,achiam2017constrained,yang2019projection,altman1999constrained,borkar2005actor,bhatnagar2012online,liang2018accelerated,paternain2019safe}
are most commonly used, which convert the constrained problem into an unconstrained one by augmenting the objective with a sum of constraints weighted by their corresponding Lagrange multipliers (i.e., dual variables). Generally, primal-dual algorithms apply a certain policy optimization update such as policy gradient alternatively with a gradient descent type update for the dual variables.
Theoretically, \cite{tessler2018reward} has provided an asymptotic convergence analysis for primal-dual method and established a local convergence guarantee. \cite{paternain2019constrained} showed that the primal-dual method achieves zero duality gap.
Recently, \cite{ding2020provably} proposed a primal-dual type proximal policy optimization (PPO) and established the regret bound for linear constrained MDP. The convergence rate of primal-dual method based on a natural policy gradient algorithm was characterized in \cite{ding2020natural}.

The {\bf primal} type of approaches \cite{liu2019ipo,chow2018lyapunov,chow2019lyapunov,dalal2018safe} enforce constraints via various designs of the objective function or the update process without an introduction of dual variables. The {\bf primal} algorithms are much less studied than the primal-dual approach. Notably, \cite{liu2019ipo} developed an interior point method, which applies logarithmic barrier functions for SRL. \cite{chow2018lyapunov,chow2019lyapunov} leveraged Lyapunov functions to handle constraints. \cite{dalal2018safe} introduced a safety layer to the policy network to enforce constraints. None of the existing primal algorithms are shown to have provable convergence guarantee to a globally optimal feasible policy. 

Comparing between the primal-dual and primal approaches, the primal-dual approach can be sensitive to the initialization of Lagrange multipliers and the learning rate, and can thus incur extensive cost in hyperparameter tuning \cite{achiam2017constrained,chow2019lyapunov}. In contrast, the primal approach does not introduce additional dual variables to optimize and involves less hyperparamter tuning, and hence holds the potential to be much easier to implement than the primal-dual approach. However, the existing {\bf primal} algorithms are not yet popular in practice so far, because of no guaranteed global convergence and no strong demonstrations to have competing performance as the primal-dual algorithms. Thus, in order to take the advantage of the primal approach which is by nature easier to implement, we need to answer the following fundamental questions.
\begin{list}{$\rhd$}{\topsep=0.ex \leftmargin=0.2in \rightmargin=0.in \itemsep =0.0in}
\item Can we design a primal algorithm for SRL, and demonstrate that it achieves competing performance or outperforms the baseline primal-dual approach?
\item If so, can we establish global optimality guarantee and the finite-time convergence rate for the proposed primal algorithm?
\end{list}

In this paper, we will provide the affirmative answers to the above questions, thus establishing appealing advantages of the primal approach for SRL.

\subsection{Main Contributions}

{\bf A New Algorithm:} We propose a novel primal approach called {\bf Constraint-Rectified Policy Optimization (CRPO)} for SRL, where all updates are taken in the primal domain. 
CRPO applies {\em unconstrained} policy maximization update w.r.t.\ the reward on the one hand, and if any constraint is violated, momentarily rectifies the policy back to the constraint set along the descent direction of the violated constraint also by applying {\em unconstrained} policy minimization update w.r.t.\ the constraint function. From the implementation perspective, CRPO can be implemented as easy as unconstrained policy optimization algorithms. Without introduction of dual variables, it does not suffer from hyperparameter tuning of the learning rates to which the dual variables are sensitive, nor does it require initialization to be feasible. Further, CRPO involves only policy gradient descent for both objective and constraints, whereas the primal-dual approach typically requires {\em projected} gradient descent, where the projection causes higher complexity to implementation as well as hyperparameter tuning due to the projection thresholds.

To further explain the advantage of CRPO over the primal-dual approach, CRPO features {\bf immediate switches} between optimizing the objective and reducing the constraints whenever constraints are violated. However, the primal-dual approach can respond much slower because the control is based on dual variables. If a dual variable is nonzero, then the policy update will descend along the corresponding constraint function. As a result, even if a constraint is already satisfied, there can often be a significant delay for the dual variable to iteratively reduce to zero to release the constraint, which slows down the algorithm. Our experiments in \Cref{sec:experiment} validates such a performance advantage of CRPO over the primal-dual approach.




{\bf Theoretical Guarantee:} To provide the theoretical guarantee for CRPO, we adopt NPG as a representative policy optimizer and investigate the convergence of CRPO in two settings: tabular and function approximation, where in the function approximation setting the state space can be infinite. For both settings, we show that CRPO converges to a global optimum at a convergence rate of $\mathcal{O}(1/\sqrt{T})$. Furthermore, the constraint violation also converges to zero at a rate of $\mathcal{O}(1/\sqrt{T})$. 
To the best of our knowledge, we establish the first provably global optimality guarantee for a primal SRL algorithm of CRPO. 

To compare with the primal-dual approach in the function approximation setting, the value function gap of CRPO achieves the same convergence rate as the primal-dual approach, but the constraint violation of CRPO decays at a rate of $\mathcal{O}(1/\sqrt{T})$, which is much faster than the rate $\mathcal{O}(1/T^{\frac{1}{4}})$ of the primal-dual approach \cite{ding2020natural}.

Technically, our analysis has the following novel developments. (a) We develop a new technique to analyze a stochastic approximation (SA) that randomly and dynamically switches  between the target objectives of the reward and the constraint. Such an SA by nature is different from the analysis of a typical policy optimization algorithm, which has a fixed target objective to optimize. Our analysis constructs novel concentration events for capturing the impact of such a dynamic process on the update of the reward and cost functions in order to establish the high probability convergence guarantee. (b) We also develop new tools to handle multiple constraints, which is particularly non-trivial for our algorithm that involves stochastic selection of a constraint if multiple constraints are violated. 



\subsection{Related Work}
\textbf{Safe RL: }Algorithms based on primal-dual methods have been widely adopted for solving constrained RL problems, such as PDO \cite{chow2017risk}, RCPO \cite{tessler2018reward}, OPDOP \cite{ding2020provably} and CPPO \cite{stooke2020responsive}. Constrained policy optimization (CPO) \cite{achiam2017constrained} extends TRPO to handle constraints, and is later modified with a two-step projection method \cite{yang2019projection}. 
The effectiveness of primal-dual methods is justified in \cite{paternain2019constrained}, in which zero duality gap is guaranteed under certain assumptions. 
A recent work \cite{ding2020natural} established the convergence rate of the primal-dual method under Slater's condition assumption.
Other methods have also been proposed. For example, \cite{chow2018lyapunov,chow2019lyapunov} leveraged Lyapunov functions to handle constraints. \cite{yu2019convergent} proposed a constrained policy gradient algorithm with convergence guarantee by solving a sequence of sub-problems. \cite{dalal2018safe} proposed to add a safety layer to the policy network so that constraints can be satisfied at each state. \cite{liu2019ipo} developed an interior point method for safe RL, which augments the objective with logarithmic barrier functions. 
Our work proposes a CRPO algorithm, which can be implemented as easy as unconstrained policy optimization methods and has global optimality guarantee under general constrained MDP. Our result is the first convergence rate characterization of primal-type algorithms for SRL.

\textbf{Finite-Time Analysis of Policy Optimization: }The finite-time analysis of various policy optimization algorithms under unconstrained MDPs have been well studied. The convergence rate of policy gradient (PG) and actor-critic (AC) algorithms have been established in \cite{shen2019hessian,papini2017adaptive,papini2018stochastic,xu2019sample,xu2019improved,xiong2020non,zhang2019global} and \cite{xu2020improving,wang2019neural,yang2019provably,kumar2019sample,qiu2019finite}, respectively, in which PG or AC algorithm is shown to converge to a local optimal. In some special settings such as tabular and LQR, PG and AC can be shown to convergence to the global optimal \cite{agarwal2019optimality,yang2019provably,fazel2018global,malik2018derivative,tu2018gap,bhandari2019global,bhandari2020note}. Algorithms such as NPG, NAC, TRPO and PPO explore the second order information, and achieve great success in practice. These algorithms have been shown to converge to a global optimum in various settings, where the convergence rate has been established in \cite{agarwal2019optimality,shani2019adaptive,liu2019neural,wang2019neural,cen2020fast,xu2020non}. 

\section{Problem Formulation and Preliminaries}
\subsection{Markov Decision Process}

A discounted Markov decision process (MDP) is a tuple $(\mcs, \mca,c_0,\msP,\xi,\gamma)$, where $\mcs$ and $\mca$ are state and action spaces; $c_0: \mcs \times \mca \times \mcs \rightarrow \sR$ is the reward function; $\msP: \mcs \times \mca \times \mcs \rightarrow [0,1]$ is the transition kernel, with $\msP(s^\prime |s,a)$ denoting the probability of transitioning to state $s^\prime$ from previous state $s$ given action $a$; $\xi: \mcs \rightarrow [0,1]$ is the initial state distribution; and $\gamma\in(0,1)$ is the discount factor. A policy $\pi: \mcs \rightarrow \gP(\mca)$ is a mapping from the state space to the space of probability distributions over the actions, with $\pi(\cdot|s)$ denoting the probability of selecting action $a$ in state $s$. When the associated Markov chain $\msP(s^\prime|s)=\sum_\mca P(s^\prime|s,a)\pi(a|s)$ is ergodic, we denote $\mu_\pi$ as the stationary distribution of this MDP, i.e. $\int_{\mcs}\msP(s^\prime|s)\mu_\pi(ds)=\mu_\pi(s^\prime)$.
Moreover, we define the visitation measure induced by the police $\pi$ as $\nu_\pi(s,a)=(1-\gamma)\sum_{t=0}^{\infty}\gamma^t \msP(s_t=s,a_t=a)$.

For a given policy $\pi$, we define the state value function as
$V^0_\pi(s)=\sE[\sum_{t=0}^{\infty}\gamma^t c_0(s_t,a_t, s_{t+1})|s_0=s,\pi]$, the state-action value function as $Q^0_\pi(s,a)=\sE[\sum_{t=0}^{\infty}\gamma^t c_0(s_t,a_t, s_{t+1})|s_0=s,a_0=a,\pi]$, and the advantage function as $A^0_\pi(s,a)=Q^0_\pi(s,a)-V^0_\pi(s)$. In reinforcement learning, we aim to find an optimal policy that maximizes the expected total reward function defined as $J_0(\pi)=\sE[\sum_{t=0}^{\infty}\gamma^t c_0(s_t,a_t,s_{t+1})]=\sE_\xi[V^0_\pi(s)]=\sE_{\xi\cdot\pi}[Q^0_\pi(s,a)]$.
\subsection{Safe Reinforcement Learning (SRL) Problem}

The SRL problem is formulated as an MDP with additional constraints that restrict the set of allowable policies. Specifically, when taking action at some state, the agent can incur a number of costs denoted by $c_1,\cdots,c_p$, where each cost function $c_i: \mcs\times\mca\times\mcs\rightarrow \sR$ maps a tuple $(s,a,s^\prime)$ to a cost value. Let $J_i(\pi)$ denotes the expected total cost function with respect to $c_i$ as $J_i(\pi)=\sE[\sum_{t=0}^{\infty}\gamma^t c_i(s_t,a_t,s_{t+1})]$. The goal of the agent in SRL is to solve the following constrained problem
\begin{flalign}
	\max\limits_{\pi}J_0(\pi), \,\,\, \text{s.t.}\,\,\, J_i(\pi)\leq d_i,\,\, \forall i=1,\cdots,p,\label{eq: 13}
\end{flalign}
where $d_i$ is a fixed limit for the $i$-th constraint. We denote the set of feasible policies as $\Omega_C\equiv \{\pi: \forall i, J_i(\pi)\leq d_i\}$, and define the optimal policy for SRL as $\pi^*=\argmin_{\pi\in\Omega_C} J_0(\pi)$. For each cost $c_i$, we define its corresponding state value function $V^i_\pi$, state-action value function $Q^i_\pi$, and advantage function $A^i_\pi$ analogously to $V^0_\pi$, $Q^0_\pi$, and $A^0_\pi$, with $c_i$ replacing $c_0$, respectively.
\subsection{Policy Parameterization and Policy Gradient}
In practice, a convenient way to solve the problem \cref{eq: 13} is to parameterize the policy and then optimize the policy over the parameter space. Let $\{\pi_w: \mcs\rightarrow \gP(\mca)|w\in \gW\}$ be a parameterized policy class, where $\gW$ is the parameter space. Then, the problem in \cref{eq: 13} becomes
\begin{flalign}
	\max\limits_{w\in\gW}J_0(\pi_w), \,\,\,\text{s.t.}\,\,\, J_i(\pi_w)\leq d_i,\, \forall i=1,\cdots,p.\label{eq: 1}
\end{flalign}
The policy gradient of the function $J_i(\pi_w)$ has been derived by \cite{sutton2000policy} as $\nabla J_i(\pi_w)=\sE[Q^i_{\pi_w}(s,a)\phi_w(s,a)]$, where $\phi_w(s,a)\coloneqq \nabla_w\log\pi_w(a|s)$ is the score function. Furthermore, the natural policy gradient was defined by \cite{kakade2002natural} as $\Delta_i(w)=F(w)^\dagger\nabla J_i(\pi_w)$, where $F(w)$ is the Fisher information matrix defined as $F(w)=\sE_{\nu_{\pi_w}}[\phi_w(s,a)\phi_w(s,a)^\top]$.

\section{Constraint-Rectified Policy Optimization (CRPO) Algorithm}
\begin{algorithm}[tb]
	\null
	\caption{Constraint-Rectified Policy Optimization (CRPO)}
	\label{algorithm_cpg}
	\small
	\begin{algorithmic}[1]
		\STATE {\bfseries Initialize:} initial parameter $w_0$, empty set $\gN_0$
		\FOR{$t=0,\cdots,T-1$}
		\STATE Policy evaluation under $\pi_{w_t}$: $\bar{Q}^i_t(s,a)\approx Q^i_{\pi_{w_t}}(s,a)$
		\STATE Sample $(s_j,a_j)\in \gB_t\sim \xi\cdot\pi_{w_t}$, compute constrain estimation $\bar{J}_{i,\gB_t}=\sum_{j\in \gB_t}\rho_{j,t}\bar{Q}^i_t(s_j,a_j)$ for $i=0,\cdots,p$,  ($\rho_{j,t}$ is the weight)
		\IF{$\bar{J}_{i,\gB_t}\leq d_i+\eta $ for all $i=1,\cdots,p$,}
		\STATE Add $w_t$ into set $\gN_0$
        \STATE Take one-step policy update towards maximize $J_0(w_t)$: $w_t\rightarrow w_{t+1}$
		\ELSE
		\STATE Choose any $i_t\in\{1,\cdots,p\}$ such that $\bar{J}_{i_t,\gB_t}> d_{i_t}+\eta$
        \STATE Take one-step policy update towards minimize $J_{i_t}(w_t)$: $w_t\rightarrow w_{t+1}$
		\ENDIF
		\ENDFOR
		\STATE {\bfseries Output:} $w_{\text{out}}$ uniformly chosen from $\gN_0$
	\end{algorithmic}
\end{algorithm}
In this section, we propose the CRPO approach (see \Cref{algorithm_cpg}) for solving the SRL problem in \cref{eq: 1}. The idea of CRPO lies in updating the policy to maximize the unconstrained objective function $J_0(\pi_{w_t})$ of the reward, alternatingly with rectifying the policy to reduce a constraint function $J_i(\pi_{w_t})$ $(i\geq1)$ (along the descent direction of this constraint) if it is violated. 
Each iteration of CRPO consists of the following three steps.

\textbf{Policy Evaluation:} At the beginning of each iteration, we estimate the state-action value function $\bar{Q}^i_{\pi_{t}}(s,a)\approx{Q}^i_{\pi_{w_t}}(s,a)$ ($i=\{0,\cdots,p\}$) for both reward and costs under current policy $\pi_{w_t}$. 

\textbf{Constraint Estimation:} After obtaining $\bar{Q}^i_{\pi_{t}}$, the constraint function ${J}_i(w_t)=\sE_{\xi\cdot\pi_{w_t}}[Q^i_{w_t}(s,a)]$ can then be approximated via a weighted sum of approximated state-action value function: $\bar{J}_{i,\gB_t}=\sum_{j\in \gB_t}\rho_{j,t}\bar{Q}^i_t(s_j,a_j)$. Note this step does not take additional sampling cost, as the generation of samples $(s_j,a_j)\in \gB_t$ from distribution $\xi\cdot\pi_{w_t}$ does not require the agent to interact with the environment. 

\textbf{Policy Optimization:} We then check whether there exists an $i_t\in\{1,\cdots,p\}$ such that the approximated constraint $\bar{J}_{i_t,\gB_t}$ violates the condition $\bar{J}_{i_t,\gB_t}\leq d_i+\eta$, where $\eta$ is the tolerance. If so, we take \textbf{one-step} update of the policy towards minimizing the corresponding constraint function $J_{i_t}(\pi_{w_t})$ to enforce the constraint. If multiple constraints are violated, we can choose to minimize any one of them. 
If all constraints are satisfied, we take \textbf{one-step} update of the policy towards maximizing the objective function $J_{0}(\pi_{w_t})$.
To apply CRPO in practice, we can use any policy optimization update such as natural policy gradient (NPG) \cite{kakade2002natural}, trust region policy optimization (TRPO) \cite{schulman2015trust}, proximal policy optimization (PPO) \cite{schulman2017proximal}, ACKTR \cite{wu2017scalable}, DDPG \cite{lillicrap2015continuous} and SAC \cite{haarnoja2018soft}, etc, in the policy optimization step (line 7 and line 10).

The advantage of CRPO over the primal-dual approach can be readily seen from its design. CRPO features {\bf immediate} switches between optimizing the objective and reducing the constraints whenever they are violated. However, the primal-dual approach can respond much slower because the control is based on dual variables. If a dual variable is nonzero, then the policy update will descend along the corresponding constraint function. As a result, even if a constraint is already satisfied, there can still be a delay (sometimes a significant delay) for the dual variable to iteratively reduce to zero to release the constraint, which yields unnecessary sampling cost and slows down the algorithm. Our experiments in \Cref{sec:experiment} validates such a performance advantage of CRPO over the primal-dual approach.

From the implementation perspective, CRPO can be implemented as easy as unconstrained policy optimization such as {\em unconstrained} policy gradient algorithms, whereas the primal-dual approach typically requires the {\em projected} gradient descent to update the dual variables, which is more complex to implement. Further, without introduction of the dual variables, CRPO does not suffer from hyperparameter tuning of the learning rates and projection threshold of the dual variables, whereas the primal-dual approach can be very sensitive to these hyperparamters. Nor does CRPO require initialization to be feasible, whereas the primal-dual approach can suffer significantly from bad initialization. We also empirically verify that the performance of CRPO is robust to the value of $\eta$ over a wide range, which does not cause additional tuning effort compared to unconstrained algorithms. More discussions can be referred to \Cref{sec:experiment}.

CRPO algorithm is inspired by, yet very different from the cooperative stochastic approximation (CSA) method \cite{lan2016algorithms} in optimization literature. First, CSA is designed for convex optimization subject to convex constraint, and is not readily capable of handling the more challenging SRL problems \cref{eq: 1}, which are nonconvex optimization subject to nonconvex constraints. Second, CSA is designed to handle only a single constraint, whereas CRPO can handle multiple constraints with guaranteed constraint satisfaction and global optimality. Thus, the finite-time analysis for CSA and CRPO feature different approaches due to the aforementioned differences in their designs.

\section{Convergence Analysis of CRPO}

In this section, we take NPG as a representative optimizer in CRPO, and establish the global convergence rate of CRPO in both the tabular and function approximation settings. Note that TRPO and ACKTR update can be viewed as the NPG approach with adaptive stepsize. Thus, the convergence we establish for NPG implies similar results for CRPO that takes TRPO or ACKTR as the optimizer.
\subsection{Tabular Setting}
In the tabular setting, we consider the softmax parameterization. For any $w\in \sR^{\lone{\mcs}\times\lone{\mca}}$, the corresponding softmax policy $\pi_w$ is defined as
{\small \begin{flalign}
	\pi_w(a|s)\coloneqq \frac{\exp(w(s,a))}{\sum_{a^\prime \in \mca}\exp(w(s,a^\prime))},\quad\forall(s,a)\in \mcs\times\mca.\label{eq: 2}
\end{flalign}}
Clearly, the policy class defined in \cref{eq: 2} is complete, as any stochastic policy in the tabular setting can be represented in this class.

\textbf{Policy Evaluation:} To perform the policy evaluation in \Cref{algorithm_cpg} (line 3), we adopt the temporal difference (TD) learning, in which a vector $\theta^i\in \sR^{\lone{\mcs}\times\lone{\mca}}$ is used to estimate the state-action value function $Q^i_{\pi_w}$ for all $i=0,\cdots,p$. Specifically, each iteration of TD learning takes the form of
\begin{align}
\theta^i_{k+1}&(s,a) =\theta^i_{k}(s,a) \nonumber \\
&+ \beta_k[c_i(s,a,s^\prime) + \gamma \theta^i_{k}(s^\prime,a^\prime) - \theta^i_{k}(s,a)],\label{eq: 3}
\end{align}
where $s\sim \mu_{\pi_w}$, $a\sim \pi_w(\cdot|s)$, $s^\prime\sim \msP(\cdot|s,a)$, $a^\prime\sim\pi_w(\cdot|s^\prime)$, and $\beta_k$ is the learning rate. In line 3 of \Cref{algorithm_cpg}, we perform the TD update in \cref{eq: 3} for $K_{\text{in}}$ iterations. It has been shown in \cite{sutton1988learning,bhandari2018finite,dalal2018finite} that the iteration in \cref{eq: 3} of TD learning converges to a fixed point $\theta^i_*(\pi_w)\in \sR^{\lone{\mcs}\times\lone{\mca}}$, where each component of the fixed point is the corresponding state-action value: $\theta^i_*(\pi_w)(s,a)=Q^i_{\pi_w}(s,a)$. 
After performing $K_{\text{in}}$ iterations of TD learning as \cref{eq: 3}, we let $\bar{Q}^i_t(s,a)=\theta^i_{K_{\text{in}}}(s,a)$ for all $(s,a)\in\mcs\times\mca$ and all $i=\{0,\cdots,p\}$. 

\textbf{Constraint Estimation:} In the tabular setting, we let the sample set $\gB_t$ include all state-action pairs, i.e., $\gB_t=\mcs\times\mca$, and the weight factor be $\rho_{j,t}=\xi(s_j)\pi_{w_t}(a_j|s_j)$ for all $t=0,\cdots,T-1$. Then, the estimation error of the constraints can be upper bounded as $|\bar{J}_i(\theta^i_t) - J_i(w_t)|=|\sE[\bar{Q}^i_t(s,a)]-\sE[Q^i_{\pi_{w_t}}(s,a)]|\leq ||\bar{Q}^i(\theta^i_t) - Q^i_{\pi_w}||^2$.
Thus, our approximation of constraints is accurate when the approximated value function $\bar{Q}^i_t(s,a)$ is accurate.

\textbf{Policy Optimization:} In the tabular setting, it can be checked that the natural policy gradient of $J_i(\pi_w)$ is $\Delta_i(w)_{s,a}=(1-\gamma)^{-1}Q^i_{\pi_w}(s,a)$ (see \Cref{app: thm1}). Once we obtain an approximation $\bar{Q}^i_t(s,a)\approx Q^i_{\pi_w}(s,a)$, we can use it to update the policy in the upcoming policy optimization step:
\begin{flalign}\label{eq: 11}
& w_{t+1} = w_t + \alpha \bar{\Delta}_t,\,\,\text{(line 7)}\nonumber \\
 \text{or}\quad & w_{t+1} =w_{t} - \alpha \bar{\Delta}_t \,\,(\text{line 10}),
\end{flalign}
 where $\alpha>0$ is the stepsize and $ \bar{\Delta}_t(s,a)=(1-\gamma)^{-1}\bar{Q}^{0}_t(s,a)$ (line 7) or $(1-\gamma)^{-1}\bar{Q}^{i_t}_t(s,a)$ (line 10). 

Our {\bf main technical challenge} lies in the analysis of policy optimization, which runs as a stochastic approximation (SA) process with {\bf random and dynamical switches} between optimization objectives of the reward and cost targets. Moreover, since critics estimate the constraints and help actor to estimate the policy update, the interaction error between actor and critics affects how the algorithm switches between objective and constraints. The typical analysis technique for NPG \cite{agarwal2019optimality} is not applicable here, because NPG has a fixed objective to optimize, and its analysis technique does not capture the overall convergence performance of an SA with dynamically switching optimization objective. Furthermore, the updates with respect to the constraint functions involve the stochastic selection of a constraint if multiple constraints are violated, which further complicates the random events to analyze. To handle these issues, we develop a {\bf novel analysis approach}, in which we focus on the event in which critic returns almost accurate value function estimation. Such an event greatly facilitates us to capture how CRPO switches between objective and multiple constraints and establish the convergence rate.

The following theorem characterizes the convergence rate of CRPO in terms of the objective function and constraint error bound.
\begin{theorem}\label{thm1}
	Consider \Cref{algorithm_cpg} in the tabular setting with softmax policy parameterization defined in \cref{eq: 2} and any initialization $w_0\in \sR^{\lone{\mcs}\times\lone{\mca}}$. Suppose the policy evaluation update in \cref{eq: 3} takes $K_{\text{in}}=\Theta(T^{1/\sigma}(1-\gamma)^{-2/\sigma}\log^{2/\sigma}(T^{1+2/\sigma}/\delta))$ iterations. Let the tolerance $\eta = \Theta(\sqrt{\lone{\mcs}\lone{\mca}}/((1-\gamma)^{1.5}\sqrt{T}))$ and perform the NPG update defined in \cref{eq: 11} with $\alpha = (1-\gamma)^{1.5}/\sqrt{\lone{\mcs}\lone{\mca}T}$. Then, with probability at least $1-\delta$, we have
	\begin{flalign*}
		&J_0(\pi^*)-\sE[J_0(w_{\text{out}})]\leq\Theta\left( \frac{\sqrt{\lone{\mcs}\lone{\mca}}}{(1-\gamma)^{1.5}\sqrt{T}} \right), \\
		&\sE[J_i(w_{\text{out}})] - d_i \leq\Theta\left( \frac{\sqrt{\lone{\mcs}\lone{\mca}}}{(1-\gamma)^{1.5}\sqrt{T}} \right)
	\end{flalign*}
    for all $i=\{1,\cdots,p\}$, where the expectation is taken with respect to selecting $w_{\text{out}}$ from $\gN_0$.
\end{theorem}
As shown in \Cref{thm1}, starting from an arbitrary initialization, CRPO algorithm is guaranteed to converge to the globally optimal policy $\pi^*$ in the feasible set $\Omega_C$ at a sublinear rate $\mathcal{O}(1/\sqrt{T})$, and the constraint violation of the output policy also converges to zero also at a sublinear rate $\mathcal{O}(1/\sqrt{T})$. 
Thus, to attain a $w_{\text{out}}$ that satisfies $J_0(\pi^*)-\sE[J_0(w_{\text{out}})]\leq \epsilon$ and $\sE[J_i(w_{\text{out}})] - d_i\leq\epsilon$ for all $1\leq i\leq p$, CRPO needs at most $T=\mathcal{O}(\epsilon^{-2})$ iterations, with each policy evaluation step consists of approximately $K_{\text{in}}=\mathcal{O}(T)$ iterations when $\sigma$ is close to 1.
\Cref{thm1} is the first global convergence for a primal-type algorithm even under the nonconcave objective with nonconcave constraints. 
\begin{proof}[Outline of Proof Idea]
We briefly explain the idea of the proof of \Cref{thm1}, and the detailed proof can be referred to \Cref{app: thm1}. The key challenge here is to analyze an SA process that randomly and dynamically switches between the target objectives of the reward and the constraint. To this end, we construct novel concentration events for capturing the impact of such a dynamic process on the update of the reward and cost functions in order to establish the high probability convergence guarantee. 

More specifically, we focus on the event in which all policy evaluation step returns an estimation with high accuracy. Then we show that under the parameter setting specified in \Cref{thm1}, either the size of the approximated feasible policy set $\gN_0$ is large, or the average policies in the set $\gN_0$ is at least as good as $\pi^*$. In the first case we have enough candidate policies in the set $\gN_0$, which guarantees the convergence of CRPO within the set $\gN_0$.
In the second case we can directly conclude that $J(w_{\text{out}})\geq J(\pi^*)$. To establish the convergence rate of the constraint violation, note that $w_{\text{out}}$ is selected from the set $\gN_0$, and thus the violation cost is not worse than the summation of constraint estimation error and the tolerance. 
\end{proof}

\subsection{Function Approximation Setting}
In the function approximation setting, we parameterize the policy by a two-layer neural network together with the softmax policy. We assign a feature vector $\psi(s,a)\in\sR^d$ with $d\geq 2$ for each state-action pair $(s,a)$. Without loss of generality, we assume that $\ltwo{\psi(s,a)}\leq 1$ for all $(s,a)\in \mcs\times\mca$. A two-layer neural network $f((s,a);W,b)$ with input $\psi(s,a)$ and width $m$ takes the form of
\begin{flalign}
	f((s,a);W,b)=\frac{1}{\sqrt{m}}\sum_{r=1}^{m}b_r\cdot\text{ReLU}(W_r^\top \psi(s,a)),\label{eq: 4}
\end{flalign}
for any $(s,a)\in \mcs\times\mca$, where $\text{ReLU}(x)=\mdsone(x>0)\cdot x$, $b=[b_1,\cdots,b_m]^\top\in \sR^m$, and $W=[W^\top_1,\cdots,W^\top_m]^\top\in \sR^{md}$ are the parameters. When training the two-layer neural network, we initialize the parameter via $[W_0]_r\sim D_w$ and $b_r\sim \text{Unif}[-1,1]$ independently, where $D_w$ is a distribution that satisfies $d_1\leq\ltwo{[W_0]_r}\leq d_2$ (where $d_1$ and $d_2$ are positive constants), for all $[W_0]_r$ in the support of $D_w$. During training, we only update $W$ and keep $b$ fixed, which is widely adopted in the convergence analysis of neural networks \cite{cai2019neural,du2018gradient}. For notational simplicity, we write $f((s,a);W,b)$ as $f((s,a);W)$ in the sequel. Using the neural network in \cref{eq: 4}, we define the softmax policy
\begin{flalign}
	\pi^\tau_W(a|s)\coloneqq \frac{\exp(\tau\cdot f((s,a);W))}{\sum_{a^\prime \mca}\exp(\tau\cdot f((s,a^\prime);W))},\label{eq: 5}
\end{flalign}
for all $(s,a)\in \mcs\times\mca$, where $\tau$ is the temperature parameter, and it can be verified that $\pi^\tau_W(a|s)=\pi_{\tau W}(a|s)$. We define the feature mapping $\phi_W(s,a)=[\phi^1_W(s,a)^\top,\cdots,\phi^m_W(s,a)^\top]^\top$: $\sR^d\rightarrow\sR^{md}$ as
\begin{flalign*}
	\phi^r_W(s,a)^\top=\frac{b_r}{\sqrt{m}}\mdsone(W_r^\top\psi(s,a)>0)\cdot\psi(s,a), 
\end{flalign*}
for all $(s,a)\in\mcs\times\mca$ and for all $ r\in\{1,\cdots,m\}$.

\textbf{Policy Evaluation:} To estimate the state-action value function in \Cref{algorithm_cpg} (line 3), we adopt another neural network $f((s,a);\theta^i)$ as an approximator, where $f((s,a);\theta^i)$ has the same structure as $f((s,a);W)$, with $W$ replaced by $\theta\in\sR^{md}$ in \cref{eq: 5}. To perform the policy evaluation step, we adopt the TD learning with neural network parametrization, which has also been used for the policy evaluation step in \cite{cai2019neural,wang2019neural,zhang2020generative}. Specifically, we choose the same initialization as the policy neural work, i.e., $\theta^i_0=W_0$, and perform the TD iteration as
\begin{flalign}
	\theta^i_{k+1/2}=&\theta^i_k + \beta(c_i(s,a,s^\prime) + \gamma f((s^\prime,a^\prime);\theta^i_k) \nonumber \\
	&- f((s,a);\theta^i_k))\nabla_\theta f((s,a);\theta^i_k),\label{eq: 6}\\
	\theta^i_{k+1}=&\argmin_{\theta\in\mB}\ltwo{\theta - \theta^i_{k+1/2}}\label{eq: 7},
\end{flalign}
where $s\sim \mu_{\pi_{W}}$, $a\sim \pi_{W}(\cdot|s)$, $s^\prime\sim \msP(\cdot|s,a)$, $a^\prime\sim\pi_{ W}(\cdot|s^\prime)$, $\beta$ is the learning rate, and $\mB$ is a compact space defined as $\mB=\{ \theta\in\sR^{md}: \ltwo{\theta-\theta^i_0}\leq R\}$. For simplicity, we denote the state-action pair as $x=(s,a)$ and $x^\prime=(s^\prime,a^\prime)$ in the sequel. We define the temporal difference error as $\delta_k(x,x^\prime,\theta^i_k)=f(x^\prime_k,\theta^i_k)-\gamma f(x_k,\theta^i_k) - c_i(x_k,x^\prime_k)$, stochastic semi-gradient as $g_k(\theta^i_k)=\delta_k(x_k,x^\prime_k.\theta^i_k)\nabla_\theta f(x_k,\theta^i_k)$, and full semi-gradient as $\bar{g}_k(\theta^i_k)=\sE_{\mu_{\pi_W}}[\delta_k(x,x^\prime,\theta^i_k)\nabla_\theta f(x,\theta^i_k)]$.
We then describe the following regularity conditions on the stationary distribution $\mu_{\pi_{ W}}$, state-action value function $Q^i_{\pi_W}$, and variance, which have been adopted widely in the analysis of TD learning with function approximation and stochastic approximation (SA) \cite{cai2019neural,wang2019neural,zhang2020generative,fu2020single}.
\begin{assumption}\label{ass1}
	There exists a constant $C_0>0$ such that for any $\tau\geq0$, $x\in\sR^d$ with $\ltwo{x}=1$ and $\pi_{W}$, it holds that $\mP\left(\lone{x^\top\psi(s,a)} \leq \tau\right)\leq C_0\cdot\tau$, where $(s,a)\sim\mu_{\pi_{W}}$.
\end{assumption}
\begin{assumption}\label{ass2}
We define the following function class:
\begin{flalign*}
	\gF_{R,\infty}&=\big\{ f((s,a);\theta)=f((s,a);\theta_0)  \nonumber \\
	&+\int\mdsone(\theta^\top\psi(s,a)>0)\cdot\lambda(\theta)^\top\psi(s,a)dp(\theta)  \big\}
\end{flalign*}
	where $f((s,a);\theta_0)$ is the two-layer neural network corresponding to the initial parameter $\theta_0=W_0$, $\lambda(\theta):$ $\sR^{d}\rightarrow\sR^d$ is a weighted function satisfying $\linf{\lambda(w)}\leq R/\sqrt{d}$, and $p(\cdot):$ $\sR^d\rightarrow\sR$ is the density $D_w$. We assume that $Q^i_{\pi_W}\in \gF_{R,\infty}$ for all $\pi_W$ and $i=\{0,\cdots,p \}$.
\end{assumption}
\begin{assumption}\label{ass3}
	For any parameterized policy $\pi_W$, there exists a constant $C_\zeta>0$ such that for all $k\geq 0$, $\sE_{\mu_{\pi_W}}\left[\exp\left(\ltwo{\bar{g}_k(\theta^i_k) - g_k(\theta^i_k)}^2/C^2_\zeta\right)\right]\leq 1$.
\end{assumption}
Assumption \ref{ass1} implies that the distribution of $\psi(s,a)$ has a uniformly upper bounded probability density over the unit sphere, which can be satisfied for most of the ergodic Markov chain. Assumption \ref{ass2} is a mild regularity condition on $Q^i_{\pi_W}$, as $\gF_{R,\infty}$ is a function class of neural networks with infinite width, which captures a sufficiently general family of functions.
Assumption \ref{ass3} on the variance bound is standard, which has been widely adopted in stochastic optimization literature \cite{ghadimi2013stochastic, nemirovski2009robust,lan2012optimal,ghadimi2016accelerated}.


In the following lemma, we characterize the convergence rate of neural TD {\em in high probability}, which is needed for our the analysis. Such a result is stronger than the convergence {\em in expectation} provided in \cite{bhandari2018finite,cai2019neural,wang2019neural,zhang2020generative,srikant2019finite}, which is not sufficient for our need later on.
\begin{lemma}[Convergence rate of TD in high probability]\label{lemma: neuralTD}
Consider the TD iteration with neural network approximation defined in \cref{eq: 6}. Let $\bar{\theta}_K=\frac{1}{K}\sum_{k=0}^{K-1}\theta_k$ be the average of the output from $k=0$ to $K-1$. Let $\bar{Q}^i_t(s,a)=f((s,a),\theta^i_{K_{\text{in}}})$ be an estimator of $Q^i_{\pi_{\tau_tW_t}}(s,a)$. Suppose Assumptions \ref{ass1}-\ref{ass3} hold, assume that the stationary distribution $\mu_{\pi_W}$ is not degenerate for all $W\in\mB$, and let the stepsize $\beta=\min\{1/\sqrt{K},(1-\gamma)/12\}$. Then, with probability at least $1-\delta$, we have
{\small \begin{flalign*}
&\lmupi{\bar{Q}^i_t(s,a) - Q^i_{\pi_{\tau_t W_t}}(s,a)}^2 \leq  \Theta\big( \frac{1}{(1-\gamma)^2\sqrt{K}}\sqrt{\log\left(\frac{1}{\delta}\right)} \big) \\
&\hspace{2cm} + \Theta\Big(\frac{1}{(1-\gamma)^3m^{1/4}} \sqrt{\log\left(\frac{K}{\delta}\right)} \Big).
\end{flalign*}}
\end{lemma}
\Cref{lemma: neuralTD} implies that after performing the neural TD learning in \cref{eq: 6}-\cref{eq: 7} for $\Theta(\sqrt{m})$ iterations, we can obtain an approximation $\bar{Q}^i_t$ such that $||\bar{Q}^i_t - Q^i_{\pi_{\tau_tW_t}}||_{\mu_\pi}=\mathcal{O}(1/m^{1/8})$ with high probability. 

\textbf{Constraint Estimation:} Since the state space is usually very large or even infinite in the function approximation setting, we cannot include all state-action pairs to estimate the constraints as for the tabular setting. Instead, we sample a batch of state-action pairs $(s_j,a_j)\in \gB_t$ from the distribution $\xi(\cdot)\pi_{W_t}(\cdot|\cdot)$, and let the weight factor $\rho_j=1/\lone{\gB_t}$ for all $j$. In this case, the estimation error of the constrains $\lone{\bar{J}_i(\theta^i_t) - J_i(w_t)}$ is small when the policy evaluation $\bar{Q}^i_t$ is accurate and the batch size $\lone{\gB_t}$ is large. We assume the following concentration property for the sampling process in the constraint estimation step. Similar assumptions have also been taken in \cite{ghadimi2013stochastic, nemirovski2009robust,lan2012optimal,ghadimi2016accelerated}.
\begin{assumption}\label{ass4}
	For any parameterized policy $\pi_W$, there exists a constant $C_f>0$ such that  for all $k\geq 0$, $\sE_{\xi\cdot\pi_{W}}\left[ \exp( {[\bar{Q}^i_t(s,a) - \sE_{\xi\cdot\mu_{\pi_{\tau_{t} W_{t}}}}[\bar{Q}^i_t(s,a)]^2}/{C^2_f} ) \right]\leq 1$.
\end{assumption}

\textbf{Policy Optimization:} In the neural softmax approximation setting, at each iteration $t$, an approximation of the natural policy gradient can be obtained by solving the following linear regression problem \cite{agarwal2019optimality,wang2019neural,xu2019two}:
\begin{flalign}
	&\Delta_i(W_t)\approx\bar{\Delta}_t\nonumber\\
	&= \argmin_{\theta\in\mB}\sE_{\nu_{\pi_{\tau_tW_t}}}[(\bar{Q}^i_t(s,a)-\phi_{W_t}(s,a)^\top\theta)^2].\label{eq: 9}
\end{flalign}
Given the approximated natural policy gradient $\bar{\Delta}_t$, the policy update takes the form of
\begin{flalign}
	&\tau_{t+1}=\tau_t + \alpha,\,\,\, \tau_{t+1}\cdot w_{t+1}=\tau_{t}\cdot w_{t} + \alpha \bar{\Delta}_t \,\,\text{(line 7)}\nonumber \\
	\text{or}\quad &\tau_{t+1}\cdot w_{t+1}=\tau_{t}\cdot w_{t} - \alpha \bar{\Delta}_t \,\,(\text{line 10}).\label{eq: 10}
\end{flalign}
Note that in \cref{eq: 10} we also update the temperature parameter by $\tau_{t+1}=\tau_t + \alpha$ simultaneously, which ensures $w_t\in\mB$ for all $t$. The following theorem characterizes the convergence rate of \Cref{algorithm_cpg} in terms of both the objective function and the constraint violation.

\begin{theorem}\label{thm2}
	Consider \Cref{algorithm_cpg} in the function approximation setting with neural softmax policy parameterization defined in \cref{eq: 5}. Suppose Assumptions \ref{ass1}-\ref{ass4} hold. Suppose the same setting of policy evaluation step stated in \Cref{lemma: neuralTD} holds, and consider performing the neural TD in \cref{eq: 6} and \cref{eq: 7} with $K_{\text{in}}=\Theta((1-\gamma)^2\sqrt{m})$ at each iteration. Let the tolerance $\eta=\Theta(m(1-\gamma)^{-1}/\sqrt{T} + (1-\gamma)^{-2.5}m^{-1/8})$ and perform the NPG update defined in \cref{eq: 10} with $\alpha=\Theta(1/\sqrt{T})$. Then with probability at least $1-\delta$, we have
	\begin{flalign*}
	J_0(\pi^*)-&\sE[J_0(\pi_{\tau_{\text{out}}W_{\text{out}}})]\leq \Theta\left( \frac{1}{(1-\gamma)\sqrt{T}} \right) \\
	&+ \Theta\left(\frac{1}{(1-\gamma)^{2.5}m^{1/8}} \log^{\frac{1}{4}}\left(\frac{(1-\gamma)^2T\sqrt{m}}{\delta}\right) \right),
	\end{flalign*}
    and for all $i=1,\cdots,p$, we have
    \begin{flalign*}
   & \sE[J_i(\pi_{\tau_{\text{out}}W_{\text{out}}})]-d_i\leq \Theta\left( \frac{1}{(1-\gamma)\sqrt{T}} \right) \\
    &\qquad+ \Theta\left(\frac{1}{(1-\gamma)^{2.5}m^{1/8}} \log^{\frac{1}{4}}\left(\frac{(1-\gamma)^2T\sqrt{m}}{\delta}\right) \right).
    \end{flalign*}
where the expectation is taken only with respect to the randomness of selecting $W_{\text{out}}$ from $\gN_0$.
\end{theorem}
\vspace{-0.2cm}
\Cref{thm2} guarantees that CRPO converges to the global optimal policy $\pi^*$ in the feasible set at a sublinear rate $\mathcal{O}(1/\sqrt{T})$ with a approximation error $\mathcal{O}(m^{-1/8})$ vanishes as the network width $m$ increases. The constraint violation bound also converges to zero at a sublinear rate $\mathcal{O}(1/\sqrt{T})$ with a vanishing error $\mathcal{O}(m^{-1/8})$ decreases as $m$ increase. The approximation error arises from both the policy evaluation and policy optimization due to the limited expressive power of neural networks. 

To compare with the primal-dual approach in the function approximation setting, \Cref{thm2} shows that while the value function gap of CRPO achieves the same convergence rate as the primal-dual approach, the constraint violation of CRPO decays at a convergence rate of $\mathcal{O}(1/\sqrt{T})$, which substantially outperforms the rate $\mathcal{O}(1/T^{\frac{1}{4}})$ of the primal-dual approach \cite{ding2020natural}. Such an advantage of CRPO is further validated by our experiments in \Cref{sec:experiment}, which show that the constraint violation of CRPO vanishes much faster than that of the primal-dual approach.




\begin{remark}
Our convergence analysis for \Cref{thm2} can still go through without Assumptions \ref{ass3} and \ref{ass4}. As a result, the convergence rate of CRPO would have polynomial dependence on $\delta$ rather than logarithmic dependence.
\end{remark}



\section{Experiments}\label{sec:experiment}
In this section, we conduct simulation experiments on different SRL tasks to compare our CRPO with the primal-dual optimization (PDO) approach.
We consider two tasks based on OpenAI gym \cite{1606.01540} with each having multiple constraints given as follows:

%


\textbf{Cartpole: }The agent is rewarded for keeping the pole upright, but is penalized with cost if (1) entering into some specific areas, or (2) having the angle of pole being large.

\textbf{Acrobot: }The agent is rewarded for swing the end-effector at a specific height, but is penalized with cost if (1) applying torque on the joint when the first link swings in a prohibited direction, or (2) when the the second link swings in a prohibited direction with respect to the first link.

The detailed experimental setting is described in \Cref{sc: expsetting}. For both experiments, we use neural softmax policy with two hidden layers of size $(128, 128)$. 
For fair comparison, we adopt TRPO as the optimizer for both CRPO and PDO. In CRPO, we let the tolerance $\eta=0.5$ in both tasks. In PDP, we initialize the Lagrange multiplier as zero, and select the best tuned stepsize for dual variable update in both tasks.
We find that the performance of CRPO is robust to the value of $\eta$ over a wide range, while in PDO method the convergence performance is very sensitive to the stepsize of the dual variable (see additional experiments of hyperparameters comparison in \Cref{sc: expsetting}). Thus, in contrast to the difficulty of tuning the PDO method, CRPO is much less sensitive to hyper-parameters and is hence much easier to tune.


The learning curves for CRPO and PDO are provided in \Cref{fig: 1}. At each step we evaluate the performance based on two metrics: the return reward and constraint value of the output policy. We show the learning curve of unconstrained TRPO (the green line), which although achieves the best reward, does not satisfy the constraints.
\begin{figure}[h]
	\vspace{-0.4cm}
	\centering 
	\begin{tabular}{c}
	\subfigure[Cartpole]{\includegraphics[width=1.5in]{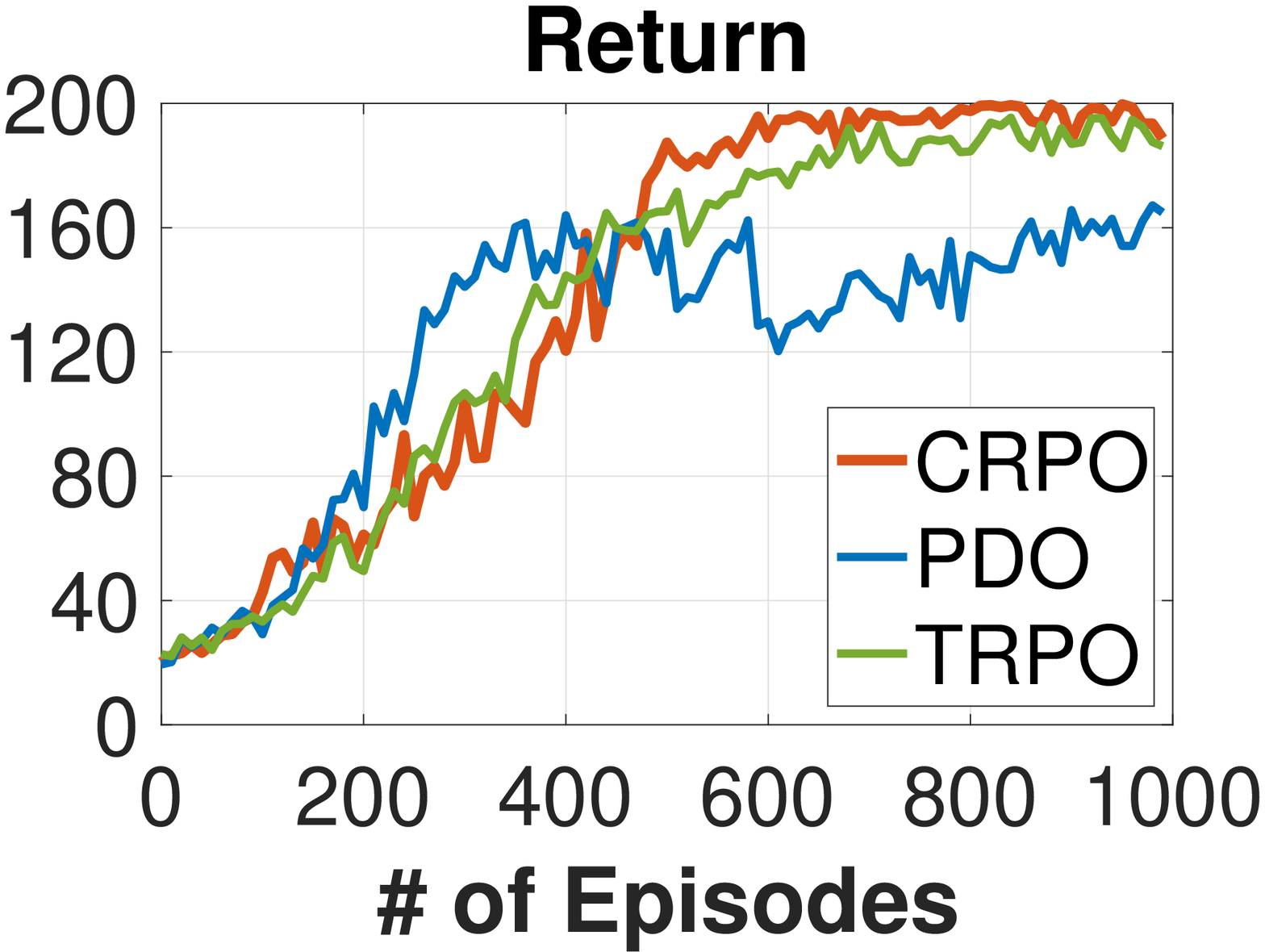}\includegraphics[width=1.5in]{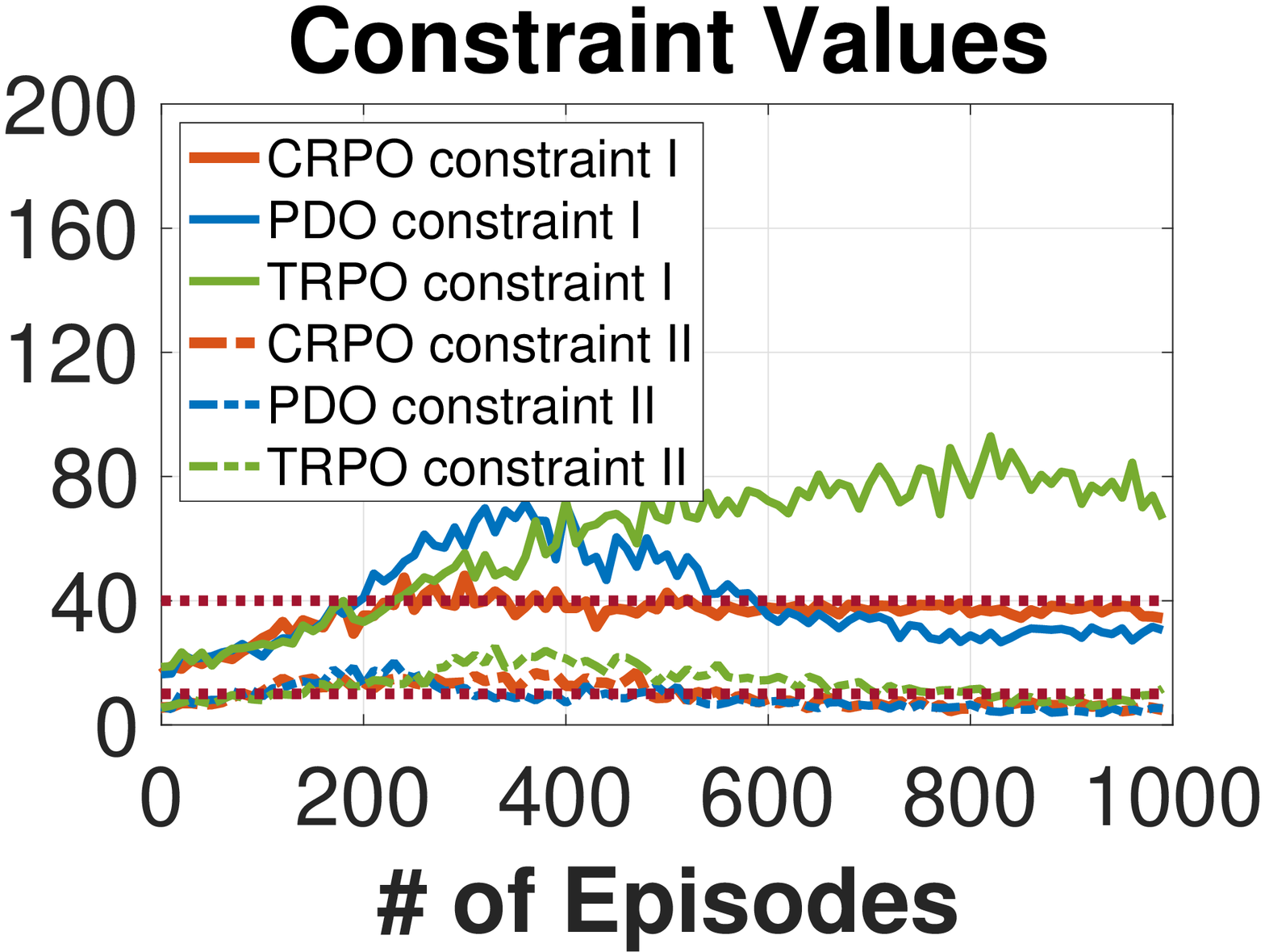}} \\
	\subfigure[Acrobot]{\includegraphics[width=1.5in]{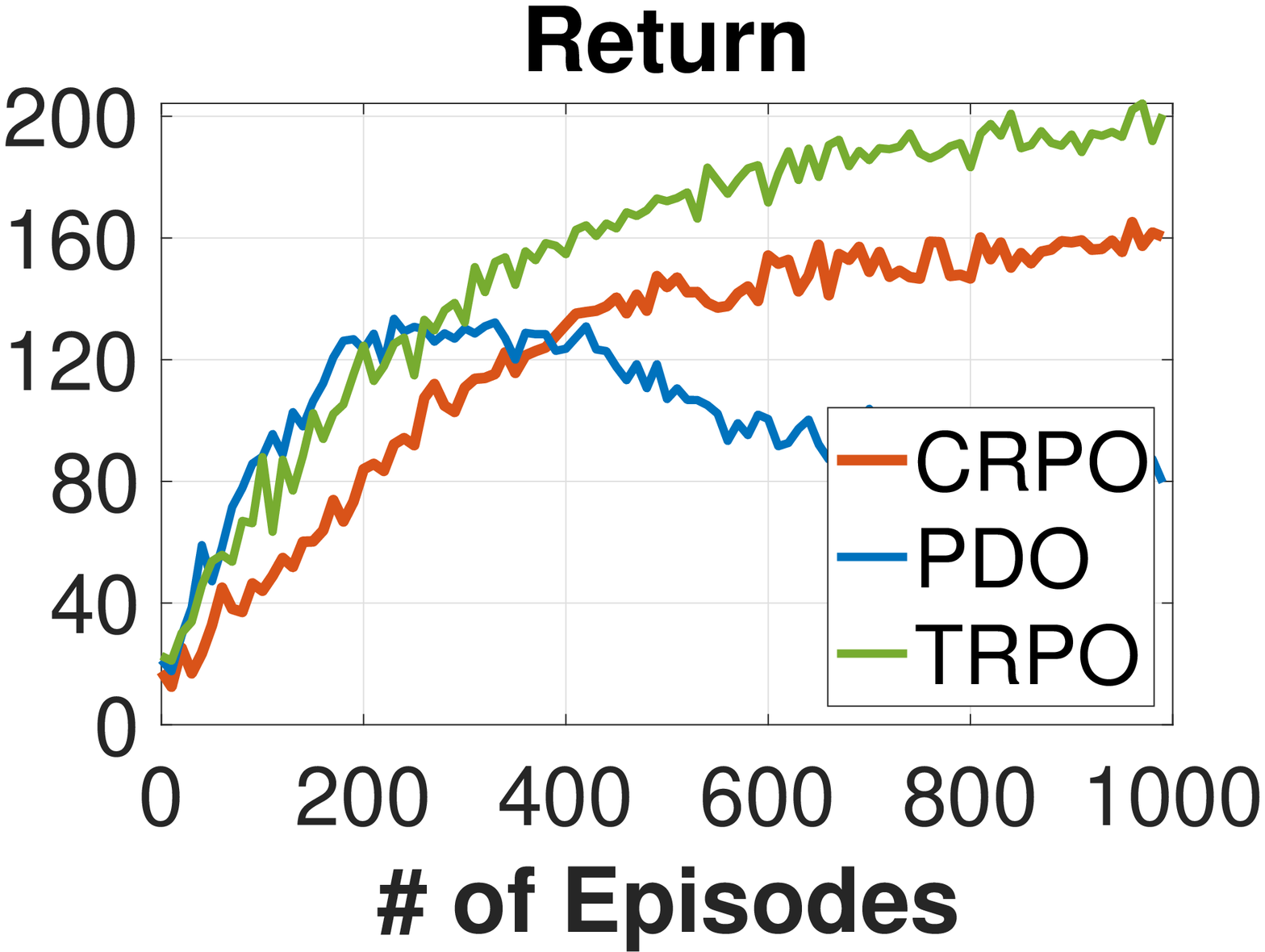}\includegraphics[width=1.5in]{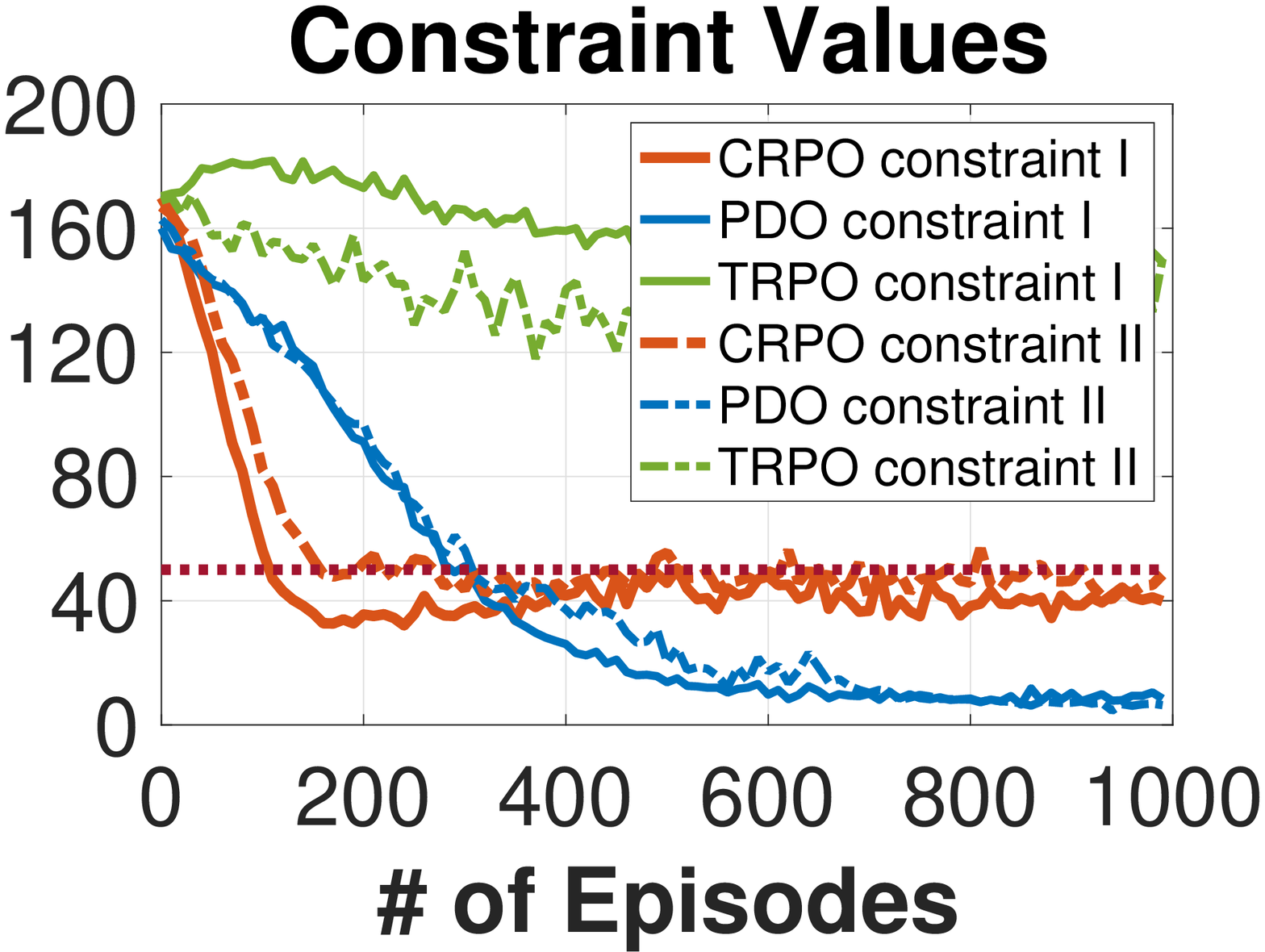}}
	\end{tabular}
	\vspace{-0.4cm}
	\caption{\small\em Average performance for CRPO, PDO, and unconstrained TRPO over 10 seeds. The red dot lines in (a) and (b) represent the limits of the constraints.} 
	\label{fig: 1}
	\vspace{-0.4cm}
\end{figure}

In both tasks, CRPO tracks the constraint returns almost exactly to the limit, indicating that CRPO sufficiently explores the boundary of the feasible set, which results in an optimal return reward. In contrast, although PDO also outputs a constraints-satisfying policy in the end, it tends to over- or under-enforce the constraints, which results in lower return reward and unstable constraint satisfaction performance. In terms of the convergence, the constraints of CRPO drop below the thresholds (and thus satisfy the constraints) much faster than that of PDO, corroborating our theoretical comparison that the constraint violation of CRPO (given in \Cref{thm2}) converges much faster than that of PDO given in \cite{ding2020natural}.

\vspace{-0.3cm}
\section{Conclusion}
In this paper, we propose a novel CRPO approach for policy optimization for SRL, which is easy to implement and has provable global optimality guarantee. We show that CRPO achieves an $\mathcal{O}(1/\sqrt{T})$ convergence rate to the global optimum and an $\mathcal{O}(1/\sqrt{T})$ rate of vanishing constraint error when NPG update is adopted as the optimizer. This is the first primal SRL algorithm that has a provable convergence guarantee to a global optimum. In the future, it is interesting to incorporate various momentum schemes to CRPO to improve its convergence performance. 




\bibliography{ref}
\bibliographystyle{icml2021}

\onecolumn
\newpage
\appendix
\noindent {\Large \textbf{Supplementary Materials}}

\section{Experimental Setting}\label{sc: expsetting}
In our constrained Cartpole environment, the cart is restricted in the area $[-2.4,2.4]$. Each episode length is no longer than 200 and terminated when the angle of the pole is larger than 12 degree. During the training, the agent receives a reward $+1$ for every step taken, but is penalized with cost $+1$ if (1) entering the area $[-2.4, -2.2]$, $[-1.3, -1.1]$, $[-0.1, 0.1]$, $[1.1, 1.3]$, and $[2.2, 2.4]$; or (2) having the angle of pole larger than 6 degree.

In our constrained Acrobot environment, each episode has length 500. 
During the training, the agent receives a reward $+1$ when the end-effector is at a height of 0.5, but is penalized with cost $+1$ when 
(1) a torque with value $+1$ is applied when the first pendulum swings along an anticlockwise direction; or
(2) a torque with value $+1$ is applied when the second pendulum swings along an anticlockwise direction with respect to the first pendulum.

For details about the update of PD, please refer to \cite{achiam2017constrained}[Section 10.3.3]. The performance of PD is very sensitive to the stepsize of the dual variable's update. If the stepsize is too small, then the dual variable will not update quickly to enforce the constraints. If the stepsize is too large, then the algorithm will behave conservatively and have low return reward.
To appropriately select the stepsize for the dual variable, we conduct the experiments with the learning rates $\{0.0001, 0.0005, 0.001,0.005,0.01,0.05\}$ for both tasks. The learning rate $0.005$ performs the best in the first task, and the learning rate $0.0005$ performs the best in the second task. Thus, our reported result of Cartpole is with the stepsize $0.005$ and our reported result of Acrobot is with the stepsize $0.0005$.

Next, we investigate the robustness of CRPO with respect to the tolerance parameter $\eta$. We conduct the experiments under the following values of $\eta$ $\{10, 5, 2, 1, 0.5\}$ for the Acrobot environment. It can be seen from \Cref{fig: 2} that the learning curves of CRPO with the tolerance parameter $\eta$ taking different values are almost the same, which indicates that the convergence performance of CRPO is robust to the value of $\eta$ over a wide range. Thus, the tolerance parameter $\eta$ does not cause much parameter tuning cost for CRPO.

\begin{figure*}[ht]  
	\centering 
	\subfigure[Return Reward]{\includegraphics[width=55mm]{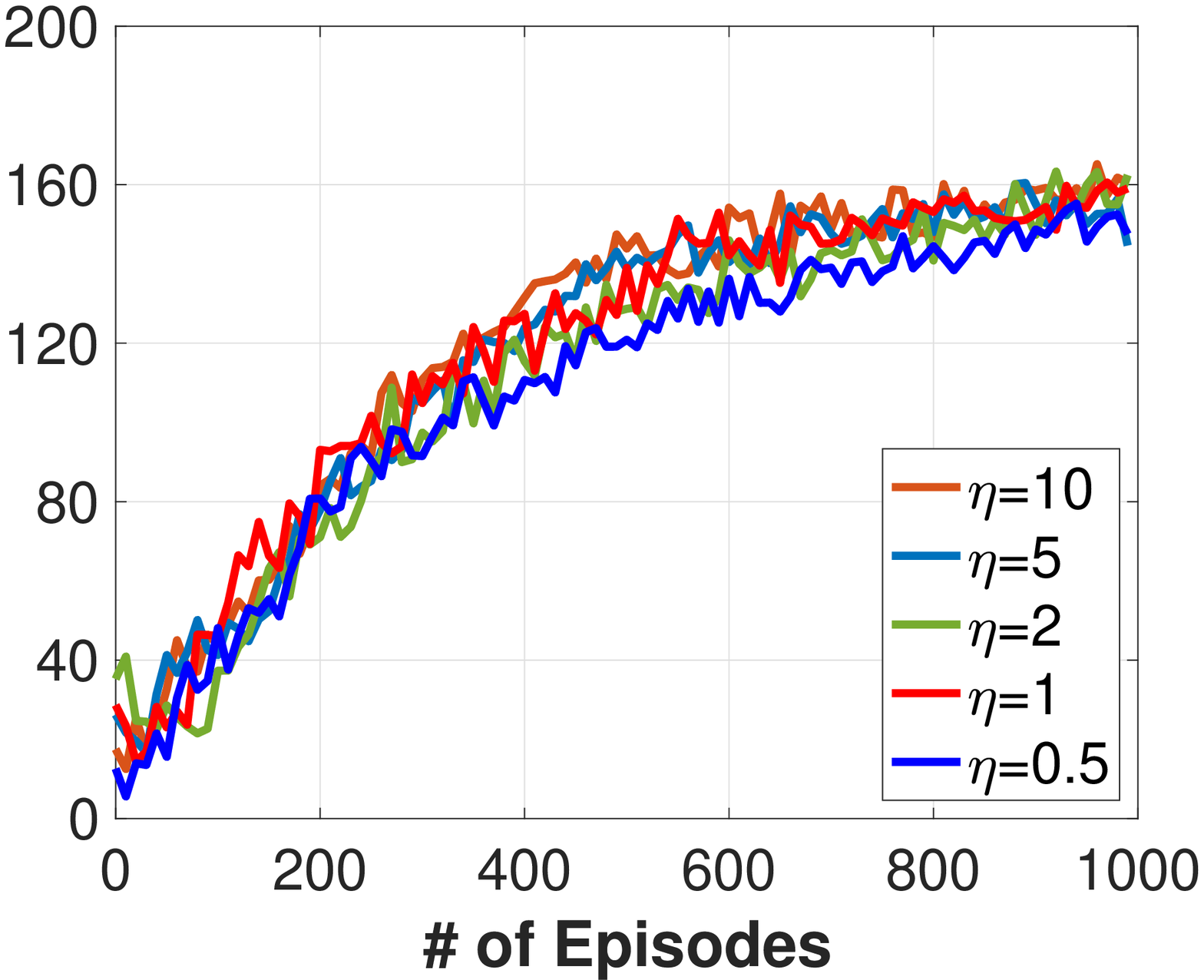}}
	\subfigure[Constraint Value I]{\includegraphics[width=55mm]{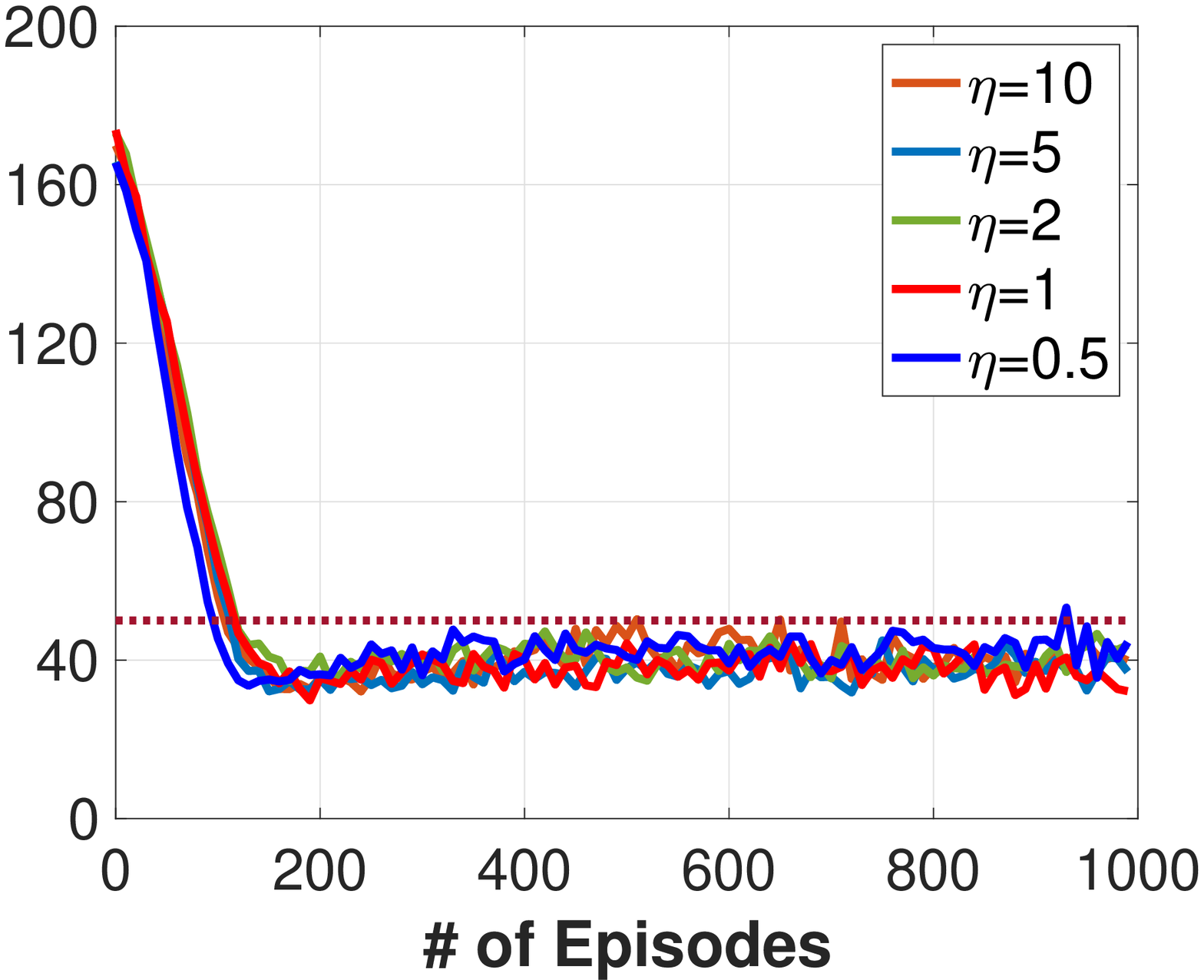}}
	\subfigure[Constraint Value II]{\includegraphics[width=55mm]{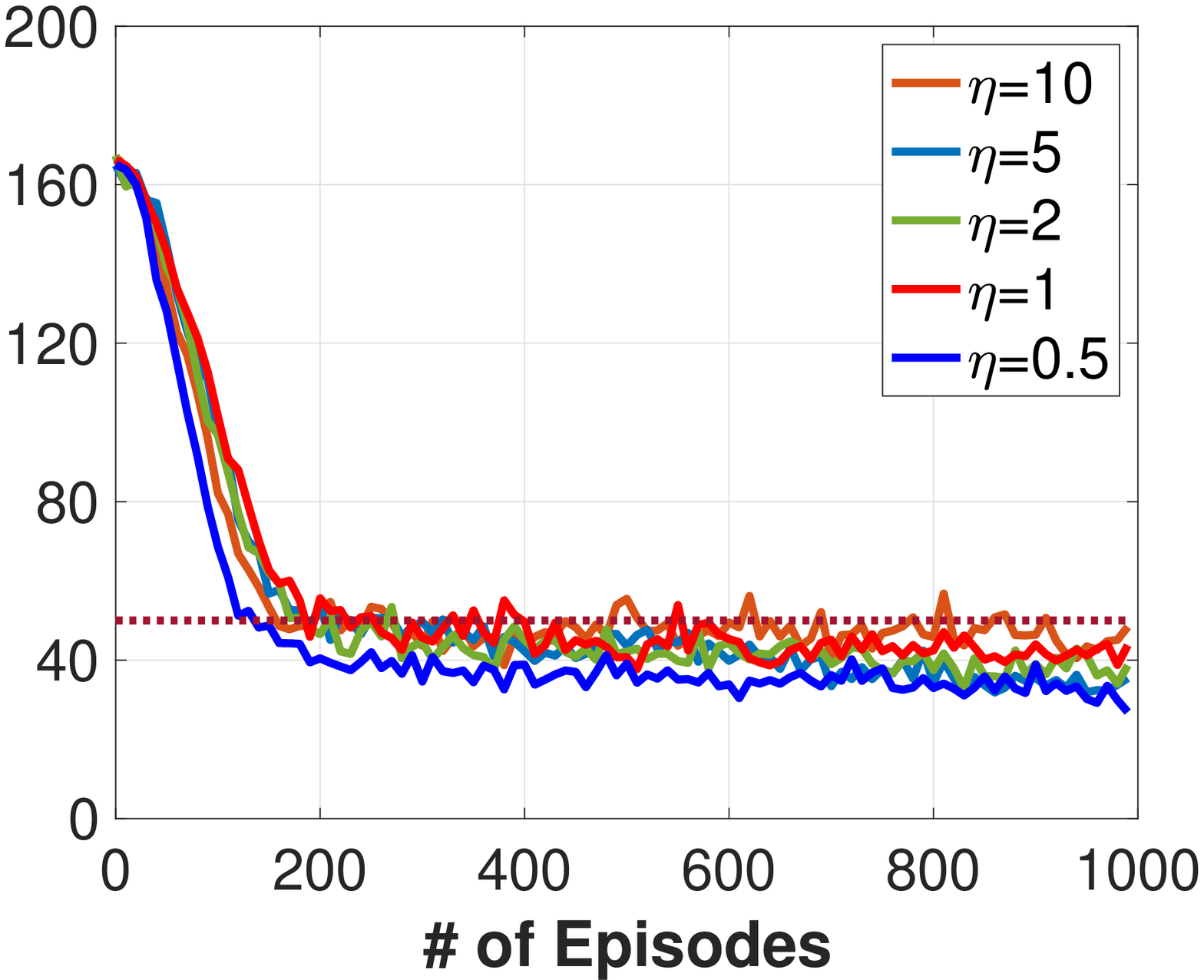}}
	\caption{Comparison of CRPO in Acrobot with tolerance parameter $\eta$ taking different values.}   \label{fig: 2}
\end{figure*}

\section{Proof of \Cref{thm1}: Tabular Setting}\label{app: thm1}

\subsection{Supporting Lemmas for Poof of \Cref{thm1}}
The following lemma characterizes the convergence rate of TD learning in the tabular setting.
\begin{lemma}[\cite{dalal2019tale}]\label{lemma: tabularTD}
	Consider the iteration given in \cref{eq: 3} with arbitrary initialization $\theta^i_0$. Assume that the stationary distribution $\mu_{\pi_w}$ is not degenerate for all $w\in \sR^{\lone{\mcs}\times\lone{\mca}}$. Let stepsize $\beta_k=\Theta(\frac{1}{t^\sigma})\,(0<\sigma<1)$. Then, with probability at least $1-\delta$, we have
	{\small\begin{flalign*}
		\ltwo{\theta^i_K-\theta^i_*(\pi_w)}=\mathcal{O}\left(\frac{\log(\lone{\mcs}^2\lone{\mca}^2K^2/\delta)}{(1-\gamma)K^{\sigma/2}}\right).
		\end{flalign*}}
\end{lemma}
Note that $\sigma$ can be arbitrarily close to $1$. 
\Cref{lemma: tabularTD} implies that we can obtain an approximation $\bar{Q}^i_t$ such that $\ltwo{\bar{Q}^i_t - Q^i_{\pi_w}}=\mathcal{O}(1/\sqrt{K_{\text{in}}})$ with high probability.

\begin{lemma}[Performance difference lemma \cite{kakade2002approximately} ]\label{lemma_s3}
	For all policies $\pi$, $\pi^\prime$ and initial distribution $\rho$, we have
	\begin{flalign*}
		J_i^\rho(\pi)-J^\rho_i(\pi^\prime)=\frac{1}{1-\gamma}\sE_{s\sim\nu_\rho}\sE_{a\sim\pi(\cdot|s)}[A^i_{\pi^\prime}(s,a)]
	\end{flalign*}
    where $J^\rho_i(\pi)$ and $\nu_\rho$ denote the accumulated reward (cost) function and visitation distribution under policy $\pi$ when the initial state distribution is $\rho$.
\end{lemma}

\begin{lemma}[Lemma 5.6. \cite{agarwal2019optimality}]\label{lemma_s1}
	Considering the approximated NPG update in line 7 of \Cref{algorithm_cpg} in the tabular setting and $i=0$, the NPG update takes the form:
	\begin{flalign*}
		w_{t+1}=w_t + \frac{\alpha}{1-\gamma}\bar{Q}^i_t,\quad\text{and}\quad \pi_{w_{t+1}}(a|s)=\pi_{w_{t}}(a|s)\frac{\exp(\alpha \bar{Q}^i_t(s,a)/(1-\gamma))}{Z_t(s)},
	\end{flalign*}
    where
    \begin{flalign*}
    	Z_t(s)=\sum_{a\in\mca}\pi_{w_t}(a|s)\exp\left( \frac{\alpha\bar{Q}^i_t(s,a) }{1-\gamma} \right).
    \end{flalign*}
\end{lemma}
Note that if we follow the update in line 10 of \Cref{algorithm_cpg}, we can obtain similar results for the case $i\in\{1,\cdots,p\}$ as stated in \Cref{lemma_s1}.

\begin{lemma}[Policy gradient property of softmax parameterization]\label{lemma_s4}
	Considering the softmax policy in the tabular setting (\cref{eq: 2}). For any initial state distribution $\rho$, we have
	\begin{flalign*}
		\nabla_w J^\rho_i(w)=\sE_{s\sim\nu_\rho}\sE_{a\sim\pi_w(\cdot|s)}\left[\left(\mdsone_{as}-\sum_{a^\prime\in\mca}\pi_w(a^\prime|s)\mdsone_{a^\prime s}\right)Q^i_{\pi_w}(s,a)\right],
	\end{flalign*}
    and
    \begin{flalign*}
    	\ltwo{\nabla_w J^\rho_i(w)}\leq \frac{2c_{\max}}{1-\gamma},
    \end{flalign*}
    where $\mdsone_{as}$ is an $\lone{\mcs}\times\lone{\mca}$-dimension vector, with $(a,s)$-th element being one, and the rest elements being zero.
\end{lemma}
\begin{proof}
	The first result can follows directly from Lemma C.1 in \cite{agarwal2019optimality}. We now proceed to prove the second result.
	\begin{flalign*}
		\ltwo{\nabla_w J^\rho_i(w)}&=\ltwo{\sE\left[\left(\mdsone_{as}-\sum_{a^\prime\in\mca}\pi_w(a^\prime|s)\mdsone_{a^\prime s}\right)Q^i_{\pi_w}(s,a)\right]}\nonumber\\
		&\leq \sE\left[\ltwo{\left(\mdsone_{as}-\sum_{a^\prime\in\mca}\pi_w(a^\prime|s)\mdsone_{a^\prime s}\right)Q^i_{\pi_w}(s,a)}\right]\nonumber\\
		&\leq \sE\left[\ltwo{\mdsone_{as}-\sum_{a^\prime\in\mca}\pi_w(a^\prime|s)\mdsone_{a^\prime s}}Q^i_{\pi_w}(s,a)\right]\nonumber\\
		&\leq 2\sE\left[Q^i_{\pi_w}(s,a)\right]\leq \frac{2c_{\max}}{1-\gamma}.
	\end{flalign*}
\end{proof}

\begin{lemma}[Performance improvement bound for approximated NPG]\label{lemma_s2}
	For the iterates $\pi_{w_t}$ generated by the approximated NPG updates in line 7 of \Cref{algorithm_cpg} in the tabular setting, we have for all initial state distribution $\rho$ and when $i=0$, the following holds
	\begin{flalign*}
		&J^\rho_0(w_{t+1})-J^\rho_0(w_{t})\nonumber\\
		&\geq\frac{1-\gamma}{\alpha}\sE_{s\sim \rho}\left(\log Z_t(s) - \frac{\alpha}{1-\gamma}V^0_{\pi_{w_t}}(s) +\frac{\alpha}{1-\gamma}\sum_{a\in\mca}\pi_{w_t}(a|s)\lone{\bar{Q}^0_t(s,a)-Q^0_{\pi_{w_t}}(s,a)} \right) \nonumber\\
		&\quad - \frac{1}{1-\gamma} \sE_{s\sim\nu_\rho}\sum_{a\in\mca}\pi_{w_t}(a|s)\lone{\bar{Q}^0_t(s,a)-Q^0_{\pi_{w_t}}(s,a)} \nonumber\\
		&\quad - \frac{1}{1-\gamma}\sE_{s\sim\nu_\rho}\sum_{a\in\mca}\pi_{w_{t+1}}(a|s) \lone{Q^0_{\pi_{w_t}}(s,a)-\bar{Q}^0_t(s,a)} \nonumber.
	\end{flalign*}
\end{lemma}
\begin{proof}
We first provide the following lower bound.
\begin{flalign*}
	&\log Z_t(s)-\frac{\alpha}{1-\gamma}V^i_{\pi_{w_t}}(s) \nonumber\\
	&= \log \sum_{a\in\mca}\pi_{w_t}(a|s)\exp\left( \frac{\alpha\bar{Q}^i_t(s,a) }{1-\gamma} \right) - \frac{\alpha}{1-\gamma}V^i_{\pi_{w_t}}(s)\nonumber\\
	&\geq  \sum_{a\in\mca}\pi_{w_t}(a|s)\log\exp\left( \frac{\alpha\bar{Q}^i_t(s,a) }{1-\gamma} \right) - \frac{\alpha}{1-\gamma}V^i_{\pi_{w_t}}(s)\nonumber\\
	&=\frac{\alpha}{1-\gamma}\sum_{a\in\mca}\pi_{w_t}(a|s)(\bar{Q}^i_t(s,a)-Q^i_{\pi_{w_t}}(s,a))+\frac{\alpha}{1-\gamma}\sum_{a\in\mca}\pi_{w_t}(a|s)Q^i_{\pi_{w_t}}(s,a) - \frac{\alpha}{1-\gamma}V^i_{\pi_{w_t}}(s)\nonumber\\
	&=\frac{\alpha}{1-\gamma}\sum_{a\in\mca}\pi_{w_t}(a|s)(\bar{Q}^i_t(s,a)-Q^i_{\pi_{w_t}}(s,a))\nonumber\\
	&\geq \frac{-\alpha}{1-\gamma}\sum_{a\in\mca}\pi_{w_t}(a|s)\lone{\bar{Q}^i_t(s,a)-Q^i_{\pi_{w_t}}(s,a)}.\nonumber
\end{flalign*}
Thus, we conclude that
\begin{flalign*}
	\log Z_t(s)-\frac{\alpha}{1-\gamma}V^i_{\pi_{w_t}}(s) + \frac{\alpha}{1-\gamma}\sum_{a\in\mca}\pi_{w_t}(a|s)\lone{\bar{Q}^i_t(s,a)-Q^i_{\pi_{w_t}}(s,a)}\geq 0.
\end{flalign*}
    We then proceed to prove \Cref{lemma_s2}. The performance difference lemma (\Cref{lemma_s3}) implies:
    \begin{flalign}
    	&J^\rho_i(w_{t+1})-J^\rho_i(w_{t})\nonumber\\
    	&=\frac{1}{1-\gamma}\sE_{s\sim\nu_\rho}\sum_{a\in\mca}\pi_{w_{t+1}}(a|s)A^i_{\pi_{w_t}}(s,a)\nonumber\\
    	&=\frac{1}{1-\gamma}\sE_{s\sim\nu_\rho}\sum_{a\in\mca}\pi_{w_{t+1}}(a|s)Q^i_{\pi_{w_t}}(s,a) - \frac{1}{1-\gamma}\sE_{s\sim\nu_\rho}V^i_{\pi_{w_t}}(s)\nonumber\\
    	&=\frac{1}{1-\gamma}\sE_{s\sim\nu_\rho}\sum_{a\in\mca}\pi_{w_{t+1}}(a|s)\bar{Q}^i_t(s,a) + \frac{1}{1-\gamma}\sE_{s\sim\nu_\rho}\sum_{a\in\mca}\pi_{w_{t+1}}(a|s)(Q^i_{\pi_{w_t}}(s,a)-\bar{Q}^i_t(s,a))\nonumber\\
    	&\quad- \frac{1}{1-\gamma}\sE_{s\sim\nu_\rho}V^i_{\pi_{w_t}}(s)\nonumber\\
    	&\overset{(i)}{=}\frac{1}{\alpha}\sE_{s\sim\nu_\rho}\sum_{a\in\mca}\pi_{w_{t+1}}(a|s)\log\left( \frac{\pi_{w_{t+1}}(a|s)Z_t(s)}{\pi_{w_t}(a|s)} \right) \nonumber\\
    	&\quad + \frac{1}{1-\gamma}\sE_{s\sim\nu_\rho}\sum_{a\in\mca}\pi_{w_{t+1}}(a|s)(Q^i_{\pi_{w_t}}(s,a)-\bar{Q}^i_t(s,a)) - \frac{1}{1-\gamma}\sE_{s\sim\nu_\rho}V^i_{\pi_{w_t}}(s)\nonumber\\
    	&=\frac{1}{\alpha}\sE_{s\sim\nu_\rho}D_{\text{KL}}(\pi_{w_{t+1}}||\pi_{w_t}) + \frac{1}{\alpha}\sE_{s\sim\nu_\rho}\log Z_t(s) \nonumber\\
    	&\quad + \frac{1}{1-\gamma}\sE_{s\sim\nu_\rho}\sum_{a\in\mca}\pi_{w_{t+1}}(a|s)(Q^i_{\pi_{w_t}}(s,a)-\bar{Q}^i_t(s,a)) - \frac{1}{1-\gamma}\sE_{s\sim\nu_\rho}V^i_{\pi_{w_t}}(s)\nonumber\\
    	&\geq \frac{1}{\alpha}\sE_{s\sim\nu_\rho}\left(\log Z_t(s) - \frac{\alpha}{1-\gamma}V^i_{\pi_{w_t}}(s) +\frac{\alpha}{1-\gamma}\sum_{a\in\mca}\pi_{w_t}(a|s)\lone{\bar{Q}^i_t(s,a)-Q^i_{\pi_{w_t}}(s,a)} \right) \nonumber\\
    	&\quad - \frac{1}{1-\gamma} \sE_{s\sim\nu_\rho}\sum_{a\in\mca}\pi_{w_t}(a|s)\lone{\bar{Q}^i_t(s,a)-Q^i_{\pi_{w_t}}(s,a)} \nonumber\\
    	&\quad - \frac{1}{1-\gamma}\sE_{s\sim\nu_\rho}\sum_{a\in\mca}\pi_{w_{t+1}}(a|s) \lone{Q^i_{\pi_{w_t}}(s,a)-\bar{Q}^i_t(s,a)} \nonumber\\
    	&\overset{(ii)}{\geq} \frac{1-\gamma}{\alpha}\sE_{s\sim \rho}\left(\log Z_t(s) - \frac{\alpha}{1-\gamma}V^i_{\pi_{w_t}}(s) +\frac{\alpha}{1-\gamma}\sum_{a\in\mca}\pi_{w_t}(a|s)\lone{\bar{Q}^i_t(s,a)-Q^i_{\pi_{w_t}}(s,a)} \right) \nonumber\\
    	&\quad - \frac{1}{1-\gamma} \sE_{s\sim\nu_\rho}\sum_{a\in\mca}\pi_{w_t}(a|s)\lone{\bar{Q}^i_t(s,a)-Q^i_{\pi_{w_t}}(s,a)} \nonumber\\
    	&\quad - \frac{1}{1-\gamma}\sE_{s\sim\nu_\rho}\sum_{a\in\mca}\pi_{w_{t+1}}(a|s) \lone{Q^i_{\pi_{w_t}}(s,a)-\bar{Q}^i_t(s,a)} \nonumber
    \end{flalign}
where $(i)$ follows from the update rule in \Cref{lemma_s1} and $(ii)$ follows from the facts that $\linf{\nu_\rho/\rho}\geq 1-\gamma$ and $\log Z_t(s)-\frac{\alpha}{1-\gamma}V^i_{\pi_{w_t}}(s)+\frac{\alpha}{1-\gamma}\sum_{a\in\mca}\pi_{w_t}(a|s)\lone{\bar{Q}^i_t(s,a)-Q^i_{\pi_{w_t}}(s,a)}\geq 0$.
\end{proof}
Note that if we follow the update in line 10 of \Cref{algorithm_cpg}, we can obtain similar results for the case $i\in\{1,\cdots,p\}$ as stated in \Cref{lemma_s2}.

\begin{lemma}[Upper bound on optimality gap for approximated NPG]\label{lemma_s5}
	Consider the approximated NPG updates in line 7 of \Cref{algorithm_cpg} in the tabular setting when $i=0$. We have
	\begin{flalign*}
		&J_0(\pi^*)-J_0(\pi_{w_t})\nonumber\\
		&\leq \frac{1}{\alpha}\sE_{s\sim\nu^*} (D_{\text{KL}}(\pi^*||\pi_{w_t})-D_{\text{KL}}(\pi^*||\pi_{w_{t+1}})) + \frac{2\alpha c^2_{\max}\lone{\mcs}\lone{\mca}}{(1-\gamma)^3} + \frac{3(1+\alpha c_{\max})}{(1-\gamma)^2}\ltwo{Q^0_{\pi_{w_t}}-\bar{Q}^0_t}.
	\end{flalign*}
\end{lemma}
\begin{proof}
	By the performance difference lemma (\Cref{lemma_s3}), we have
	\begin{flalign}
		&J_i(\pi^*)-J_i(\pi_{w_t})\nonumber\\
		&=\frac{1}{1-\gamma}\sE_{s\sim\nu^*}\sum_{a\in\mca}\pi^*(a|s)A^i_{\pi_{w_t}}(s,a)\nonumber\\
		&=\frac{1}{1-\gamma}\sE_{s\sim\nu^*}\sum_{a\in\mca}\pi^*(a|s)Q^i_{\pi_{w_t}}(s,a) - \frac{1}{1-\gamma}\sE_{s\sim\nu^*}V^i_{\pi_{w_t}}(s)\nonumber\\
		&=\frac{1}{1-\gamma}\sE_{s\sim\nu^*}\sum_{a\in\mca}\pi^*(a|s)\bar{Q}^i_t(s,a) + \frac{1}{1-\gamma}\sE_{s\sim\nu^*}\sum_{a\in\mca}\pi^*(a|s)(Q^i_{\pi_{w_t}}(s,a)-\bar{Q}^i_t(s,a)) \nonumber\\
		&\quad - \frac{1}{1-\gamma}\sE_{s\sim\nu^*}V^i_{\pi_{w_t}}(s)\nonumber\\
		&\overset{(i)}{=}\frac{1}{1-\gamma}\sE_{s\sim\nu^*}\sum_{a\in\mca}\pi^*(a|s)\log\frac{\pi_{w_{t+1}}(a|s)Z_t(s)}{\pi_{w_t}(a|s)} + \frac{1}{1-\gamma}\sE_{s\sim\nu^*}\sum_{a\in\mca}\pi^*(a|s)(Q^i_{\pi_{w_t}}(s,a)-\bar{Q}^i_t(s,a)) \nonumber\\
		&\quad - \frac{1}{1-\gamma}\sE_{s\sim\nu^*}V^i_{\pi_{w_t}}(s)\nonumber\\
		&=\frac{1}{\alpha}\sE_{s\sim\nu^*} (D_{\text{KL}}(\pi^*||\pi_{w_t})-D_{\text{KL}}(\pi^*||\pi_{w_{t+1}})) + \frac{1}{\alpha}\sE_{s\sim\nu^*}\left(\log Z_t(s)-\frac{\alpha}{1-\gamma}V^i_{\pi_{w_t}}(s)\right)  \nonumber\\
		&\quad + \frac{1}{1-\gamma}\sE_{s\sim\nu^*}\sum_{a\in\mca}\pi^*(a|s)(Q^i_{\pi_{w_t}}(s,a)-\bar{Q}^i_t(s,a))\nonumber\\
		&\leq \frac{1}{\alpha}\sE_{s\sim\nu^*} (D_{\text{KL}}(\pi^*||\pi_{w_t})-D_{\text{KL}}(\pi^*||\pi_{w_{t+1}})) \nonumber\\
		&\quad + \frac{1}{\alpha}\sE_{s\sim\nu^*}\left(\log Z_t(s)-\frac{\alpha}{1-\gamma}V^i_{\pi_{w_t}}(s) +\frac{\alpha}{1-\gamma}\sum_{a\in\mca}\pi_{w_t}(a|s)\lone{\bar{Q}^i_t(s,a)-Q^i_{\pi_{w_t}}(s,a)} \right)  \nonumber\\
		&\quad + \frac{1}{1-\gamma}\sE_{s\sim\nu^*}\sum_{a\in\mca}\pi^*(a|s)(Q^i_{\pi_{w_t}}(s,a)-\bar{Q}^i_t(s,a))\nonumber\\
		&\overset{(ii)}{\leq} \frac{1}{\alpha}\sE_{s\sim\nu^*} (D_{\text{KL}}(\pi^*||\pi_{w_t})-D_{\text{KL}}(\pi^*||\pi_{w_{t+1}}))\nonumber\\ 
		&\quad + \frac{1}{1-\gamma}(J^{\nu^*}_i(w_{t+1})-J^{\nu^*}_i(w_{t})) + \frac{1}{(1-\gamma)^2}\sE_{s\sim\nu_{\nu^*}}\sum_{a\in\mca}\pi_{w_{t+1}}(a|s)\lone{Q^i_{\pi_{w_t}}(s,a)-\bar{Q}^i_t(s,a)}  \nonumber\\
		&\quad + \frac{1}{(1-\gamma)^2}\sE_{s\sim\nu_{\nu^*}}\sum_{a\in\mca}\pi_{w_t}(a|s)\lone{Q^i_{\pi_{w_t}}(s,a)-\bar{Q}^i_t(s,a)}\nonumber\\
		&\quad + \frac{1}{1-\gamma}\sE_{s\sim\nu^*}\sum_{a\in\mca}\pi^*(a|s)\lone{Q^i_{\pi_{w_t}}(s,a)-\bar{Q}^i_t(s,a)}\nonumber\\
		&\overset{(iii)}{\leq} \frac{1}{\alpha}\sE_{s\sim\nu^*} (D_{\text{KL}}(\pi^*||\pi_{w_t})-D_{\text{KL}}(\pi^*||\pi_{w_{t+1}})) + \frac{2c_{\max}}{(1-\gamma)^2}\ltwo{w_{t+1}-w_t} + \frac{3}{(1-\gamma)^2}\ltwo{Q^i_{\pi_{w_t}}-\bar{Q}^i_t}\nonumber\\
		&=\frac{1}{\alpha}\sE_{s\sim\nu^*} (D_{\text{KL}}(\pi^*||\pi_{w_t})-D_{\text{KL}}(\pi^*||\pi_{w_{t+1}})) + \frac{2\alpha c_{\max}}{(1-\gamma)^2}\ltwo{\bar{Q}^i_t} + \frac{3}{(1-\gamma)^2}\ltwo{Q^i_{\pi_{w_t}}-\bar{Q}^i_t}\nonumber\\
		&\leq \frac{1}{\alpha}\sE_{s\sim\nu^*} (D_{\text{KL}}(\pi^*||\pi_{w_t})-D_{\text{KL}}(\pi^*||\pi_{w_{t+1}})) + \frac{2\alpha c_{\max}}{(1-\gamma)^2}\ltwo{Q^i_{\pi_{w_t}}} + \frac{3(1+\alpha c_{\max})}{(1-\gamma)^2}\ltwo{Q^i_{\pi_{w_t}}-\bar{Q}^i_t}\nonumber\\
		&\leq \frac{1}{\alpha}\sE_{s\sim\nu^*} (D_{\text{KL}}(\pi^*||\pi_{w_t})-D_{\text{KL}}(\pi^*||\pi_{w_{t+1}})) + \frac{2\alpha c^2_{\max}\lone{\mcs}\lone{\mca}}{(1-\gamma)^3} + \frac{3(1+\alpha c_{\max})}{(1-\gamma)^2}\ltwo{Q^i_{\pi_{w_t}}-\bar{Q}^i_t},\nonumber
	\end{flalign}
where $(i)$ follows from \Cref{lemma_s1}, $(ii)$ follows from \Cref{lemma_s2} and $(iii)$ follows from the Lipschitz property of $J^{\nu^*}_i(w)$ such that $J^{\nu^*}_i(w_{t+1})-J^{\nu^*}_i(w_{t})\leq \frac{2c_{\max}}{1-\gamma}\ltwo{w_{t+1}-w_t}$, which is proved by Proposition 1 in \cite{xu2020improving}.
\end{proof}
Note that if we follow the update in line 10 of \Cref{algorithm_cpg}, we can obtain the following result for the case $i\in\{1,\cdots,p\}$ as stated in \Cref{lemma_s5}:
\begin{flalign*}
	& J_i(\pi_{w_t})-J_i(\pi^*)\nonumber\\
	&\leq \frac{1}{\alpha}\sE_{s\sim\nu^*} (D_{\text{KL}}(\pi^*||\pi_{w_t})-D_{\text{KL}}(\pi^*||\pi_{w_{t+1}})) + \frac{2\alpha c^2_{\max}\lone{\mcs}\lone{\mca}}{(1-\gamma)^3} + \frac{3(1+\alpha c_{\max})}{(1-\gamma)^2}\ltwo{Q^i_{\pi_{w_t}}-\bar{Q}^i_t}.
\end{flalign*}

\begin{lemma}\label{lemma_m2}
	Considering CRPO in \Cref{algorithm_cpg} in the tabular setting. Let $K_{\text{in}}=\Theta(T^{1/\sigma}\log^{2/\sigma}(\lone{\mcs}^2\lone{\mca}^2T^{1+2/\sigma}/\delta))$. Define $\gN_i$ as the set of steps that CRPO algorithm chooses to minimize the $i$-th constraint. With probability at least $1-\delta$, we have
	\begin{flalign*}
		&\sum_{t\in\gN_0}(J_0(\pi^*)-J_0(\pi_{w_t})) + \alpha\eta\sum_{i=1}^{p}\lone{\gN_i} \nonumber\\
		&\leq \sE_{s\sim\nu^*} D_{\text{KL}}(\pi^*||\pi_{w_0}) + \frac{2\alpha^2 c^2_{\max}\lone{\mcs}\lone{\mca}T}{(1-\gamma)^3} + \frac{\alpha\sqrt{T}(2+(1-\gamma)^2+2\alpha c_{\max})}{(1-\gamma)^2}.
	\end{flalign*}
\end{lemma}
\begin{proof}
	If $t\in \gN_0$, by \Cref{lemma_s5} we have
	\begin{flalign}
		&\alpha(J_0(\pi^*)-J_0(\pi_{w_t}))\nonumber\\
		&\leq \sE_{s\sim\nu^*} (D_{\text{KL}}(\pi^*||\pi_{w_t})-D_{\text{KL}}(\pi^*||\pi_{w_{t+1}})) + \frac{2\alpha^2 c^2_{\max}\lone{\mcs}\lone{\mca}}{(1-\gamma)^3} + \frac{3\alpha(1+\alpha c_{\max})}{(1-\gamma)^2}\ltwo{Q^0_{\pi_{w_t}}-\bar{Q}^0_t}.\label{eq: m1}
	\end{flalign}
	If $t\in \gN_i$, similarly we can obtain
	\begin{flalign}
		&\alpha(J_i(\pi_{w_t})-J_i(\pi^*))\nonumber\\
		&\leq \sE_{s\sim\nu^*} (D_{\text{KL}}(\pi^*||\pi_{w_t})-D_{\text{KL}}(\pi^*||\pi_{w_{t+1}})) + \frac{2\alpha^2 c^2_{\max}\lone{\mcs}\lone{\mca}}{(1-\gamma)^3} + \frac{3\alpha(1+\alpha c_{\max})}{(1-\gamma)^2}\ltwo{Q^i_{\pi_{w_t}}-\bar{Q}^i_t}.\label{eq: m2}
	\end{flalign}
	Taking the summation of \cref{eq: m1} and \cref{eq: m2} from $t=0$ to $T-1$ yields
	\begin{flalign}
		&\alpha\sum_{t\in\gN_0}(J_0(\pi^*)-J_0(\pi_{w_t})) + \alpha\sum_{i=1}^{p}\sum_{t\in\gN_i}(J_i(\pi_{w_t})-J_i(\pi^*))\nonumber\\
		&\leq \sE_{s\sim\nu^*} D_{\text{KL}}(\pi^*||\pi_{w_0}) + \frac{2\alpha^2 c^2_{\max}\lone{\mcs}\lone{\mca}T}{(1-\gamma)^3} + \frac{3\alpha(1+\alpha c_{\max})}{(1-\gamma)^2} \sum_{i=0}^{p}\sum_{t\in\gN_i}\ltwo{Q^i_{\pi_{w_t}}-\bar{Q}^i_t}.\label{eq: m3}
	\end{flalign}
	Note that when $t\in \gN_i$ ($i\neq 0$), we have $\bar{J}_{i}(\theta^{i}_t)> d_{i}+\eta$ (line 9 in \Cref{algorithm_cpg}), which implies that
	\begin{flalign}
		J_i(\pi_{w_t})-J_i(\pi^*)&\geq \bar{J}_{i}(\theta^{i}_t)-J_i(\pi^*) - \lone{\bar{J}_{i}(\theta^{i}_t) - J_i(\pi_{w_t})}\nonumber\\
		&\geq d_{i}+\eta-J_i(\pi^*) - \lone{\bar{J}_{i}(\theta^{i}_t) - J_i(\pi_{w_t})}\nonumber\\
		&\geq \eta - \ltwo{Q^i_{\pi_w}-\bar{Q}^i_t}.\label{eq: m4}
	\end{flalign}
	Substituting \cref{eq: m4} into \cref{eq: m3} yields
	\begin{flalign*}
		&\alpha\sum_{t\in\gN_0}(J_0(\pi^*)-J_0(\pi_{w_t})) + \alpha\eta\sum_{i=1}^{p}\lone{\gN_i} - \alpha\sum_{i=1}^{p}\sum_{t\in\gN_i}\ltwo{Q^i_{\pi_{w_t}}-\bar{Q}^i_t} \nonumber\\
		&\leq \sE_{s\sim\nu^*} D_{\text{KL}}(\pi^*||\pi_{w_0}) + \frac{2\alpha^2 c^2_{\max}\lone{\mcs}\lone{\mca}T}{(1-\gamma)^3} + \frac{3\alpha(1+\alpha c_{\max})}{(1-\gamma)^2} \sum_{i=0}^{p}\sum_{t\in\gN_i}\ltwo{Q^i_{\pi_{w_t}}-\bar{Q}^i_t},
	\end{flalign*}
	which implies
	\begin{flalign}
		&\alpha\sum_{t\in\gN_0}(J_0(\pi^*)-J_0(\pi_{w_t})) + \alpha\eta\sum_{i=1}^{p}\lone{\gN_i} \nonumber\\
		&\leq \sE_{s\sim\nu^*} D_{\text{KL}}(\pi^*||\pi_{w_0}) + \frac{2\alpha^2 c^2_{\max}\lone{\mcs}\lone{\mca}T}{(1-\gamma)^3} + \frac{\alpha(2+(1-\gamma)^2+3\alpha c_{\max})}{(1-\gamma)^2} \sum_{i=0}^{p}\sum_{t\in\gN_i}\ltwo{Q^i_{\pi_{w_t}}-\bar{Q}^i_t}.
	\end{flalign}
    By \Cref{lemma: tabularTD}, we have with probability at least $1-\delta$, the following holds
    \begin{flalign*}
    	\ltwo{Q^i_{\pi_{w_t}}-\bar{Q}^i_t}=\mathcal{O}\left(\frac{\log(\lone{\mcs}^2\lone{\mca}^2K_{\text{in}}^2/\delta)}{(1-\gamma)K_{\text{in}}^{\sigma/2}}\right).
    \end{flalign*}
    Thus, if we let
    \begin{flalign*}
    	K_{\text{in}}=\Theta\left( \left(\frac{T}{(1-\gamma)^2\lone{\mcs}\lone{\mca}}\right)^{\frac{1}{\sigma}} \log^{\frac{2}{\sigma}}\left( \frac{T^{\frac{2}{\sigma}+1}}{\delta(1-\gamma)^{\frac{2}{\sigma}}\lone{\mcs}^{\frac{2}{\sigma}-2}\lone{\mca}^{\frac{2}{\sigma}-2}} \right) \right),
    \end{flalign*}
    then with probability at least $1-\delta/T$, we have
    \begin{flalign}
    	\ltwo{Q^i_{\pi_{w_t}}-\bar{Q}^i_t}\leq\frac{\sqrt{(1-\gamma)\lone{\mcs}\lone{\mca}}}{\sqrt{T}}.\label{eq: m5}
    \end{flalign}
    Applying the union bound to \cref{eq: m5} from $t=0$ to $T-1$, we have with probability at least $1-\delta$ the following holds
    \begin{flalign}
    	\sum_{i=0}^{p}\sum_{t\in\gN_i}\ltwo{Q^i_{\pi_{w_t}}-\bar{Q}^i_t}\leq \sqrt{(1-\gamma)\lone{\mcs}\lone{\mca}T},\label{eq: even1}
    \end{flalign}
    which further implies that, with probability at least $1-\delta$, we have
    \begin{flalign*}
    	&\alpha\sum_{t\in\gN_0}(J_0(\pi^*)-J_0(\pi_{w_t})) + \alpha\eta\sum_{i=1}^{p}\lone{\gN_i} \nonumber\\
    	&\leq \sE_{s\sim\nu^*} D_{\text{KL}}(\pi^*||\pi_{w_0}) + \frac{2\alpha^2 c^2_{\max}\lone{\mcs}\lone{\mca}T}{(1-\gamma)^3} + \frac{\alpha\sqrt{\lone{\mcs}\lone{\mca}T}(2+(1-\gamma)^2+3\alpha c_{\max})}{(1-\gamma)^{1.5}},
    \end{flalign*}
    which completes the proof.
\end{proof}

\begin{lemma}\label{lemma_m1}
	If
	\begin{flalign}
		\frac{1}{2}\alpha\eta T\geq \sE_{s\sim\nu^*} D_{\text{KL}}(\pi^*||\pi_{w_0}) + \frac{2\alpha^2 c^2_{\max}\lone{\mcs}\lone{\mca}T}{(1-\gamma)^3} + \frac{\alpha\sqrt{\lone{\mcs}\lone{\mca}T}(2+(1-\gamma)^2+3\alpha c_{\max})}{(1-\gamma)^{1.5}},\label{eq: m6}
	\end{flalign}
    then with probability at least $1-\delta$, we have the following holds
    \begin{enumerate}
    	\item $\gN_0\neq \emptyset$, i.e., $w_{\text{out}}$ is well-defined,
    	\item One of the following two statements must hold,
    	\begin{enumerate}
    		\item $\lone{\gN_0}\geq T/2$,
    		\item $\sum_{t\in\gG}(J_0(\pi^*)-J_0(w_t))\leq 0$.
    	\end{enumerate}
    \end{enumerate}
\end{lemma}
\begin{proof}
	We prove \Cref{lemma_m1} in the event that \cref{eq: even1} holds, which happens with probability at least $1-\delta$.
    Under such an event, the following inequality holds, which is also the result of \Cref{lemma_m2}.
	\begin{flalign}
		&\alpha\sum_{t\in\gN_0}(J_0(\pi^*)-J_0(\pi_{w_t})) + \alpha\eta\sum_{i=1}^{p}\lone{\gN_i} \nonumber\\
		&\leq \sE_{s\sim\nu^*} D_{\text{KL}}(\pi^*||\pi_{w_0}) + \frac{2\alpha^2 c^2_{\max}\lone{\mcs}\lone{\mca}T}{(1-\gamma)^3} + \frac{\alpha\sqrt{\lone{\mcs}\lone{\mca}T}(2+(1-\gamma)^2+2\alpha c_{\max})}{(1-\gamma)^{1.5}}.\label{eq: 8}
	\end{flalign}
	We first verify item 1. If $\gN_0= \emptyset$, then $\sum_{i=1}^{p}\lone{\gN_i} = T$, and \cref{eq: 8} implies that
	\begin{flalign*}
		\alpha\eta T\leq \sE_{s\sim\nu^*} D_{\text{KL}}(\pi^*||\pi_{w_0}) + \frac{2\alpha^2 c^2_{\max}\lone{\mcs}\lone{\mca}T}{(1-\gamma)^3} + \frac{\alpha\sqrt{\lone{\mcs}\lone{\mca}T}(2+(1-\gamma)^2+2\alpha c_{\max})}{(1-\gamma)^{1.5}},
	\end{flalign*}
    which contradicts \cref{eq: m6}. Thus, we must have $\gN_0\neq \emptyset$.
    
    We then proceed to verify item 2. If $\sum_{t\in\gN_0}(J_0(\pi^*)-J_0(w_t))\leq 0$, then (b) in item 2 holds. If $\sum_{t\in\gN_0}(J_0(\pi^*)-J_0(w_t))\leq 0$, then \cref{eq: 8} implies that
    \begin{flalign*}
    	\alpha\eta\sum_{i=1}^{p}\lone{\gN_i}\leq \sE_{s\sim\nu^*} D_{\text{KL}}(\pi^*||\pi_{w_0}) + \frac{2\alpha^2 c^2_{\max}\lone{\mcs}\lone{\mca}T}{(1-\gamma)^3} + \frac{\alpha\sqrt{\lone{\mcs}\lone{\mca}T}(2+(1-\gamma)^2+3\alpha c_{\max})}{(1-\gamma)^{1.5}}.
    \end{flalign*}
    Suppose that $\lone{\gN_0}<T/2$, i.e., $\sum_{i=1}^{p}\lone{\gN_i}\geq T/2$. Then,
    \begin{flalign*}
    	\frac{1}{2}\alpha\eta T&\leq \alpha\eta\sum_{i=1}^{p}\lone{\gN_i}\nonumber\\
    	&\leq \sE_{s\sim\nu^*} D_{\text{KL}}(\pi^*||\pi_{w_0}) + \frac{2\alpha^2 c^2_{\max}\lone{\mcs}\lone{\mca}T}{(1-\gamma)^3} + \frac{\alpha\sqrt{\lone{\mcs}\lone{\mca}T}(2+(1-\gamma)^2+3\alpha c_{\max})}{(1-\gamma)^{1.5}},
    \end{flalign*}
    which contradicts \cref{eq: m6}. Hence, (a) in item 2 holds.
\end{proof}

\subsection{Proof of \Cref{thm1}}
We restate \Cref{thm1} as follows to include the specifics of the parameters.
\begin{theorem}[Restatement of \Cref{thm1}]\label{thm3}
	Consider \Cref{algorithm_cpg} in the tabular setting. Let $\alpha = (1-\gamma)^{1.5}/\sqrt{\lone{\mcs}\lone{\mca}T}$, $\eta=\frac{2\sqrt{\lone{\mcs}\lone{\mca}}}{(1-\gamma)^{1.5}\sqrt{T}}(3 + \sE_{s\sim\nu^*} D_{\text{KL}}(\pi^*||\pi_{w_0})+3c_{\max}+c^2_{\max})$, and
	\begin{flalign*}
		K_{\text{in}}=\Theta\left( \left(\frac{T}{(1-\gamma)\lone{\mcs}\lone{\mca}}\right)^{\frac{1}{\sigma}} \log^{\frac{2}{\sigma}}\left( \frac{T^{\frac{2}{\sigma}+1}}{\delta(1-\gamma)^{\frac{2}{\sigma}}\lone{\mcs}^{\frac{2}{\sigma}-2}\lone{\mca}^{\frac{2}{\sigma}-2}} \right) \right).
	\end{flalign*}
	 Suppose the same setting for policy evaluation in \Cref{lemma: tabularTD} hold. Then, with probability at least $1-\delta$, we have
	\begin{flalign*}
		J_0(\pi^*)-\sE[J_0(w_{\text{out}})]=\frac{2\sqrt{\lone{\mcs}\lone{\mca}}}{(1-\gamma)^{1.5}\sqrt{T}}\left(\sE_{s\sim\nu^*} D_{\text{KL}}(\pi^*||\pi_{w_0}) + 3 + 2 c^2_{\max} + 3c_{\max}\right),
	\end{flalign*}
	and for all $i\in\{1,\cdots,p\}$, we have
	\begin{flalign*}
		\sE[J_i(\pi_{w_{\text{out}}})]-d_i\leq \frac{2\sqrt{\lone{\mcs}\lone{\mca}}}{(1-\gamma)^{1.5}\sqrt{T}}(3 + \sE_{s\sim\nu^*} D_{\text{KL}}(\pi^*||\pi_{w_0})+3c_{\max}+c^2_{\max}) + \frac{2\sqrt{(1-\gamma)\lone{\mcs}\lone{\mca}}}{\sqrt{T}}.
	\end{flalign*}
	
\end{theorem}

	To prove \Cref{thm1} (or \Cref{thm3}), we still consider the following event given in \cref{eq: even1} that happens with probability at least $1-\delta$:
	\begin{flalign*}
		\sum_{i=0}^{p}\sum_{t\in\gN_i}\ltwo{Q^i_{\pi_{w_t}}-\bar{Q}^i_t}\leq \sqrt{(1-\gamma)\lone{\mcs}\lone{\mca}T},
	\end{flalign*}
    which implies
	\begin{flalign*}
		&\alpha\sum_{t\in\gN_0}(J_0(\pi^*)-J_0(\pi_{w_t})) + \alpha\eta\sum_{i=1}^{p}\lone{\gN_i} \nonumber\\
		&\leq \sE_{s\sim\nu^*} D_{\text{KL}}(\pi^*||\pi_{w_0}) + \frac{2\alpha^2 c^2_{\max}\lone{\mcs}\lone{\mca}T}{(1-\gamma)^3} + \frac{\alpha\sqrt{\lone{\mcs}\lone{\mca}T}(2+(1-\gamma)^2+3\alpha c_{\max})}{(1-\gamma)^{1.5}}.
	\end{flalign*}
    We first consider the convergence rate of the objective function. Under the above event, the following holds
    \begin{flalign*}
    	&\alpha\sum_{t\in\gN_0}(J_0(\pi^*)-J_0(\pi_{w_t})) \nonumber\\
    	&\leq \sE_{s\sim\nu^*} D_{\text{KL}}(\pi^*||\pi_{w_0}) + \frac{2\alpha^2 c^2_{\max}\lone{\mcs}\lone{\mca}T}{(1-\gamma)^3} + \frac{\alpha\sqrt{\lone{\mcs}\lone{\mca}T}(2+(1-\gamma)^2+3\alpha c_{\max})}{(1-\gamma)^{1.5}}.
    \end{flalign*}
    If $\sum_{t\in\gN_0}(J_0(\pi^*)-J_0(\pi_{w_t}))\leq 0$, then we have $J_0(\pi^*)-J_0(\pi_{w_{\text{out}}})\leq 0$. If $\sum_{t\in\gN_0}(J_0(\pi^*)-J_0(\pi_{w_t}))\geq 0$, we have $\lone{\gN_0}\geq T/2$, which implies the following convergence rate
    \begin{flalign*}
    	&J_0(\pi^*)-\sE[J_0(\pi_{w_{\text{out}}})]\nonumber\\
    	&=\frac{1}{\lone{\gN_0}}\sum_{t\in\gN_0}(J_0(\pi^*)-J_0(\pi_{w_t}))\nonumber\\
    	&\leq \frac{2}{\alpha T}\sE_{s\sim\nu^*} D_{\text{KL}}(\pi^*||\pi_{w_0}) + \frac{4\alpha c^2_{\max}\lone{\mcs}\lone{\mca}}{(1-\gamma)^3} + \frac{2\sqrt{\lone{\mcs}\lone{\mca}}(2+(1-\gamma)^2+3\alpha c_{\max})}{(1-\gamma)^{1.5}\sqrt{T}}\nonumber\\
    	&\leq \frac{\sqrt{\lone{\mcs}\lone{\mca}}}{(1-\gamma)^{1.5}\sqrt{T}}\left(2\sE_{s\sim\nu^*} D_{\text{KL}}(\pi^*||\pi_{w_0}) + 6 + 4 c^2_{\max} + 6c_{\max}\right).
    \end{flalign*}
    We then proceed to bound the constrains violation. For any $i\in\{1,\cdots,p\}$, we have
    \begin{flalign*}
    	\sE[J_i(\pi_{w_{\text{out}}})]-d_i&=\frac{1}{\lone{\gN_0}}\sum_{t\in\gN_0}J_i(\pi_{w_t}) - d_i\nonumber\\
    	&\leq \frac{1}{\lone{\gN_0}}\sum_{t\in\gN_0}(\bar{J}_i(\theta^i_t) - d_i) + \frac{1}{\lone{\gN_0}}\sum_{t\in\gN_0} \lone{J_i(\pi_{w_t})-\bar{J}_i(\theta^i_t)} \nonumber\\
    	&\leq \eta + \frac{1}{\lone{\gN_0}}\sum_{t=0}^{T-1} \lone{J_i(\pi_{w_t})-\bar{J}_i(\theta^i_t)} \nonumber\\
    	&\leq \eta + \frac{1}{\lone{\gN_0}}\sum_{i=0}^{p}\sum_{t\in\gN_i}\ltwo{Q^i_{\pi_{w_t}}-\bar{Q}^i_t} \nonumber\\
    	&\leq \eta + \frac{2}{T}\sum_{i=0}^{p}\sum_{t\in\gN_i}\ltwo{Q^i_{\pi_{w_t}}-\bar{Q}^i_t}.
    \end{flalign*}
    Under the event defined in \cref{eq: even1}, we have $\sum_{i=0}^{p}\sum_{t\in\gN_i}\ltwo{Q^i_{\pi_{w_t}}-\bar{Q}^i_t}\leq \sqrt{(1-\gamma)\lone{\mcs}\lone{\mca}T}$. Recall the value of the tolerance $\eta=\frac{2\sqrt{\lone{\mcs}\lone{\mca}}}{(1-\gamma)^{1.5}\sqrt{T}}(3 + \sE_{s\sim\nu^*} D_{\text{KL}}(\pi^*||\pi_{w_0})+3c_{\max}+c^2_{\max})$. With probability at least $1-\delta$, we have
    \begin{flalign*}
    	\sE[J_i(\pi_{w_{\text{out}}})]-d_i\leq \frac{2\sqrt{\lone{\mcs}\lone{\mca}}}{(1-\gamma)^{1.5}\sqrt{T}}(3 + \sE_{s\sim\nu^*} D_{\text{KL}}(\pi^*||\pi_{w_0})+3c_{\max}+c^2_{\max}) + \frac{2\sqrt{(1-\gamma)\lone{\mcs}\lone{\mca}}}{\sqrt{T}}.
    \end{flalign*}

\section{Proof of \Cref{lemma: neuralTD} and \Cref{thm2}: Function Approximation Setting}
For notation simplicity, we denote the state action pairs $(s,a)$ and $(s^\prime,a^\prime)$ as $x$ and $x^\prime$, respectively. We define the weighted norm $\ld{f} = \sqrt{\int f(x)^2d\mathcal{D}(x)}$ for any distribution $\mathcal{D}$ over $\lone{\mcs}\times\lone{\mca}$. We will write $\theta^i_k$ as $\theta_k$ whenever there is no confusion in this subsection. We define
\begin{flalign*}
	f_0(x,\theta)=\frac{1}{\sqrt{m}}\sum_{r=1}^{m}b_r\mdsone(\theta_{0,r}^\top\psi(x)>0)\theta_r^\top \psi(x)
\end{flalign*}
as the local linearizion of $f(x,\theta)$ at the initial point $\theta_0$.
We denote the temporal differences as $\delta_0(x,x^\prime.\theta_k)=f_0((s^\prime,a^\prime);\theta_k)-\gamma f_0((s,a);\theta_k) - r(s,a,s^\prime)$ and $\delta_k(x,x^\prime.\theta_k)=f((s^\prime,a^\prime);\theta_k)-\gamma f((s,a);\theta_k) - r(s,a,s^\prime)$. 
We define the stochastic semi-gradient $g_k(\theta_k)=\delta_k(x_k,x^\prime_k.\theta_k)\nabla_\theta f(x_k,\theta_k)$, and the full semi-gradients $\bar{g}_0(\theta_k)=\sE_{\mu_\pi}[\delta_0(x,x^\prime.\theta_k)\nabla_\theta f_0(x,\theta_k)]$, and $\bar{g}_k(\theta_k)=\sE_{\mu_\pi}[\delta_k(x,x^\prime.\theta_k)\nabla_\theta f(x,\theta_k)]$.
The approximated stationary point $\theta^*$ satisfies $\bar{g}_0(\theta)^\top(\theta-\theta^*)\geq$ for any $\theta\in\mB$. We define the following function spaces
\begin{flalign*}
	\gF_{0,m}=\left\{ \frac{1}{\sqrt{m}}\sum_{r=1}^{m}b_r\mdsone(\theta_{0,r}^\top\psi(x)>0)\theta_r^\top \psi(x):\ltwo{\theta-\theta_0}\leq R \right\},
\end{flalign*}
and
\begin{flalign*}
	\overline{\gF}_{0,m}=\left\{ \frac{1}{\sqrt{m}}\sum_{r=1}^{m}b_r\mdsone(\theta_{0,r}^\top\psi(s)>0)\theta_r^\top \psi(x):\linf{\theta_r-\theta_{0,r}}\leq R/\sqrt{md} \right\},
\end{flalign*}
and define $f_0(x,\theta^*_\pi)$ as the projection of $Q_\pi(x)$ onto the function space $\gF_{0,m}$ in terms of $\lmupi{\cdot}$norm. Without loss of generality, we assume $0<\delta<\frac{1}{e}$ in the sequel. 

\subsection{Supporting Lemmas for Proof of \Cref{lemma: neuralTD}}
We provide the proof of supporting lemmas for \Cref{lemma: neuralTD}. 
\begin{lemma}[\cite{rahimi2009weighted}]\label{lemma: b4}
	Let $f\in\gF_{0,\infty}$, where $\gF_{0,\infty}$ is defined in Assumption \ref{ass2}. For any $\delta>0$, it holds with probability at least $1-\delta$ that
	\begin{flalign*}
		\ld{\Pi_{\overline{\gF}_{0,m}}f-f}^2\leq \frac{4R^2\log(\frac{1}{\delta})}{m},
	\end{flalign*}
    where $\mathcal{D}$ is any distribution over $\mcs\times\mca$.
\end{lemma}

For the following \Cref{lemma: b1} and \Cref{lemma: b2}, we provide slightly different proofs from those in \cite{cai2019neural}, which are included here for completeness.
\begin{lemma}\label{lemma: b1}
	Suppose Assumption \ref{ass1} holds. For any policy $\pi$ and all $k\geq 0$, it holds that
	\begin{flalign*}
		\sE_{\mu_\pi}\left[ \frac{1}{m}\sum_{r=1}^{m}\lone{\mdsone\left( \theta^\top_{k,r}\psi(x)>0\right) - \mdsone\left( \theta^\top_{0,r}\psi(x)>0\right)} \right]\leq \frac{C_0R}{d_1\sqrt{m}}.
	\end{flalign*}
\end{lemma}
\begin{proof}
	Note that $\mdsone\left( \theta^\top_{k,r}\psi(x)>0\right) \neq \mdsone\left( \theta^\top_{0,r}\psi(x)>0\right)$ implies
	\begin{flalign*}
		\lone{\theta^\top_{0,r}\psi(x)}\leq \lone{\theta^\top_{k,r}\psi(x)-\theta^\top_{0,r}\psi(x)}\leq \ltwo{\theta_{k,r}-\theta_{0,r}},
	\end{flalign*}
    which further implies
    \begin{flalign}
    	\lone{\mdsone\left( \theta^\top_{k,r}\psi(x)>0\right) - \mdsone\left( \theta^\top_{0,r}\psi(x)>0\right)}\leq \mdsone(\lone{\theta^\top_{0,r}\psi(x)}\leq \ltwo{\theta_{k,r}-\theta_{0,r}}).\label{eq: b1}
    \end{flalign}
    Then, we can derive the following upper bound
    \begin{flalign}
    	&\sE_{\mu_\pi}\left[ \frac{1}{m}\sum_{r=1}^{m}\lone{\mdsone\left( \theta^\top_{k,r}\psi(x)>0\right) - \mdsone\left( \theta^\top_{0,r}\psi(x)>0\right)} \right]\nonumber\\
    	&\leq \sE_{\mu_\pi}\left[ \frac{1}{m}\sum_{r=1}^{m}\mdsone(\lone{\theta^\top_{0,r}\psi(x)}\leq \ltwo{\theta_{k,r}-\theta_{0,r}}) \right]\label{eq: b2}\\
    	&=  \frac{1}{m}\sum_{r=1}^{m}\mP_{\mu_\pi}(\lone{\theta^\top_{0,r}\psi(x)}\leq \ltwo{\theta_{k,r}-\theta_{0,r}})\nonumber\\
    	&\overset{(i)}{\leq} \frac{C_0}{m}\sum_{r=1}^{m}\frac{\ltwo{\theta_{k,r}-\theta_{0,r}}}{\ltwo{\theta_{0,r}}}\nonumber\\
    	&\leq \frac{C_0}{m}\left( \sum_{r=1}^{m}\ltwo{\theta_{k,r}-\theta_{0,r}}^2 \right)^{1/2}\left( \sum_{r=1}^{m}\frac{1}{\ltwo{\theta_{0,r}}^2} \right)^{1/2}\nonumber\\
    	&\overset{(ii)}{\leq} \frac{C_0R}{d_1\sqrt{m}}.
    \end{flalign}
    where $(i)$ follows from Assumption \ref{ass1} and $(ii)$ follows from the fact that $\ltwo{\theta_{0,r}}\geq d_1.$
\end{proof}

\begin{lemma}\label{lemma: b2}
	Suppose Assumption \ref{ass1} holds. For any policy $\pi$ and all $k\geq 0$, it holds that
	\begin{flalign*}
		\sE_{\mu_\pi}\left[ \lone{f((s,a);\theta_k) - f_0((s,a);\theta_k)}^2 \right]\leq \frac{4C_0R^3}{d_1\sqrt{m}}.
	\end{flalign*}
\end{lemma}
\begin{proof}
	By definition, we have
	\begin{flalign}
		&\lone{f((s,a);\theta_t) - f_0((s,a);\theta_t)}\nonumber\\
		&=\frac{1}{\sqrt{m}}\lone{\sum_{r=1}^{m}\left( \mdsone(\theta_{k,r}^\top\psi(x)>0) - \mdsone(\theta_{0,r}^\top\psi(x)>0) \right)b_r\theta_{k,r}^\top\psi(x)}\nonumber\\
		&\leq\frac{1}{\sqrt{m}}\sum_{r=1}^{m}\lone{\left( \mdsone(\theta_{k,r}^\top\psi(x)>0) - \mdsone(\theta_{0,r}^\top\psi(x)>0) \right)}\lone{b_r}\ltwo{\theta_{k,r}^\top\psi(x)}\nonumber\\
		&\overset{(i)}{\leq}\frac{1}{\sqrt{m}}\sum_{r=1}^{m}\mdsone(\lone{\theta^\top_{0,r}\psi(x)}\leq \ltwo{\theta_{k,r}-\theta_{0,r}})\ltwo{\theta_{k,r}^\top\psi(x)}\nonumber\\
		&\leq \frac{1}{\sqrt{m}}\sum_{r=1}^{m}\mdsone(\lone{\theta^\top_{0,r}\psi(x)}\leq \ltwo{\theta_{k,r}-\theta_{0,r}})\left(\ltwo{\theta_{0,r}-\theta_{k,r}}+\ltwo{\theta_{0,r}^\top\psi(x)}\right)\nonumber\\
		&\leq \frac{2}{\sqrt{m}}\sum_{r=1}^{m}\mdsone(\lone{\theta^\top_{0,r}\psi(x)}\leq \ltwo{\theta_{k,r}-\theta_{0,r}})\ltwo{\theta_{0,r}-\theta_{k,r}}.
	\end{flalign}
    where $(i)$ follows from \cref{eq: b1}. We can then obtain the following upper bound.
    \begin{flalign}
    	&\sE_{\mu_\pi}\left[ \lone{f((s,a);\theta_t) - f_0((s,a);\theta_t)}^2 \right]\nonumber\\
    	&\leq \frac{4}{m}\sE_{\mu_\pi}\left[ \left( \sum_{r=1}^{m}\mdsone(\lone{\theta^\top_{0,r}\psi(x)}\leq \ltwo{\theta_{k,r}-\theta_{0,r}})\ltwo{\theta_{0,r}-\theta_{k,r}} \right)^2\right]\nonumber\\
    	&\overset{(i)}{\leq} \frac{4}{m}\sE_{\mu_\pi}\left[ \sum_{r=1}^{m}\mdsone(\lone{\theta^\top_{0,r}\psi(x)}\leq \ltwo{\theta_{k,r}-\theta_{0,r}})\sum_{r=1}^{m}\ltwo{\theta_{0,r}-\theta_{k,r}}^2 \right]\nonumber\\
    	&=\frac{4R^2}{m}\sE_{\mu_\pi}\left[ \sum_{r=1}^{m}\mdsone(\lone{\theta^\top_{0,r}\psi(x)}\leq \ltwo{\theta_{k,r}-\theta_{0,r}}) \right]\nonumber\\
    	&\overset{(ii)}{\leq} \frac{4C_0R^3}{d_1\sqrt{m}},
    \end{flalign}
    where $(i)$ follows from Holder's inequality, and $(ii)$ follows from the derivation in \Cref{lemma: b1} after \cref{eq: b2}.
\end{proof}

\begin{lemma}\label{lemma: b3}
		Suppose Assumption \ref{ass1} holds. For any policy $\pi$ and all $k\geq 0$, with probability at least $1-\delta$, we have
		\begin{flalign*}
			\ltwo{\bar{g}_k(\theta_k)-\bar{g}_0(\theta_k)}\leq \Theta\left(\frac{\sqrt{\log(\frac{1}{\delta})}}{(1-\gamma)m^{1/4}}\right).
		\end{flalign*}
\end{lemma}
\begin{proof}
	By definition, we have
	\begin{flalign}
		&\ltwo{\bar{g}_k(\theta_k)-\bar{g}_0(\theta_k)}\nonumber\\
		&=\ltwo{\sE_{\mu_\pi}[\delta_k(x,x^\prime.\theta_k)\nabla_\theta f(x,\theta_k)] - \sE_{\mu_\pi}[\delta_0(x,x^\prime.\theta_k)\nabla_\theta f_0(x,\theta_k)]}\nonumber\\
		&=\ltwo{\sE_{\mu_\pi}[\left(\delta_k(x,x^\prime.\theta_k)-\delta_0(x,x^\prime.\theta_k)\right)\nabla_\theta f(x,\theta_k) + \delta_0(x,x^\prime.\theta_k)\left(\nabla_\theta f(x,\theta_k)-\nabla_\theta f_0(x,\theta_k)\right)]}\nonumber\\
		&\leq\sE_{\mu_\pi}[\lone{\delta_k(x,x^\prime.\theta_k)-\delta_0(x,x^\prime.\theta_k)}\ltwo{\nabla_\theta f(x,\theta_k)} + \lone{\delta_0(x,x^\prime.\theta_k)} \ltwo{\nabla_\theta f(x,\theta_k)-\nabla_\theta f_0(x,\theta_k)} ]\nonumber\\
		&\overset{(i)}{\leq}\sE_{\mu_\pi}[\lone{\delta_k(x,x^\prime.\theta_k)-\delta_0(x,x^\prime.\theta_k)}] + \sE_{\mu_\pi}[\lone{\delta_0(x,x^\prime.\theta_k)} \ltwo{\nabla_\theta f(x,\theta_k)-\nabla_\theta f_0(x,\theta_k)} ],\label{eq: b3}
	\end{flalign}
    where $(i)$ follows from the fact that $\ltwo{\nabla_\theta f(x,\theta_k)}\leq 1$. Then, \cref{eq: b3} implies that
    \begin{flalign}
    	&\ltwo{\bar{g}_k(\theta_k)-\bar{g}_0(\theta_k)}^2\nonumber\\
    	&\leq 2\sE_{\mu_\pi}[\lone{\delta_k(x,x^\prime.\theta_k)-\delta_0(x,x^\prime.\theta_k)}^2] + 2\left(\sE_{\mu_\pi}[\lone{\delta_0(x,x^\prime.\theta_k)} \ltwo{\nabla_\theta f(x,\theta_k)-\nabla_\theta f_0(x,\theta_k)} ]\right)^2\nonumber\\
    	&\leq 2\sE_{\mu_\pi}[\lone{\delta_k(x,x^\prime.\theta_k)-\delta_0(x,x^\prime.\theta_k)}^2] + 2\sE_{\mu_\pi}[\lone{\delta_0(x,x^\prime.\theta_k)}^2]\sE_{\mu_\pi}[\ltwo{\nabla_\theta f(x,\theta_k)-\nabla_\theta f_0(x,\theta_k)}^2 ].\label{eq: b4}
    \end{flalign}
    We first upper bound the term $\sE_{\mu_\pi}[\lone{\delta_k(x,x^\prime.\theta_k)-\delta_0(x,x^\prime.\theta_k)}^2]$. By definition, we have
    \begin{flalign*}
    	&\lone{\delta_k(x,x^\prime.\theta_k)-\delta_0(x,x^\prime.\theta_k)}\nonumber\\
    	&=\lone{f(x,\theta_k)-f_0(x,\theta_k) - \gamma(f(x^\prime,\theta_k)-f_0(x^\prime,\theta_k))}\nonumber\\
    	&\leq \lone{f(x,\theta_k)-f_0(x,\theta_k)} + \lone{f(x^\prime,\theta_k)-f_0(x^\prime,\theta_k)},
    \end{flalign*}
    which implies
    \begin{flalign}
    	&\sE_{\mu_\pi}[\lone{\delta_k(x,x^\prime.\theta_k)-\delta_0(x,x^\prime.\theta_k)}^2]\nonumber\\
    	&\leq 2\sE_{\mu_\pi}[\lone{f(x,\theta_k)-f_0(x,\theta_k)}^2] + 2\sE_{\mu_\pi}[\lone{f(x^\prime,\theta_k)-f_0(x^\prime,\theta_k)}^2]\nonumber\\
    	&=4\sE_{\mu_\pi}[\lone{f(x,\theta_k)-f_0(x,\theta_k)}^2]\nonumber\\
    	&\overset{(i)}{\leq} \frac{16C_0R^2}{d_1\sqrt{m}},\label{eq: b6}
    \end{flalign}
    where $(i)$ follows from \Cref{lemma: b2}.
    We then proceed to bound the term $\sE_{\mu_\pi}[\ltwo{\nabla_\theta f(x,\theta_k)-\nabla_\theta f_0(x,\theta_k)}^2 ]$. By definition, we have
    \begin{flalign}
    	&\ltwo{\nabla_\theta f(x,\theta_k)-\nabla_\theta f_0(x,\theta_k)}\nonumber\\
    	&=\frac{1}{\sqrt{m}}\ltwo{\sum_{r=1}^{m}\left[\mdsone\left( \theta^\top_{k,r}\psi(x)>0\right) - \mdsone\left( \theta^\top_{0,r}\psi(x)>0\right)\right]b_r\theta^\top_{0,r}\psi(x)}\nonumber\\
    	&\overset{(i)}{\leq}\frac{1}{\sqrt{m}}\sum_{r=1}^{m}\lone{\mdsone\left( \theta^\top_{k,r}\psi(x)>0\right) - \mdsone\left( \theta^\top_{0,r}\psi(x)>0\right)}\ltwo{\theta_{0,r}}\nonumber\\
    	&\overset{(ii)}{\leq}\frac{1}{\sqrt{m}}\sum_{r=1}^{m}\mdsone(\lone{\theta^\top_{0,r}\psi(x)}\leq \ltwo{\theta_{k,r}-\theta_{0,r}})\ltwo{\theta_{0,r}},\label{eq: b5}
    \end{flalign}
    where $(i)$ follows because $\lone{b_r}\leq 1$ and $\ltwo{\psi(s)}\leq 1$, and $(ii)$ follows from \cref{eq: b1}. Further, \cref{eq: b5} implies that
    \begin{flalign}
    	&\sE_{\mu_\pi}[\ltwo{\nabla_\theta f(x,\theta_k)-\nabla_\theta f_0(x,\theta_k)}^2 ]\nonumber\\
    	&\leq \frac{1}{m}\sE_{\mu_\pi}\left[\left(\sum_{r=1}^{m}\mdsone(\lone{\theta^\top_{0,r}\psi(x)}\leq \ltwo{\theta_{k,r}-\theta_{0,r}})\right)\left(\sum_{r=1}^{m}\ltwo{\theta_{0,r}}^2\right)\right]\nonumber\\
    	&\leq \frac{R^2}{m}\sum_{r=1}^{m}\mdsone(\lone{\theta^\top_{0,r}\psi(x)}\leq \ltwo{\theta_{k,r}-\theta_{0,r}})\nonumber\\
    	&\overset{(i)}{\leq} \frac{C_0R^3}{d_1\sqrt{m}},\label{eq: b7}
    \end{flalign}
    where $(i)$ follows from the derivation in \Cref{lemma: b1} after \cref{eq: b2}.
    
    Finally, we upper-bound $\sE_{\mu_\pi}[\lone{\delta_0(x,x^\prime.\theta_k)}^2]$. We proceed as follows.
    \begin{flalign}
    	&\sE_{\mu_\pi}[\lone{\delta_0(x,x^\prime.\theta_k)}^2]\nonumber\\
    	&\leq \sE_{\mu_\pi}[\lone{f_0(x,\theta_k)-r(x,x^\prime)-\gamma f_0(x^\prime,\theta_k)}^2]\nonumber\\
    	&\leq 3\sE_{\mu_\pi}[\lone{f_0(x,\theta_k)}^2]+3\sE_{\mu_\pi}[r^2(x,x^\prime)]+ 3\gamma^2\sE_{\mu_\pi}[\ltwo{f_0(x^\prime,\theta_k)}^2]\nonumber\\
    	&\leq 6\sE_{\mu_\pi}[\lone{f_0(x,\theta_k)}^2]+3c^2_{\max}\nonumber\\
    	&= 6\sE_{\mu_\pi}[\lone{f_0(x,\theta_k)-f_0(x,\theta^*_\pi)+f_0(x,\theta^*_\pi)-Q_\pi(x) + Q_\pi(x) }^2]+3c^2_{\max}\nonumber\\
    	&= 18\sE_{\mu_\pi}[\lone{f_0(x,\theta_k)-f_0(x,\theta^*_\pi)}^2] + 18\sE_{\mu_\pi}[\lone{f_0(x,\theta^*_\pi)-Q_\pi(x)}^2] + 18\sE_{\mu_\pi}[\lone{Q_\pi(x) }^2]+3c^2_{\max}\nonumber\\
    	&\overset{(i)}{\leq} 18R^2 + \frac{21 c^2_{\max}}{(1-\gamma)^2} + 18\sE_{\mu_\pi}[\lone{f_0(x,\theta^*_\pi)-Q_\pi(x)}^2],\label{eq: b8}
    \end{flalign}
    where $(i)$ follows from the fact that $Q_{\pi}(x)\leq \frac{c_{\max}}{1-\gamma}$, $\ltwo{\theta_k}\leq R$ and $\ltwo{\theta^*_\pi}\leq R$.
    
    Since $\overline{\gF}_{0,m}\subset\gF_{0,m}$. \Cref{lemma: b4} implies that with probability at least $1-\delta$, we have
    \begin{flalign}
    	\sE_{\mu_\pi}[\lone{f_0(x,\theta^*_\pi)-Q_\pi(x)}^2]\leq \frac{4R^2\log\left(\frac{1}{\delta}\right)}{m}\leq 4R^2\log\left(\frac{1}{\delta}\right).\label{eq: b18}
    \end{flalign}
    Thus, with probability at least $1-\delta$, we have
    \begin{flalign}
    	\sE_{\mu_\pi}[\lone{\delta_0(x,x^\prime.\theta_k)}^2]\leq 18R^2 + \frac{21 c^2_{\max}}{(1-\gamma)^2}+72R^2\log\left(\frac{1}{\delta}\right).\label{eq: b9}
    \end{flalign}
    Combining \cref{eq: b6}, \cref{eq: b7} and \cref{eq: b9}, we can obtain that, with probability at least $1-\delta$, we have
    \begin{flalign*}
    	\ltwo{\bar{g}_k(\theta_k)-\bar{g}_0(\theta_k)}^2\leq \Theta\left(\frac{\log(\frac{1}{\delta})}{(1-\gamma)^2\sqrt{m}}\right),
    \end{flalign*}
    which implies that with probability at least $1-\delta$, we have
    \begin{flalign*}
    	\ltwo{\bar{g}_k(\theta_k)-\bar{g}_0(\theta_k)}\leq \Theta\left(\frac{\sqrt{\log(\frac{1}{\delta})}}{(1-\gamma)m^{1/4}}\right),
    \end{flalign*}
    which completes the proof.
\end{proof}

\subsection{Proof of \Cref{lemma: neuralTD}}
We consider the convergence of $\theta^i_k$ for a given $i$ under a fixed policy $\pi$. 
For the iteration of $\theta_k$, we proceed as follows.
\begin{flalign}
	&\ltwo{\theta_{k+1}-\theta^*}^2\nonumber\\
	&=\ltwo{\Pi_\mB(\theta_k-\beta g_k(\theta_k)) - \Pi_\mB(\theta^*-\beta \bar{g}_0(\theta^*))}^2\nonumber\\
	&\leq\ltwo{(\theta_k - \theta^*)-\beta (g_k(\theta_k) - \bar{g}_0(\theta^*))}^2\nonumber\\
	&=\ltwo{\theta_k - \theta^*}^2 - 2\beta(g_k(\theta_k) - \bar{g}_0(\theta^*))^\top(\theta_k-\theta^*) + \beta^2\ltwo{g_k(\theta_k) - \bar{g}_0(\theta^*)}^2\nonumber\\
	&=\ltwo{\theta_k - \theta^*}^2  - 2\beta(\bar{g}_0(\theta_k) - \bar{g}_0(\theta^*))^\top(\theta_k-\theta^*) + 2\beta(\bar{g}_k(\theta_k) - g_k(\theta_k))^\top(\theta_k-\theta^*) \nonumber\\
	&\quad + 2\beta(\bar{g}_0(\theta_k)-\bar{g}_k(\theta_k))^\top(\theta_k-\theta^*) + \beta^2\ltwo{g_k(\theta_k) - \bar{g}_0(\theta^*)}^2\nonumber\\
	&\leq \ltwo{\theta_k - \theta^*}^2  - 2\beta(\bar{g}_0(\theta_k) - \bar{g}_0(\theta^*))^\top(\theta_k-\theta^*) + 2\beta(\bar{g}_k(\theta_k) - g_k(\theta_k))^\top(\theta_k-\theta^*) \nonumber\\
	&\quad + 2\beta(\bar{g}_0(\theta_k)-\bar{g}_k(\theta_k))^\top(\theta_k-\theta^*) + 3\beta^2\ltwo{g_k(\theta_k) - \bar{g}_k(\theta_k)}^2 + 3\beta^2\ltwo{\bar{g}_k(\theta_k) - \bar{g}_0(\theta_k)}^2\nonumber\\
	&\quad + 3\beta^2\ltwo{\bar{g}_0(\theta_k) - \bar{g}_0(\theta^*)}^2\nonumber\\
	&\overset{(i)}{\leq} \ltwo{\theta_k - \theta^*}^2  - 2(1-\gamma)\beta\sE_{\mu_\pi}\left[ (f_0((s,a);\theta_k) - f_0((s,a);\theta^*))^2 \right] \nonumber\\
	&\quad + 2\beta(\bar{g}_k(\theta_k) - g_k(\theta_k))^\top(\theta_k-\theta^*) + 4R\beta \ltwo{\bar{g}_0(\theta_k)-\bar{g}_k(\theta_k)} + 3\beta^2\ltwo{g_k(\theta_k) - \bar{g}_k(\theta_k)}^2 \nonumber\\
	&\quad + 3\beta^2\ltwo{\bar{g}_k(\theta_k) - \bar{g}_0(\theta_k)}^2 + 3\beta^2\ltwo{\bar{g}_0(\theta_k) - \bar{g}_0(\theta^*)}^2\nonumber\\
	&\overset{(ii)}{\leq} \ltwo{\theta_k - \theta^*}^2  - [2\beta(1-\gamma)-12\beta^2]\sE_{\mu_\pi}\left[ (f_0((s,a);\theta_k) - f_0((s,a);\theta^*))^2 \right] \nonumber\\
	&\quad + 2\beta(\bar{g}_k(\theta_k) - g_k(\theta_k))^\top(\theta_k-\theta^*) + 4R\beta\ltwo{\bar{g}_0(\theta_k)-\bar{g}_k(\theta_k)} + 3\beta^2\ltwo{g_k(\theta_k) - \bar{g}_k(\theta_k)}^2 \nonumber\\
	&\quad + 3\beta^2\ltwo{\bar{g}_k(\theta_k) - \bar{g}_0(\theta_k)}^2,\label{eq: n1}
\end{flalign}
where $(i)$ follows from the fact that
\begin{flalign*}
	&(\bar{g}_0(\theta_k) - \bar{g}_0(\theta^*))^\top(\theta_k-\theta^*)\\
	&\geq (1-\gamma)\sE_{\mu_\pi}\left[ (f_0((s,a);\theta_k) - f_0((s,a);\theta^*))^2 \right] - R\ltwo{\bar{g}_k(\theta_k) - \bar{g}_0(\theta_k)},
\end{flalign*}
and $(ii)$ follows from the fact that
\begin{flalign*}
	\ltwo{\bar{g}_0(\theta_k) - \bar{g}_0(\theta^*)}^2\leq 4\sE_{\mu_\pi}\left[ (f_0((s,a);\theta_k) - f_0((s,a);\theta^*))^2 \right].
\end{flalign*}
Rearranging \cref{eq: n1} yields
\begin{flalign}
	&[2\beta(1-\gamma)-12\beta^2]\sE_{\mu_\pi}\left[ (f_0((s,a);\theta_k) - f_0((s,a);\theta^*))^2 \right]\nonumber\\
	&\leq \ltwo{\theta_k - \theta^*}^2 - \ltwo{\theta_{k+1} - \theta^*}^2 + 2\beta(\bar{g}_k(\theta_k) - g_k(\theta_k))^\top(\theta_k-\theta^*) + 4R\beta \ltwo{\bar{g}_0(\theta_k)-\bar{g}_k(\theta_k)} \nonumber\\
	&\quad + 3\beta^2\ltwo{g_k(\theta_k) - \bar{g}_k(\theta_k)}^2 + 3\beta^2\ltwo{\bar{g}_k(\theta_k) - \bar{g}_0(\theta_k)}^2.\label{eq: n2}
\end{flalign}
Taking summation of \cref{eq: n2} over $t=0$ to $K-1$ yields
\begin{flalign}
	&[2\beta(1-\gamma)-12\beta^2]\sum_{t=0}^{K-1}\sE_{\mu_\pi}\left[ (f_0((s,a);\theta_k) - f_0((s,a);\theta^*))^2 \right]\nonumber\\
	&\leq \ltwo{\theta_0-\theta^*}^2 - \ltwo{\theta_K-\theta^*}^2 + 2\beta\sum_{t=0}^{K-1}(\bar{g}_k(\theta_k) - g_k(\theta_k))^\top(\theta_k-\theta^*) + 4R\beta \sum_{t=0}^{K-1}\ltwo{\bar{g}_0(\theta_k)-\bar{g}_k(\theta_k)} \nonumber\\
	&\quad + 3\beta^2\sum_{t=0}^{K-1}\ltwo{g_k(\theta_k) - \bar{g}_k(\theta_k)}^2 + 3\beta^2\sum_{t=0}^{K-1}\ltwo{\bar{g}_k(\theta_k) - \bar{g}_0(\theta_k)}^2\nonumber\\
	&\overset{(i)}{\leq} R^2 + 2\beta\sum_{t=0}^{K-1}\zeta_k(\theta_k)^\top(\theta_k-\theta^*) + 3\beta^2\sum_{t=0}^{K-1}\ltwo{\zeta_k(\theta_k)}^2 + 4R\beta \sum_{t=0}^{K-1}\ltwo{\xi_k(\theta_k)} +  3\beta^2\sum_{t=0}^{K-1}\ltwo{\xi_k(\theta_k)}^2\nonumber,
\end{flalign}
where in $(i)$ we define $\zeta_k(\theta_k)=\bar{g}_k(\theta_k) - g_k(\theta_k)$ and $\xi_k(\theta_k)=\bar{g}_k(\theta_k) - \bar{g}_0(\theta_k)$.

We first consider the term $\sum_{t=0}^{K-1}\ltwo{\zeta_k(\theta_k)}^2$. We proceed as follows.
\begin{flalign}
	&\mP_{\mu_\pi}\left( \sum_{t=0}^{K-1}\ltwo{\zeta_k(\theta_k)}^2\geq (1+\Lambda)C^2_\zeta K \right)\nonumber\\
	&=\mP_{\mu_\pi}\left( \frac{\sum_{t=0}^{K-1}\ltwo{\zeta_k(\theta_k)}^2}{C^2_\zeta K}\geq 1+\Lambda \right)\nonumber\\
	&=\mP_{\mu_\pi}\left( \exp\left(\frac{\sum_{t=0}^{K-1}\ltwo{\zeta_k(\theta_k)}^2}{C^2_\zeta K}\right)\geq \exp(1+\Lambda) \right)\nonumber\\
	&\leq \mP_{\mu_\pi}\left( \frac{1}{K}\sum_{t=0}^{K-1}\exp\left(\frac{\ltwo{\zeta_k(\theta_k)}^2}{C^2_\zeta }\right)\geq \exp(1+\Lambda) \right)\nonumber\\
	&\overset{(i)}{\leq} \frac{1}{K}\sum_{t=0}^{K-1}\sE_{\mu_\pi}\left[\exp\left(\frac{\ltwo{\zeta_k(\theta_k)}^2}{C^2_\zeta }\right)\right]/\exp(1+\Lambda)\nonumber\\
	&\overset{(ii)}{\leq}\exp(-\Lambda),\label{eq: b10}
\end{flalign}
where $(i)$ follows from Markov's inequality, $(ii)$ follows from Assumption \ref{ass3}. Then, \cref{eq: b10} implies that with probability at least $1-\delta_1$, we have
\begin{flalign}
	\sum_{t=0}^{K-1}\ltwo{\zeta_k(\theta_k)}^2\leq \left(1+\log\left(\frac{1}{\delta_1}\right)\right)C^2_\zeta K\leq 2\log\left(\frac{1}{\delta_1}\right)C^2_\zeta K.\label{eq: b11}
\end{flalign}

We then consider the term $\sum_{t=0}^{K-1}\zeta_k(\theta_k)^\top(\theta_k-\theta^*)$.
Note that for any $0\leq k\leq K-1$, we have
\begin{flalign*}
	\lone{\zeta_k(\theta_k)^\top(\theta_k-\theta^*)}^2\leq \ltwo{\zeta_k(\theta_k)}^2\ltwo{\theta_k-\theta^*}^2\leq R^2\ltwo{\zeta_k(\theta_k)}^2,
\end{flalign*}
which implies
\begin{flalign*}
	\sE_{\mu_\pi}\left[ \exp\left( \frac{\lone{\zeta_k(\theta_k)^\top(\theta_k-\theta^*)}^2}{B^2C^2_\zeta} \right) \right]\leq \sE_{\mu_\pi}\left[ \exp\left( \frac{\ltwo{\zeta_k(\theta_k)}^2}{C^2_\zeta} \right) \right]\leq \exp(1).
\end{flalign*}
Applying Bernstein's inequality for martingale \cite{ghadimi2013stochastic}[Lemma 2.3], we can obtain
\begin{flalign*}
	\mP_{\mu_\pi}\left( \lone{\sum_{t=0}^{K-1}\zeta_k(\theta_k)^\top(\theta_k-\theta^*)}\geq \sqrt{2}(1+\Lambda) C_\zeta \sqrt{K} \right)\leq \exp(-\Lambda^2/3),
\end{flalign*}
which implies that with probability at least $1-\delta_2$, we have
\begin{flalign}
	\lone{\sum_{t=0}^{K-1}\zeta_k(\theta_k)^\top(\theta_k-\theta^*)}\leq \sqrt{2}\left(1+\sqrt{3\log\left(\frac{1}{\delta_2}\right)}\right) C_\zeta \sqrt{K}\leq 5C_\zeta \sqrt{\log\left(\frac{1}{\delta_2}\right)}\sqrt{K}.\label{eq: b12}
\end{flalign}
We then consider the terms $\sum_{t=0}^{K-1}\ltwo{\xi_k(\theta_k)}$ and $\sum_{t=0}^{K-1}\ltwo{\xi_k(\theta_k)}^2$. \Cref{lemma: b3} implies that with probability at least $1-\delta_3/K$, we have
\begin{flalign*}
	\ltwo{\xi_k(\theta_k)}\leq \Theta\left(\frac{\sqrt{\log(\frac{K}{\delta_3})}}{(1-\gamma)m^{1/4}}\right).
\end{flalign*}
Applying then union bound we can obtain that with probability at least $1-\delta_3$, we have
\begin{flalign}
	\sum_{t=0}^{K-1}\ltwo{\xi_k(\theta_k)}\leq \Theta\left(\frac{K\sqrt{\log(\frac{K}{\delta_3})}}{(1-\gamma)m^{1/4}}\right).\label{eq: b13}
\end{flalign}
Similarly, we can obtain that with probability at least $1-\delta_3$, we have
\begin{flalign}
	\sum_{t=0}^{K-1}\ltwo{\xi_k(\theta_k)}^2\leq \Theta\left(\frac{K\log(\frac{K}{\delta_3})}{(1-\gamma)^2m^{1/2}}\right).\label{eq: b14}
\end{flalign}
Combining \cref{eq: b11}, \cref{eq: b12}, \cref{eq: b13} and \cref{eq: b14} and applying the union bound, we can obtain that with probability at least $1-(\delta_1+\delta_2+\delta_3+\delta_4)$, we have
\begin{flalign}
		&[2\beta(1-\gamma)-12\beta^2]\sum_{t=0}^{K-1}\sE_{\mu_\pi}\left[ (f_0((s,a);\theta_k) - f_0((s,a);\theta^*))^2 \right]\nonumber\\
		&\leq R^2 + 10\beta C_\zeta \sqrt{\log\left(\frac{1}{\delta_2}\right)}\sqrt{K} + 6\beta^2\log\left(\frac{1}{\delta_1}\right)C^2_\zeta K + \beta K\Theta\left(\frac{\sqrt{\log(\frac{K}{\delta_3})}}{(1-\gamma)m^{1/4}}\right)\nonumber\\ 
		&\quad +  \beta^2K\Theta\left(\frac{\log(\frac{K}{\delta_3})}{(1-\gamma)^2m^{1/2}}\right).\label{eq: b15}
\end{flalign}
Divide both sides of \cref{eq: b15} by $[2\beta(1-\gamma)-12\beta^2]K$. Recalling that the stepsize $\beta=\min\{1/\sqrt{K},(1-\gamma)/12\}$, which implies that $\frac{1}{\sqrt{K}[2\beta(1-\gamma)-12\beta^2]}\leq \frac{12}{(1-\gamma)^2}$. Then, with probability at least $1-(\delta_1+\delta_2+\delta_3+\delta_4)$, we have
\begin{flalign}
	&\lmupi{f_0((s,a);\bar{\theta}_K) - f_0((s,a);\theta^*)}^2\nonumber\\
	&\leq \frac{1}{K}\sum_{t=0}^{K-1}\sE_{\mu_\pi}\left[ (f_0((s,a);\theta_k) - f_0((s,a);\theta^*))^2 \right] \nonumber\\
	&\leq \frac{R^2}{[2\beta(1-\gamma)-12\beta^2]K} + \frac{10\beta C_\zeta \sqrt{\log\left(\frac{1}{\delta_2}\right)}}{[2\beta(1-\gamma)-12\beta^2]\sqrt{K}} + \frac{6\beta\log\left(\frac{1}{\delta_1}\right)C^2_\zeta }{[2\beta(1-\gamma)-12\beta^2]\sqrt{K}}\nonumber\\
	&\quad + \Theta\left(\frac{\sqrt{\log(\frac{K}{\delta_3})}}{(1-\gamma)m^{1/4}}\right)\frac{1}{[2\beta(1-\gamma)-12\beta^2]\sqrt{K}}\nonumber\\
	&\quad + \Theta\left(\frac{\log(\frac{K}{\delta_3})}{(1-\gamma)^2m^{1/2}}\right)\frac{1}{[2\beta(1-\gamma)-12\beta^2]\sqrt{K}}\nonumber\\
	&\leq \Theta\left( \frac{1}{(1-\gamma)^2\sqrt{K}} \right) + \Theta\left( \frac{1}{(1-\gamma)^2\sqrt{K}}\sqrt{\log\left(\frac{1}{\delta_1}\right)} \right) + \Theta\left( \frac{1}{(1-\gamma)^2\sqrt{K}}\sqrt{\log\left(\frac{1}{\delta_2}\right)} \right)\nonumber\\
	&\quad + \Theta\left(\frac{\sqrt{\log(\frac{K}{\delta_3})}}{(1-\gamma)^3m^{1/4}}\right) + \Theta\left(\frac{\sqrt{\log(\frac{K}{\delta_4})}}{(1-\gamma)^3m^{1/4}}\right)\nonumber\\
	&=\Theta\left( \frac{1}{(1-\gamma)^2\sqrt{K}}\left(\sqrt{\log\left(\frac{1}{\delta_1}\right)} + \sqrt{\log\left(\frac{1}{\delta_1}\right)}\right) \right) \nonumber\\
	&\quad + \Theta\left(\frac{1}{(1-\gamma)^3m^{1/4}}\left( \sqrt{\log\left(\frac{K}{\delta_3}\right)} + \sqrt{\log\left(\frac{K}{\delta_4}\right)} \right)\right).\label{eq: b16}
\end{flalign}
Finally, we upper bound $\lmupi{f((s,a);\bar{\theta}_K) - Q_\pi(s,a)}^2$. We proceed as follows
\begin{flalign}
	&\lmupi{f((s,a);\bar{\theta}_K) - Q_\pi(s,a)}^2\nonumber\\
	&\leq 3\lmupi{f((s,a);\bar{\theta}_K) - f_0((s,a);\bar{\theta}_K)}^2 + 3\lmupi{f_0((s,a);\bar{\theta}_K) -f_0((s,a);\theta^*)}^2 \nonumber\\
	&\quad + 3\lmupi{f_0((s,a);\theta^*) - Q_\pi(s,a)}^2\nonumber\\
	&\overset{(i)}{\leq} \Theta\left(\frac{1}{\sqrt{m}}\right) + 3\lmupi{f_0((s,a);\bar{\theta}_K) -f_0((s,a);\theta^*)}^2 + \frac{3}{1-\gamma}\lmupi{f_0((s,a);\theta^*_\pi) - Q_\pi(s,a)}^2,\label{eq: b17}
\end{flalign}
where $(i)$ follows from \Cref{lemma: b2} and the fact that
\begin{flalign*}
	\lmupi{f_0((s,a);\theta^*) - Q_\pi(s,a)}^2\leq \frac{1}{1-\gamma}\lmupi{f_0((s,a);\theta^*_\pi) - Q_\pi(s,a)}^2,
\end{flalign*}
which is given in \cite{cai2019neural}. Then, \cref{eq: b18} implies that, with probability at least $\delta_5$, we have
\begin{flalign}
	\lmupi{f_0((s,a);\theta^*_\pi) - Q_\pi(s,a)}^2\leq \frac{4R^2\log\left(\frac{1}{\delta_5}\right)}{m}.\label{eq: b19}
\end{flalign}
Substituting \cref{eq: b16} and \cref{eq: b19} into \cref{eq: b17}, we have with probability at least $1-(\delta_1+\delta_2+\delta_3+\delta_4+\delta_5)$, the following holds:
\begin{flalign*}
	&\lmupi{f((s,a);\bar{\theta}_K) - Q_\pi(s,a)}^2\nonumber\\
	&\leq \Theta\left( \frac{1}{(1-\gamma)^2\sqrt{K}}\left(\sqrt{\log\left(\frac{1}{\delta_1}\right)} + \sqrt{\log\left(\frac{1}{\delta_1}\right)}\right) \right) \nonumber\\
	&\quad + \Theta\left(\frac{1}{(1-\gamma)^3m^{1/4}}\left( \sqrt{\log\left(\frac{K}{\delta_3}\right)} + \sqrt{\log\left(\frac{K}{\delta_4}\right)} \right)\right)\nonumber\\
	&\quad + \Theta\left(\frac{1}{(1-\gamma)m}\log\left(\frac{1}{\delta_5}\right)\right).
\end{flalign*}
Letting $\delta_1=\delta_2=\delta_3=\delta_4=\delta_5=\frac{\delta}{5}$, we have with probability at least $1-\delta$, the following holds:
\begin{flalign*}
	&\lmupi{f((s,a);\bar{\theta}_K) - Q_\pi(s,a)}^2\nonumber\\
	&\leq \Theta\left( \frac{1}{(1-\gamma)^2\sqrt{K}}\sqrt{\log\left(\frac{1}{\delta}\right)} \right) + \Theta\left(\frac{1}{(1-\gamma)^3m^{1/4}} \sqrt{\log\left(\frac{K}{\delta}\right)} \right),
\end{flalign*}
which completes the proof.

\subsection{Supporting Lemmas for Proof of \Cref{thm2}}
For the two-layer neural network defined in \cref{eq: 4}, we have the following property: $\tau\cdot f(x,W) = f(x, \tau W)$. Thus, in the sequel, we write $\pi^\tau_W(a|s)=\pi_{\tau W}(a|s)$. In the technical proof, we consider the following policy class:
\begin{flalign}
	\pi_W(a|s)\coloneqq \frac{\exp(f((s,a);W))}{\sum_{a^\prime \mca}\exp( f((s,a^\prime);W))},\quad\forall(s,a)\in \mcs\times\mca,\label{eq: c1}
\end{flalign}
and $J_i(W)$ as the accumulated cost with policy $\pi_W$.
We denote $\phi^i_W(s,a)=\nabla_Wf_i((s,a),W)$. We define the diameter of $\mB_W$ as $R_W$. When performing each NPG update, we will need to solve the linear regression problem specified in \cref{eq: 9}. As shown in \cite{wang2019neural}, when the neural network for the policy parametrization and value function approximation share the same initialization, $\bar{\theta}_t$ is an approximated solution of the problem \cref{eq: 9}. Thus, instead of solving the problem \cref{eq: 9} directly, here we simply use $\bar{\theta}_t$ as the approximated NPG update at each iteration:
\begin{flalign*}
	\tau_{t+1}\cdot W_{t+1} = \tau_{t}\cdot W_{t} + \frac{\alpha}{1-\gamma} \bar{\theta}_t.
\end{flalign*}
Without loss of generality, we assume that for the visitation distribution of the global optimal policy $\nu^*$, there exists a constants $C_{RN}$ such that for all $\pi_W$, the following holds
{\begin{flalign}
	\int_{x}\left(\frac{d\nu^*(x)}{d\mu_{\pi_{ W}}(x)}\right)^2d\mu_{\pi_{W}}(x)\leq C^2_{RN}.\label{eq: b20}
\end{flalign}}
\begin{lemma}\label{lemma_c1}
	For any $\theta,\theta^\prime\in\mB$ and $\pi$, we have
	\begin{flalign*}
		\lmupi{\phi_{\theta}(s,a)^\top\theta^\prime - \phi_{\theta_0}(s,a)^\top\theta^\prime}^2\leq \frac{4C_0R^3}{d_1\sqrt{m}}.
	\end{flalign*}
\end{lemma}
\begin{proof}
	By definition, we have
	\begin{flalign}
		&\phi_{\theta}(s,a)^\top\theta^\prime - \phi_{\theta_0}(s,a)^\top\theta^\prime\nonumber\\
		&=\frac{1}{\sqrt{m}}\lone{\sum_{r=1}^{m}\left( \mdsone(\theta_{r}^\top\psi(x)>0) - \mdsone(\theta_{0,r}^\top\psi(x)>0) \right)b_r\theta_{r}^{\prime\top}\psi(x)}\nonumber\\
		&\leq\frac{1}{\sqrt{m}}\sum_{r=1}^{m}\lone{\left( \mdsone(\theta_{r}^\top\psi(x)>0) - \mdsone(\theta_{0,r}^\top\psi(x)>0) \right)}\lone{b_r}\ltwo{\theta_{r}^{\prime\top}\psi(x)}\nonumber\\
		&\overset{(i)}{\leq}\frac{1}{\sqrt{m}}\sum_{r=1}^{m}\mdsone(\lone{\theta^\top_{0,r}\psi(x)}\leq \ltwo{\theta_{r}-\theta_{0,r}})\ltwo{\theta_{r}^{\prime\top}\psi(x)}\nonumber\\
		&\leq\frac{1}{\sqrt{m}}\sum_{r=1}^{m}\mdsone(\lone{\theta^\top_{0,r}\psi(x)}\leq \ltwo{\theta_{r}-\theta_{0,r}})\left(\ltwo{\theta_{r}^{\prime\top}\psi(x)-\theta_{0,r}^\top \psi(x)}+\ltwo{\theta_{0,r}^\top\psi(x)}\right)\nonumber\\
		&\leq \frac{1}{\sqrt{m}}\sum_{r=1}^{m}\mdsone(\lone{\theta^\top_{0,r}\psi(x)}\leq \ltwo{\theta_{r}-\theta_{0,r}})\left(\ltwo{\theta^\prime_{r}-\theta_{0,r}}+\ltwo{\theta_{0,r}^\top\psi(s)}\right)\nonumber\\
		&\leq \frac{1}{\sqrt{m}}\sum_{r=1}^{m}\mdsone(\lone{\theta^\top_{0,r}\psi(x)}\leq \ltwo{\theta_{r}-\theta_{0,r}})\left(\ltwo{\theta^\prime_{r}-\theta_{0,r}} + \ltwo{\theta_{r}-\theta_{0,r}}\right),\label{eq: c6}
	\end{flalign}
    where $(i)$ follows from \cref{eq: b1}. Following from Holder's inequality, we obtain from \cref{eq: c6} that
    \begin{flalign}
    	&\lone{\phi_{\theta}(s,a)^\top\theta^\prime - \phi_{\theta_0}(s,a)^\top\theta^\prime}^2\nonumber\\
    	&\leq \frac{1}{m}\left[\sum_{r=1}^{m}\mdsone^2(\lone{\theta^\top_{0,r}\psi(x)}\leq \ltwo{\theta_{r}-\theta_{0,r}})\right]\left[\sum_{r=1}^{m}\left(\ltwo{\theta^\prime_{r}-\theta_{0,r}} + \ltwo{\theta_{r}-\theta_{0,r}}\right)^2 \right]\nonumber\\
    	&\leq \frac{2}{m}\left[\sum_{r=1}^{m}\mdsone^2(\lone{\theta^\top_{0,r}\psi(x)}\leq \ltwo{\theta_{r}-\theta_{0,r}})\right]\left[\sum_{r=1}^{m}\ltwo{\theta^\prime_{r}-\theta_{0,r}}^2 + \sum_{r=1}^{m}\ltwo{\theta_{r}-\theta_{0,r}}^2 \right]\nonumber\\
    	&\leq \frac{4R^2}{m}\sum_{r=1}^{m}\mdsone(\lone{\theta^\top_{0,r}\psi(x)}\leq \ltwo{\theta_{r}-\theta_{0,r}})\nonumber,
    \end{flalign}
     which implies
     \begin{flalign}
     	\lmupi{\phi_{\theta}(s,a)^\top\theta^\prime - \phi_{\theta_0}(s,a)^\top\theta^\prime}^2=\sE_{\mu_\pi}[\lone{\phi_{\theta}(s,a)^\top\theta^\prime - \phi_{\theta_0}(s,a)^\top\theta^\prime}^2]\leq \frac{4C_0R^3}{d_1\sqrt{m}},
     \end{flalign}
    where $(i)$ follows from the derivation in \Cref{lemma: b1} after \cref{eq: b2}.
\end{proof}

\begin{lemma}[Upper bound on optimality gap for neural NPG]\label{lemma_c2}
	Consider the approximated NPG updates in the neural network approximation setting. We have
	\begin{flalign*}
		&\alpha (1-\gamma)(J_0(\pi^*)-J_0(\pi_{\tau_{t} W_{t}}))\nonumber\\
		&\leq \sE_{\nu^*}\left[ D_\text{KL}(\pi^*||\pi_{\tau_t W_t}) \right] - \sE_{\nu^*}\left[ D_\text{KL}(\pi^*||\pi_{\tau_{t+1} W_{t+1}}) \right] + \frac{8\alpha C_{RN}\sqrt{C_0}R^{1.5}}{\sqrt{d_1}m^{1/4}} + \alpha^2L_f(R^2+md^2_2)  \nonumber\\
		&\quad + 2\alpha C_{RN} \lmutauWt{f((s,a),\bar{\theta}_t) - Q_{\pi_{\tau_{t} W_{t}}}(s,a)}.
	\end{flalign*}
\end{lemma}

\begin{proof}

It has been verified that the feature mapping $\phi^r_W(s,a)$ is bounded \cite{wang2019neural,cai2019neural}. By following the argument similar to that in \cite{agarwal2019optimality}[Example 6.3], we can show that $\log(\pi_w(a|s))$ is $L_f$-Lipschitz. Applying the Lipschitz property of $\log(\pi_w(a|s))$, we can obtain the following.
\begin{flalign}
	&\sE_{\nu^*}\left[ D_\text{KL}(\pi^*||\pi_{\tau_t W_t}) \right] - \sE_{\nu^*}\left[ D_\text{KL}(\pi^*||\pi_{\tau_{t+1} W_{t+1}}) \right]\nonumber\\
	&= \sE_{\nu^*}\left[ \log(\pi_{\tau_{t+1} W_{t+1}}(a|s)) - \log(\pi_{\tau_{t} W_{t}}(a|s)) \right]\nonumber\\
	&\overset{(i)}{\geq} \sE_{\nu^*}\left[ \nabla_W \log(\pi_{\tau_{t} W_{t}}(a|s)) \right]^\top (\tau_{t+1} W_{t+1} - \tau_{t} W_{t} ) - \frac{L_f}{2}\ltwo{\tau_{t+1} W_{t+1}-\tau_{t} W_{t}}^2\nonumber\\
	&=\alpha \sE_{\nu^*}\left[ \nabla_W \log(\pi_{\tau_{t} W_{t}}(a|s)) \right]^\top \bar{\theta}_t - \frac{\alpha^2L_f}{2}\ltwo{\bar{\theta}_t}^2\nonumber\\
	&=\alpha \sE_{\nu^*}\left[ \phi_{W_t}(s,a) - \sE_{\pi_{\tau_{t} W_{t}}}\left[ \phi_{W_t}(s,a^\prime) \right] \right]^\top \bar{\theta}_t - \frac{\alpha^2L_f}{2}\ltwo{\bar{\theta}_t}^2\nonumber\\
	&=\alpha \sE_{\nu^*}\left[ Q_{\pi_{\tau_{t} W_{t}}}(s,a) - \sE_{\pi_{\tau_{t} W_{t}}}\left[ Q_{\pi_{\tau_{t} W_{t}}}(s,a^\prime) \right] \right] + \alpha \sE_{\nu^*}\left[ \phi_{W_t}(s,a)^\top\bar{\theta}_t - Q_{\pi_{\tau_{t} W_{t}}}(s,a)  \right] \nonumber\\
	&\quad  + \alpha \sE_{\nu^*}\sE_{\pi_{\tau_{t} W_{t}}}\left[ Q_{\pi_{\tau_{t} W_{t}}}(s,a^\prime) - \phi_{W_t}(s,a^\prime)^\top\bar{\theta}_t \right] - \frac{\alpha^2L_f}{2}\ltwo{\bar{\theta}_t}^2\nonumber\\
	&=\alpha (1-\gamma)(J_0(\pi^*)-J_0(\pi_{\tau_{t} W_{t}})) + \alpha \sE_{\nu^*}\left[ \phi_{W_t}(s,a)^\top\bar{\theta}_t - f((s,a),\bar{\theta}_t)  \right] \nonumber\\
	&\quad + \alpha \sE_{\nu^*}\left[ f((s,a),\bar{\theta}_t) - Q_{\pi_{\tau_{t} W_{t}}}(s,a)  \right]  + \alpha \sE_{\nu^*}\sE_{\pi_{\tau_{t} W_{t}}}\left[ Q_{\pi_{\tau_{t} W_{t}}}(s,a^\prime) - f((s,a^\prime),\bar{\theta}_t) \right] \nonumber\\
	&\quad + \alpha \sE_{\nu^*}\sE_{\pi_{\tau_{t} W_{t}}}\left[ f((s,a^\prime),\bar{\theta}_t) - \phi_{W_t}(s,a^\prime)^\top\bar{\theta}_t \right] - \frac{\alpha^2L_f}{2}\ltwo{\bar{\theta}_t}^2\nonumber\\
	&=\alpha (1-\gamma)(J_0(\pi^*)-J_0(\pi_{\tau_{t} W_{t}})) + \alpha \sE_{\nu^*}\left[ \phi_{W_t}(s,a)^\top\bar{\theta}_t - f((s,a),\bar{\theta}_t)  \right] \nonumber\\
	&\quad + \alpha \sE_{\nu^*}\left[ f((s,a),\bar{\theta}_t) - Q_{\pi_{\tau_{t} W_{t}}}(s,a)  \right]  + \alpha \sE_{\nu^*}\sE_{\pi_{\tau_{t} W_{t}}}\left[ Q_{\pi_{\tau_{t} W_{t}}}(s,a^\prime) - f((s,a^\prime),\bar{\theta}_t) \right] \nonumber\\
	&\quad + \alpha \sE_{\nu^*}\sE_{\pi_{\tau_{t} W_{t}}}\left[ f((s,a^\prime),\bar{\theta}_t) - \phi_{W_t}(s,a^\prime)^\top\bar{\theta}_t \right] - \frac{\alpha^2L_f}{2}\ltwo{\bar{\theta}_t}^2\nonumber\\
	&=\alpha (1-\gamma)(J_0(\pi^*)-J_0(\pi_{\tau_{t} W_{t}})) + \alpha \sE_{\nu^*}\sE_{\bar{\theta}_t}\left[ \phi_{W_t}(s,a)^\top\bar{\theta}_t - f((s,a),\bar{\theta}_t)  \right] \nonumber\\
	&\quad + \alpha \sE_{\nu^*}\left[ f((s,a),\bar{\theta}_t) - Q_{\pi_{\tau_{t} W_{t}}}(s,a)  \right]  + \alpha \sE_{\nu^*}\sE_{\pi_{\tau_{t} W_{t}}}\left[ Q_{\pi_{\tau_{t} W_{t}}}(s,a^\prime) - f((s,a^\prime),\bar{\theta}_t) \right] \nonumber\\
	&\quad + \alpha \sE_{\nu^*}\sE_{\pi_{\tau_{t} W_{t}}}\left[ f((s,a^\prime),\bar{\theta}_t) - \phi_{W_t}(s,a^\prime)^\top\bar{\theta}_t \right] - \frac{\alpha^2L_f}{2}\ltwo{\bar{\theta}_t}^2\nonumber\\
	&\geq\alpha (1-\gamma)(J_0(\pi^*)-J_0(\pi_{\tau_{t} W_{t}})) - \alpha \sqrt{\sE_{\nu^*}\left[ \left(\phi_{W_t}(s,a)^\top\bar{\theta}_t - f((s,a),\bar{\theta}_t)\right)^2  \right]} \nonumber\\
	&\quad - \alpha \sqrt{\sE_{\nu^*}\left[ (f((s,a),\bar{\theta}_t) - Q_{\pi_{\tau_{t} W_{t}}}(s,a))^2  \right]}  - \alpha \sqrt{\sE_{\nu^*}\sE_{\pi_{\tau_{t} W_{t}}}\left[ (Q_{\pi_{\tau_{t} W_{t}}}(s,a^\prime) - f((s,a^\prime),\bar{\theta}_t))^2 \right]} \nonumber\\
	&\quad - \alpha \sqrt{\sE_{\nu^*}\sE_{\pi_{\tau_{t} W_{t}}}\left[ (f((s,a^\prime),\bar{\theta}_t) - \phi_{W_t}(s,a^\prime)^\top\bar{\theta}_t)^2 \right]} - \frac{\alpha^2L_f}{2}\ltwo{\bar{\theta}_t}^2,\label{eq: c2}
\end{flalign}
where $(i)$ follows from the $L_f$-Lipschitz property of $\log(\pi_w(a|s))$.
Note that for any $x\sim \nu_{\pi_{W}}$, and any function $h(x)$, we have
\begin{flalign}
	\int_{x}h(x) d\nu^*(x) &= \int_{x}h(x) \frac{d\nu^*(x)}{d\mu_{\pi_{W}}(x)} d\mu_{\pi_{W}}(x)\nonumber\\
	&\overset{(i)}{\leq} \sqrt{\int_{x} h^2(x) d\mu_{\pi_{W}}(x)}\sqrt{\int_{x}\left(\frac{d\nu^*(x)}{d\mu_{\pi_{ W}}(x)}\right)^2d\mu_{\pi_{W}}(x)}\nonumber\\
	& \overset{(ii)}{\leq} C^2_{RN}\lmuW{h(x)},\label{eq: c3}
\end{flalign}
where $(i)$ follows from Holder's inequality, and $(ii)$ follows from \cref{eq: b20}. Similarly, we can obtain
\begin{flalign}
	\int_{x}h(x) d(\nu^*\pi_{W})(x) \leq C^2_{RN}\lmuW{h(x)}. \label{eq: c4}
\end{flalign}
Substituting \cref{eq: c3} and \cref{eq: c4} into \cref{eq: c2} and using the fact that $\ltwo{\bar{\theta}_t}\leq R+\sqrt{m}d_2$ yield
\begin{flalign}
	&\sE_{\nu^*}\left[ D_\text{KL}(\pi^*||\pi_{\tau_t W_t}) \right] - \sE_{\nu^*}\left[ D_\text{KL}(\pi^*||\pi_{\tau_{t+1} W_{t+1}}) \right]\nonumber\\
	&\geq\alpha (1-\gamma)(J_0(\pi^*)-J_0(\pi_{\tau_{t} W_{t}})) - \alpha C_{RN} \sqrt{\sE_{\nu_{\pi_{\tau_{t} W_{t}}}}\left[ \left(\phi_{W_t}(s,a)^\top\bar{\theta}_t - f((s,a),\bar{\theta}_t)\right)^2  \right]} \nonumber\\
	&\quad - \alpha C_{RN} \sqrt{\sE_{\mu_{\pi_{\tau_{t} W_{t}}}}\left[ (f((s,a),\bar{\theta}_t) - Q_{\pi_{\tau_{t} W_{t}}}(s,a))^2  \right]}  \nonumber\\
	&\quad - \alpha C_{RN} \sqrt{\sE_{\mu_{\pi_{\tau_{t} W_{t}}}}\left[ (Q_{\pi_{\tau_{t} W_{t}}}(s,a^\prime) - f((s,a^\prime),\bar{\theta}_t))^2 \right]} \nonumber\\
	&\quad - \alpha C_{RN} \sqrt{\sE_{\mu_{\pi_{\tau_{t} W_{t}}}}\left[ (f((s,a^\prime),\bar{\theta}_t) - \phi_{W_t}(s,a^\prime)^\top\bar{\theta}_t)^2 \right]} - \alpha^2L_f(R^2+md^2_2)\nonumber\\
	&=\alpha (1-\gamma)(J_0(\pi^*)-J_0(\pi_{\tau_{t} W_{t}})) - 2\alpha C_{RN} \sqrt{\sE_{\mu_{\pi_{\tau_{t} W_{t}}}}\left[ \left(\phi_{W_t}(s,a)^\top\bar{\theta}_t - f((s,a),\bar{\theta}_t)\right)^2  \right]} \nonumber\\
	&\quad - 2\alpha C_{RN} \sqrt{\sE_{\mu_{\pi_{\tau_{t} W_{t}}}}\left[ (f((s,a),\bar{\theta}_t) - Q_{\pi_{\tau_{t} W_{t}}}(s,a))^2  \right]} - \alpha^2L_f(R^2+md^2_2)\nonumber\\
	&=\alpha (1-\gamma)(J_0(\pi^*)-J_0(\pi_{\tau_{t} W_{t}})) - 2\alpha C_{RN}  \lmutauWt{\phi_{W_t}(s,a)^\top\bar{\theta}_t - f((s,a),\bar{\theta}_t) } \nonumber\\
	&\quad - 2\alpha C_{RN}  \lmutauWt{f((s,a),\bar{\theta}_t) - Q_{\pi_{\tau_{t} W_{t}}}(s,a)} - \alpha^2L_f(R^2+md^2_2).\label{eq: c5}
\end{flalign}
We then proceed to upper bound the term $\lmutauWt{\phi_{W_t}(s,a)^\top\bar{\theta}_t - f((s,a),\bar{\theta}_t) }^2$.
\begin{flalign}
	&\lmutauWt{\phi_{W_t}(s,a)^\top\bar{\theta}_t - f((s,a),\bar{\theta}_t) }^2\nonumber\\
	&=\lmutauWt{\phi_{W_t}(s,a)^\top\bar{\theta}_t - \phi_{W_0}(s,a)^\top\bar{\theta}_t + \phi_{W_0}(s,a)^\top\bar{\theta}_t - f((s,a),\bar{\theta}_t) }^2\nonumber\\
	&\leq 2\lmutauWt{\phi_{W_t}(s,a)^\top\bar{\theta}_t - \phi_{W_0}(s,a)^\top\bar{\theta}_t}^2 + 2\lmutauWt{\phi_{W_0}(s,a)^\top\bar{\theta}_t - f((s,a),\bar{\theta}_t)}^2\nonumber\\
	&\overset{(i)}{\leq} \frac{16C_0R^3}{d_1\sqrt{m}},\label{eq: c7}
\end{flalign}
where $(i)$ follows from \Cref{lemma: b2} and \Cref{lemma_c1}. Substituting \cref{eq: c7} into \cref{eq: c5} yields
\begin{flalign*}
	&\sE_{\nu^*}\left[ D_\text{KL}(\pi^*||\pi_{\tau_t W_t}) \right] - \sE_{\nu^*}\left[ D_\text{KL}(\pi^*||\pi_{\tau_{t+1} W_{t+1}}) \right]\nonumber\\
	&\leq \alpha (1-\gamma)(J_0(\pi^*)-J_0(\pi_{\tau_{t} W_{t}})) - \frac{8\alpha C_{RN}\sqrt{C_0}R^{1.5}}{\sqrt{d_1}m^{1/4}} - \alpha^2L_f(R^2+md^2_2) \nonumber\\
	&\quad - 2\alpha C_{RN}  \lmutauWt{f((s,a),\bar{\theta}_t) - Q_{\pi_{\tau_{t} W_{t}}}(s,a)}.
\end{flalign*}
Rearranging the above inequality yields the desired result.
\end{proof}
Note that when we follow the update in line 10 of \Cref{algorithm_cpg}, we can obtain similar results for the case $i\in\{1,\cdots,p\}$ as stated in \Cref{lemma_c2}:
\begin{flalign*}
	&\alpha (1-\gamma)(J_i(\pi_{\tau_{t} W_{t}})-J_i(\pi^*))\nonumber\\
	&\leq \sE_{\nu^*}\left[ D_\text{KL}(\pi^*||\pi_{\tau_t W_t}) \right] - \sE_{\nu^*}\left[ D_\text{KL}(\pi^*||\pi_{\tau_{t+1} W_{t+1}}) \right] + \frac{8\alpha C_{RN}\sqrt{C_0}R^{1.5}}{\sqrt{d_1}m^{1/4}} + \alpha^2L_f(R^2+md^2_2)  \nonumber\\
	&\quad + 2\alpha C_{RN} \lmutauWt{f((s,a),\bar{\theta}_t) - Q_{\pi_{\tau_{t} W_{t}}}(s,a)}.
\end{flalign*}

\begin{lemma}\label{lemma_n1}
	Considering the CRPO update in \Cref{algorithm_cpg} in the neural network approximation setting. Let $K_{\text{in}}=C_1((1-\gamma)^2\sqrt{m})$ and $N=T\log(2T/\delta)$. With probability at least $1-\delta$, we have
	\begin{flalign*}
	&\alpha(1-\gamma)\sum_{t\in\gN_0}(J_0(\pi^*)-J_0(\pi_{w_t})) + \alpha(1-\gamma)\eta\sum_{i=1}^{p}\lone{\gN_i}\nonumber\\
	&\leq \sE_{s\sim\nu^*} D_{\text{KL}}(\pi^*||\pi_{w_0}) + C_3\left(\frac{\alpha T}{m^{1/4}}\right) + C_4(\alpha^2mT) \nonumber\\
	&\quad  + C_5\left(\frac{\alpha T}{(1-\gamma)^{1.5}m^{1/8}} \log^{\frac{1}{4}}\left(\frac{T^3}{\delta}\right) \right) + C_6\left(\alpha(1-\gamma)\sqrt{T}\right).
	\end{flalign*}
where $C_3=\frac{8 C_{RN}\sqrt{C_0}R^{1.5}}{\sqrt{d_1}}$, $C_4=L_f(R^2+d^2_2)$, $C_5=3\alpha C_2C_{RN}$, $C_6=2C_f$ and $C_2$ is a positive constant depend on $C_1$.
\end{lemma}
\begin{proof}
	We define $\gN_i$ as the set of steps that CRPO algorithm chooses to minimize the $i$-th constraint. If $t\in \gN_0$, by \Cref{lemma_c2} we have
	\begin{flalign}
		&\alpha (1-\gamma)(J_0(\pi^*)-J_0(\pi_{\tau_{t} W_{t}}))\nonumber\\
		&\leq \sE_{\nu^*}\left[ D_\text{KL}(\pi^*||\pi_{\tau_t W_t}) \right] - \sE_{\nu^*}\left[ D_\text{KL}(\pi^*||\pi_{\tau_{t+1} W_{t+1}}) \right] + \frac{8\alpha C_{RN}\sqrt{C_0}R^{1.5}}{\sqrt{d_1}m^{1/4}} + \alpha^2L_f(R^2+md^2_2)  \nonumber\\
		&\quad + 2\alpha C_{RN} \lmutauWt{f_0((s,a),\bar{\theta}_t) - Q^0_{\pi_{\tau_{t} W_{t}}}(s,a)}.\label{eq: a1}
	\end{flalign}
	If $t\in \gN_i$, similarly we can obtain
	\begin{flalign}
		&\alpha (1-\gamma)(J_i(\pi_{\tau_{t} W_{t}})-J_i(\pi^*))\nonumber\\
		&\leq \sE_{\nu^*}\left[ D_\text{KL}(\pi^*||\pi_{\tau_t W_t}) \right] - \sE_{\nu^*}\left[ D_\text{KL}(\pi^*||\pi_{\tau_{t+1} W_{t+1}}) \right] + \frac{8\alpha C_{RN}\sqrt{C_0}R^{1.5}}{\sqrt{d_1}m^{1/4}} + \alpha^2L_f(R^2+md^2_2)  \nonumber\\
		&\quad + 2\alpha C_{RN} \lmutauWt{f_i((s,a),\bar{\theta}_t) - Q^i_{\pi_{\tau_{t} W_{t}}}(s,a)}.\label{eq: a2}
	\end{flalign}
	Taking summation of \cref{eq: m1} and \cref{eq: m2} from $t=0$ to $T-1$ yields
	\begin{flalign}
		&\alpha(1-\gamma)\sum_{t\in\gN_0}(J_0(\pi^*)-J_0(\pi_{w_t})) + \alpha(1-\gamma)\sum_{i=1}^{p}\sum_{t\in\gN_i}(J_i(\pi_{w_t})-J_i(\pi^*))\nonumber\\
		&\leq \sE_{s\sim\nu^*} D_{\text{KL}}(\pi^*||\pi_{w_0}) + \frac{8\alpha C_{RN}\sqrt{C_0}R^{1.5}T}{\sqrt{d_1}m^{1/4}} + \alpha^2L_f(R^2+md^2_2)T \nonumber\\
		&\quad + 2\alpha C_{RN} \sum_{i=0}^{p}\sum_{t\in\gN_i}\lmutauWt{f_i((s,a),\bar{\theta}_t) - Q^i_{\pi_{\tau_{t} W_{t}}}(s,a)}.\label{eq: a3}
	\end{flalign}
    Note that when $t\in \gN_i$ ($i\neq 0$), we have $\bar{J}_{i}(\theta^{i}_t)> d_{i}+\eta$ (line 9 in \Cref{algorithm_cpg}), which implies that
    \begin{flalign}
    	J_i(\pi_{\tau_tW_t})-J_i(\pi^*)&\geq \bar{J}_{i}(\theta^{i}_t)-J_i(\pi^*) - \lone{\bar{J}_{i}(\theta^{i}_t) - J_i(\pi_{\tau_tW_t})}\nonumber\\
    	&\geq d_{i}+\eta-J_i(\pi^*) - \lone{\bar{J}_{i}(\theta^{i}_t) - J_i(\pi_{\tau_tW_t})}\nonumber\\
    	&\geq \eta - \lone{\bar{J}_{i}(\theta^{i}_t) - J_i(\pi_{\tau_tW_t})}.\label{eq: a4}
    \end{flalign}
   To bound the term $\lone{\bar{J}_{i}(\theta^{i}_t) - J_i(\pi_{\tau_tW_t})}$, we proceed as follows
    \begin{flalign}
    	&\lone{\bar{J}_{i}(\theta^{i}_t) - J_i(\pi_{\tau_tW_t})}\nonumber\\
    	&=\lone{\bar{J}_{i}(\theta^{i}_t) - \sE_{\nu_{\pi_{\tau_{t} W_{t}}}}[f_i((s,a),\bar{\theta}_t)] + \sE_{\nu_{\pi_{\tau_{t} W_{t}}}}[f_i((s,a),\bar{\theta}_t)] - J_i(\pi_{\tau_tW_t})}\nonumber\\
    	&\leq \lone{\bar{J}_{i}(\theta^{i}_t) - \sE_{\nu_{\pi_{\tau_{t} W_{t}}}}[f_i((s,a),\bar{\theta}_t)]} + \lnutauWt{f_i((s,a),\bar{\theta}_t) - Q^i_{\pi_{\tau_{t} W_{t}}}(s,a)}\nonumber\\
    	&\overset{(i)}{\leq} \lone{\bar{J}_{i}(\theta^{i}_t) - \sE_{\nu_{\pi_{\tau_{t} W_{t}}}}[f_i((s,a),\bar{\theta}_t)]} + C_{RN}\lmutauWt{f_i((s,a),\bar{\theta}_t) - Q^i_{\pi_{\tau_{t} W_{t}}}(s,a)},\label{eq: a5}
    \end{flalign}
    where $(i)$ can be obtained by following steps similar to those in \cref{eq: c3}. Substituting \cref{eq: a5} into \cref{eq: a4} yields
    \begin{flalign}
    	&J_i(\pi_{\tau_tW_t})-J_i(\pi^*)\nonumber\\
    	&\geq \eta - \left( \lone{\bar{J}_{i}(\theta^{i}_t) - \sE_{\nu_{\pi_{\tau_{t} W_{t}}}}[f_i((s,a),\bar{\theta}_t)]} + C_{RN}\lmutauWt{f_i((s,a),\bar{\theta}_t) - Q^i_{\pi_{\tau_{t} W_{t}}}(s,a)} \right).\label{eq: a6}
    \end{flalign}
    Then, substituting \cref{eq: a6} into \cref{eq: a3} yields
    \begin{flalign}
    	&\alpha(1-\gamma)\sum_{t\in\gN_0}(J_0(\pi^*)-J_0(\pi_{w_t})) + \alpha(1-\gamma)\eta\sum_{i=1}^{p}\lone{\gN_i}\nonumber\\
    	&\leq \sE_{s\sim\nu^*} D_{\text{KL}}(\pi^*||\pi_{w_0}) + \frac{8\alpha C_{RN}\sqrt{C_0}R^{1.5}T}{\sqrt{d_1}m^{1/4}} + \alpha^2L_f(R^2+md^2_2)T \nonumber\\
    	&\quad + 3\alpha C_{RN} \sum_{i=0}^{p}\sum_{t\in\gN_i}\lmutauWt{f_i((s,a),\bar{\theta}_t) - Q^i_{\pi_{\tau_{t} W_{t}}}(s,a)} \nonumber\\
    	&\quad + \alpha(1-\gamma)\sum_{i=1}^{p}\sum_{t\in\gN_i}\lone{\bar{J}_{i}(\theta^{i}_t) - \sE_{\nu_{\pi_{\tau_{t} W_{t}}}}[f_i((s,a),\bar{\theta}_t)]}\nonumber\\
    	&\leq \sE_{s\sim\nu^*} D_{\text{KL}}(\pi^*||\pi_{w_0}) + \frac{8\alpha C_{RN}\sqrt{C_0}R^{1.5}T}{\sqrt{d_1}m^{1/4}} + \alpha^2L_f(R^2+md^2_2)T \nonumber\\
    	&\quad + 3\alpha C_{RN} \sum_{t=0}^{T-1}\lmutauWt{f_i((s,a),\bar{\theta}_t) - Q^i_{\pi_{\tau_{t} W_{t}}}(s,a)} \nonumber\\
    	&\quad + \alpha(1-\gamma)\sum_{t=0}^{T-1}\lone{\bar{J}_{i}(\theta^{i}_t) - \sE_{\nu_{\pi_{\tau_{t} W_{t}}}}[f_i((s,a),\bar{\theta}_t)]}.\label{eq: a7}
    \end{flalign}
    We then upper bound the term $\sum_{t=0}^{T-1}\lmutauWt{f_i((s,a),\bar{\theta}_t) - Q^i_{\pi_{\tau_{t} W_{t}}}(s,a)}$. \Cref{lemma: neuralTD} implies that if we let $K_{\text{in}}=C_1((1-\gamma)^2\sqrt{m})$, then with probability at least $1-\delta_1/T$, we have
    \begin{flalign*}
    	\lmupi{f((s,a);\bar{\theta}_K) - Q_\pi(s,a)}\leq   C_2\left(\frac{1}{(1-\gamma)^{1.5}m^{1/8}} \log^{\frac{1}{4}}\left(\frac{(1-\gamma)^2T\sqrt{m}}{\delta_1}\right) \right),
    \end{flalign*}
    where $C_1$ and $C_2$ are positive constant. Applying the union bound, we have with probability at least $1-\delta_1$,
    \begin{flalign}
    	&\sum_{t=0}^{T-1}\lmutauWt{f_i((s,a),\bar{\theta}_t) - Q^i_{\pi_{\tau_{t} W_{t}}}(s,a)}\leq C_2\left(\frac{T}{(1-\gamma)^{1.5}m^{1/8}} \log^{\frac{1}{4}}\left(\frac{(1-\gamma)^2T\sqrt{m}}{\delta_1}\right) \right).\label{eq: a9}
    \end{flalign}
    We then bound the term $\sum_{t=0}^{T-1}\lone{\bar{J}_{i}(\theta^{i}_t) - \sE_{\nu_{\pi_{\tau_{t} W_{t}}}}[f_i((s,a),\bar{\theta}_t)]}$. For simplicity, we denote $J^\prime_i(\bar{\theta}_t)=\sE_{\xi\cdot\mu_{\pi_{\tau_{t} W_{t}}}}[f_i((s,a),\bar{\theta}_t)]$. Recall that $\bar{J}_{i}(\theta^{i}_t)=\frac{1}{N}\sum_{j=1}^{N}f_i((s_j,a_j),\bar{\theta}_t)$. For each $t\geq 0$, we bound the error $\bar{J}_{i}(\theta^{i}_t) - J^\prime_i(\bar{\theta}_t)$ as follows:
    \begin{flalign}
    	&\mP\left( \left( \frac{1}{N}\sum_{j=1}^{N}f_i((s_j,a_j),\bar{\theta}_t) - J^\prime_i(\bar{\theta}_t) \right)^2\geq \frac{(1+\Lambda)C_f^2}{N}\right)\nonumber\\
    	&\leq \mP\left(  \frac{1}{N}\sum_{j=1}^{N}\frac{\left[f_i((s_j,a_j),\bar{\theta}_t) - J^\prime_i(\bar{\theta}_t)\right]^2}{C_f^2} \geq 1+\Lambda\right)\nonumber\\
    	&= \mP\left( \exp\left(\frac{1}{N}\sum_{j=1}^{N}\frac{\left[f_i((s_j,a_j),\bar{\theta}_t) - J^\prime_i(\bar{\theta}_t)\right]^2}{C^2_f}\right) \geq 1+\Lambda\right)\nonumber\\
    	&\leq \mP\left( \frac{1}{N}\sum_{j=1}^{N}\exp\left( \frac{\left[f_i((s_j,a_j),\bar{\theta}_t) - J^\prime_i(\bar{\theta}_t)\right]^2}{C^2_f} \right) \geq 1+\Lambda\right)\nonumber\\
    	&\overset{(i)}{\leq} \frac{1}{N}\sum_{j=1}^{N}\sE\left[ \exp\left( \frac{\left[f_i((s_j,a_j),\bar{\theta}_t) - J^\prime_i(\bar{\theta}_t)\right]^2}{C^2_f} \right) \right]/\exp(1+\Lambda)\nonumber\\
    	&\leq \exp(-\Lambda),\label{eq: a8}
    \end{flalign}
    where $(i)$ follows from Markov's inequality. Then, \cref{eq: a8} implies that with probability at least $1-\delta_2/T$, we have
    \begin{flalign*}
    	\lone{\frac{1}{N}\sum_{j=1}^{N}f_i((s_j,a_j),\bar{\theta}_t) - J^\prime_i(\bar{\theta}_t)}\leq \frac{C_f}{\sqrt{N}}\left( 1 + \sqrt{\log\left( \frac{T}{\delta_2} \right)} \right).
    \end{flalign*}
    Applying the union bound, we have with probability at least $1-\delta_2$,
    \begin{flalign}
    	\sum_{t=0}^{T-1}\lone{\bar{J}_{i}(\theta^{i}_t) - \sE_{\nu_{\pi_{\tau_{t} W_{t}}}}[f_i((s,a),\bar{\theta}_t)]}\leq \frac{C_f T}{\sqrt{N}}\left( 1 + \sqrt{\log\left( \frac{T}{\delta_2} \right)} \right).\label{eq: a10}
    \end{flalign}
    Letting $\delta_1=\delta_2=\frac{\delta}{2}$, $N=T\log(2T/\delta)$, and combing \cref{eq: a9} and \cref{eq: a10}, we have with probability at least $1-\delta$
    \begin{flalign*}
    &\alpha(1-\gamma)\sum_{t\in\gN_0}(J_0(\pi^*)-J_0(\pi_{w_t})) + \alpha(1-\gamma)\eta\sum_{i=1}^{p}\lone{\gN_i}\nonumber\\
    &\leq \sE_{s\sim\nu^*} D_{\text{KL}}(\pi^*||\pi_{w_0}) + C_3\left(\frac{\alpha T}{m^{1/4}}\right) + C_4(\alpha^2mT) \nonumber\\
    &\quad  + C_5\left(\frac{\alpha T}{(1-\gamma)^{1.5}m^{1/8}} \log^{\frac{1}{4}}\left(\frac{(1-\gamma)^2T\sqrt{m}}{\delta}\right) \right) + C_6\left(\alpha(1-\gamma)\sqrt{T}\right),
    \end{flalign*}
where $C_3=\frac{8 C_{RN}\sqrt{C_0}R^{1.5}}{\sqrt{d_1}}$, $C_4=L_f(R^2+d^2_2)$, $C_5=3\alpha C_2C_{RN}$, and $C_6=2C_f$ are positive constants.
\end{proof}

\begin{lemma}\label{lemma_n2}
	Let $K_{\text{in}}=C_1((1-\gamma)^2\sqrt{m})$, $N=T\log(2T/\delta)$, and
	\begin{flalign}
	\frac{1}{2}\alpha(1-\gamma)\eta T&\geq \sE_{s\sim\nu^*} D_{\text{KL}}(\pi^*||\pi_{w_0}) + C_3\left(\frac{\alpha T}{m^{1/4}}\right) + C_4(\alpha^2mT)  \nonumber\\
	&\quad  + C_5\left(\frac{\alpha T}{(1-\gamma)^{1.5}m^{1/8}} \log^{\frac{1}{4}}\left(\frac{(1-\gamma)^2T\sqrt{m}}{\delta}\right) \right) + C_6\left(\alpha(1-\gamma)\sqrt{T}\right).\label{eq: a11}
	\end{flalign}
	Then with probability at least $1-\delta$, we have the following holds
	\begin{enumerate}
		\item $\gN_0\neq \emptyset$, i.e., $w_{\text{out}}$ is well-defined,
		\item One of the following two statements must hold,
		\begin{enumerate}
			\item $\lone{\gN_0}\geq T/2$,
			\item $\sum_{t\in\gG}(J_0(\pi^*)-J_0(w_t))\leq 0$.
		\end{enumerate}
	\end{enumerate}
\end{lemma}
\begin{proof}
	Under the event given in \Cref{lemma_n1}, which happens with probability at least $1-\delta$, we have
	 \begin{flalign}
	&\alpha(1-\gamma)\sum_{t\in\gN_0}(J_0(\pi^*)-J_0(\pi_{w_t})) + \alpha(1-\gamma)\eta\sum_{i=1}^{p}\lone{\gN_i}\nonumber\\
	&\leq \sE_{s\sim\nu^*} D_{\text{KL}}(\pi^*||\pi_{w_0}) + C_3\left(\frac{\alpha T}{m^{1/4}}\right) + C_4(\alpha^2mT)  \nonumber\\
	&\quad  + C_5\left(\frac{\alpha T}{(1-\gamma)^{1.5}m^{1/8}} \log^{\frac{1}{4}}\left(\frac{(1-\gamma)^2T\sqrt{m}}{\delta}\right) \right) + C_6\left(\alpha(1-\gamma)\sqrt{T}\right).\label{eq: a12}
	\end{flalign}
	We first verify item 1. If $\gN_0= \emptyset$, then $\sum_{i=1}^{p}\lone{\gN_i} = T$, and \Cref{lemma_n1} implies that
	\begin{flalign*}
	\alpha(1-\gamma)\eta T&\leq \sE_{s\sim\nu^*} D_{\text{KL}}(\pi^*||\pi_{w_0}) + C_3\left(\frac{\alpha T}{m^{1/4}}\right) + C_4(\alpha^2mT) \nonumber\\
	&\quad  + C_5\left(\frac{\alpha T}{(1-\gamma)^{1.5}m^{1/8}} \log^{\frac{1}{4}}\left(\frac{(1-\gamma)^2T\sqrt{m}}{\delta}\right) \right) + C_6\left(\alpha(1-\gamma)\sqrt{T}\right),
	\end{flalign*}
	which contradicts \cref{eq: a11}. Thus, we must have $\gN_0\neq \emptyset$.
	
	We then proceed to verify the item 2. If $\sum_{t\in\gG}(J_0(\pi^*)-J_0(w_t))\leq 0$, then (b) in item 2 holds. If $\sum_{t\in\gG}(J_0(\pi^*)-J_0(w_t))\leq 0$, then \cref{eq: a12} implies that
	\begin{flalign*}
	\alpha(1-\gamma)\eta\sum_{i=1}^{p}\lone{\gN_i}&\leq \sE_{s\sim\nu^*} D_{\text{KL}}(\pi^*||\pi_{w_0}) + C_3\left(\frac{\alpha T}{m^{1/4}}\right) + C_4(\alpha^2mT)  \nonumber\\
	&\quad  + C_5\left(\frac{\alpha T}{(1-\gamma)^{1.5}m^{1/8}} \log^{\frac{1}{4}}\left(\frac{(1-\gamma)^2T\sqrt{m}}{\delta}\right) \right) + C_6\left(\alpha(1-\gamma)\sqrt{T}\right).
	\end{flalign*}
	Suppose that $\lone{\gN_0}<T/2$, i.e., $\sum_{i=1}^{p}\lone{\gN_i}\geq T/2$. Then,
	\begin{flalign*}
	\frac{1}{2}\alpha(1-\gamma)\eta T&\leq \sE_{s\sim\nu^*} D_{\text{KL}}(\pi^*||\pi_{w_0}) + C_3\left(\frac{\alpha T}{m^{1/4}}\right) + C_4(\alpha^2mT) \nonumber\\
	&\quad  + C_5\left(\frac{\alpha T}{(1-\gamma)^{1.5}m^{1/8}} \log^{\frac{1}{4}}\left(\frac{(1-\gamma)^2T\sqrt{m}}{\delta}\right) \right) + C_6\left(\alpha(1-\gamma)\sqrt{T}\right),
	\end{flalign*}
	which contradicts \cref{eq: a11}. Hence, (a) in item 2 holds.
\end{proof}

\subsection{Proof of \Cref{thm2}}
We restate \Cref{thm2} as follows to include the specifics of the parameters.
\begin{theorem}[Restatement of \Cref{thm2}]\label{thm4}
	Consider \Cref{algorithm_cpg} in the neural network approximation setting. Suppose Assumptions \ref{ass1}-\ref{ass4} hold. Let $\alpha = \frac{1}{2C_4\sqrt{T}}$ and 
	\begin{flalign*}
		\eta&= \frac{4C_4\sE_{s\sim\nu^*} D_{\text{KL}}(\pi^*||\pi_{w_0})}{(1-\gamma)\sqrt{T}} + \frac{2C_3}{(1-\gamma)m^{1/4}} + \frac{m}{(1-\gamma)\sqrt{T}}  \nonumber\\
		&\quad  + 2C_5\left(\frac{\alpha T}{(1-\gamma)^{2.5}m^{1/8}} \log^{\frac{1}{4}}\left(\frac{(1-\gamma)^2T\sqrt{m}}{\delta}\right) \right) + \frac{2C_6}{\sqrt{T}}.
	\end{flalign*}
	Suppose performing neural TD with $K_{\text{in}}=C_1(1-\gamma)^2\sqrt{m}$ iterations at each iteration of CRPO. Then, with probability at least $1-\delta$, we have
	\begin{flalign*}
	J_0(\pi^*)-\sE[J_0(\pi_{w_{\text{out}}})]\leq \frac{C_7m}{(1-\gamma)\sqrt{T}} + \frac{C_8}{(1-\gamma)^{2.5}m^{1/8}} \log^{\frac{1}{4}}\left(\frac{(1-\gamma)^2T\sqrt{m}}{\delta}\right),
	\end{flalign*}
where 
\begin{flalign*}
	C_7 = \frac{4C_4 D_{\text{KL}}(\pi^*||\pi_{w_0})}{m} + \frac{2(1-\gamma)C_6}{m} + 1,
\end{flalign*}
and 
\begin{flalign*}
	C_8 = 2C_5 + \frac{2C_3(1-\gamma)^{1.5}}{m^{1/8}}.
\end{flalign*}
	For all $i\in\{1,\cdots,p\}$, we have
	\begin{flalign*}
	\sE[J_i(\pi_{w_{\text{out}}})]-d_i&\leq \frac{4C_4\sE_{s\sim\nu^*} D_{\text{KL}}(\pi^*||\pi_{w_0})}{(1-\gamma)\sqrt{T}} + \frac{2C_3}{(1-\gamma)m^{1/4}} + \frac{m}{(1-\gamma)\sqrt{T}}  \nonumber\\
	&\quad  + 2(C_2+C_5)\left(\frac{\alpha T}{(1-\gamma)^{2.5}m^{1/8}} \log^{\frac{1}{4}}\left(\frac{(1-\gamma)^2T\sqrt{m}}{\delta}\right) \right) + \frac{4C_6}{\sqrt{T}}.
	\end{flalign*}
	
\end{theorem}

	To proceed the proof of \Cref{thm2}/\Cref{thm4}, we consider the event given in \Cref{lemma_n1}, which happens with probability at least $1-\delta$:
	\begin{flalign}
	&\alpha(1-\gamma)\sum_{t\in\gN_0}(J_0(\pi^*)-J_0(\pi_{w_t})) + \alpha(1-\gamma)\eta\sum_{i=1}^{p}\lone{\gN_i}\nonumber\\
	&\leq \sE_{s\sim\nu^*} D_{\text{KL}}(\pi^*||\pi_{w_0}) + C_3\left(\frac{\alpha T}{m^{1/4}}\right) + C_4(\alpha^2mT) \nonumber\\
	&\quad  + C_5\left(\frac{\alpha T}{(1-\gamma)^{1.5}m^{1/8}} \log^{\frac{1}{4}}\left(\frac{(1-\gamma)^2T\sqrt{m}}{\delta}\right) \right) + C_6\left(\alpha(1-\gamma)\sqrt{T}\right).\label{eq: a13}
	\end{flalign}
	We first consider the convergence rate of the objective function. Under the aforementioned event, we have the following holds:
	\begin{flalign}
	&\alpha(1-\gamma)\sum_{t\in\gN_0}(J_0(\pi^*)-J_0(\pi_{w_t}))\nonumber\\
	&\leq \sE_{s\sim\nu^*} D_{\text{KL}}(\pi^*||\pi_{w_0}) + C_3\left(\frac{\alpha T}{m^{1/4}}\right) + C_4(\alpha^2mT)  \nonumber\\
	&\quad  + C_5\left(\frac{\alpha T}{(1-\gamma)^{1.5}m^{1/8}} \log^{\frac{1}{4}}\left(\frac{(1-\gamma)^2T\sqrt{m}}{\delta}\right) \right) + C_6\left(\alpha(1-\gamma)\sqrt{T}\right).\nonumber
	\end{flalign}
	If $\sum_{t\in\gN_0}(J_0(\pi^*)-J_0(\pi_{w_t}))\leq 0$, then we have $J_0(\pi^*)-J_0(\pi_{w_{\text{out}}})\leq 0$. If $\sum_{t\in\gN_0}(J_0(\pi^*)-J_0(\pi_{w_t}))\geq 0$, we have $\lone{\gN_0}\geq T/2$, which implies the following convergence rate
	\begin{flalign*}
	J_0(\pi^*)-\sE[J_0(\pi_{w_{\text{out}}})]&=\frac{1}{\lone{\gN_0}}\sum_{t\in\gN_0}(J_0(\pi^*)-J_0(\pi_{w_t}))\nonumber\\
	&\leq \frac{2\sE_{s\sim\nu^*} D_{\text{KL}}(\pi^*||\pi_{w_0})}{\alpha(1-\gamma)T} + \frac{2C_3}{(1-\gamma)m^{1/4}} + \frac{2C_4\alpha m}{1-\gamma} \nonumber\\
	&\quad  + \frac{2C_5}{(1-\gamma)^{2.5}m^{1/8}} \log^{\frac{1}{4}}\left(\frac{(1-\gamma)^2T\sqrt{m}}{\delta}\right)  + \frac{2C_6}{\sqrt{T}}.
	\end{flalign*}
	Letting $\alpha = \frac{1}{2C_4\sqrt{T}}$, we can obtain the following convergence rate
	\begin{flalign*}
		J_0(\pi^*)-\sE[J_0(\pi_{w_{\text{out}}})]\leq \frac{C_7m}{(1-\gamma)\sqrt{T}} + \frac{C_8}{(1-\gamma)^{2.5}m^{1/8}} \log^{\frac{1}{4}}\left(\frac{(1-\gamma)^2T\sqrt{m}}{\delta}\right),
	\end{flalign*}
	where
	\begin{flalign*}
		C_7 = \frac{4C_4 D_{\text{KL}}(\pi^*||\pi_{w_0})}{m} + \frac{2(1-\gamma)C_6}{m} + 1,
	\end{flalign*}
    and 
    \begin{flalign*}
    	C_8 = 2C_5 + \frac{2C_3(1-\gamma)^{1.5}}{m^{1/8}}.
    \end{flalign*}
	We then proceed to bound the constraint violation. For any $i\in\{1,\cdots,p\}$, we have
	\begin{flalign*}
	\sE[J_i(\pi_{w_{\text{out}}})]-d_i&=\frac{1}{\lone{\gN_0}}\sum_{t\in\gN_0}J_i(\pi_{w_t}) - d_i\nonumber\\
	&\leq \frac{1}{\lone{\gN_0}}\sum_{t\in\gN_0}(\bar{J}_i(\theta^i_t) - d_i) + \frac{1}{\lone{\gN_0}}\sum_{t\in\gN_0} \lone{J_i(\pi_{w_t})-\bar{J}_i(\theta^i_t)} \nonumber\\
	&\leq \eta + \frac{1}{\lone{\gN_0}}\sum_{t=0}^{T-1} \lone{J_i(\pi_{w_t})-\bar{J}_i(\theta^i_t)} \nonumber\\
	&\leq \eta + \frac{1}{\lone{\gN_0}}\sum_{t=0}^{T-1}\lone{\bar{J}_{i}(\theta^{i}_t) - \sE_{\nu_{\pi_{\tau_{t} W_{t}}}}[f_i((s,a),\bar{\theta}_t)]} \nonumber\\
	&\quad + \frac{C_{RN}}{\gN_0}\sum_{t=0}^{T-1}\lmutauWt{f_i((s,a),\bar{\theta}_t) - Q^i_{\pi_{\tau_{t} W_{t}}}(s,a)}.\nonumber
	\end{flalign*}
	Recalling \cref{eq: a9} and \cref{eq: a10}, under the event defined in \cref{eq: a13}, we have 
	\begin{flalign}
		\sum_{t=0}^{T-1}\lone{\bar{J}_{i}(\theta^{i}_t) - \sE_{\nu_{\pi_{\tau_{t} W_{t}}}}[f_i((s,a),\bar{\theta}_t)]}\leq C_6\sqrt{T},\label{eq: a14}
	\end{flalign}
	and
	\begin{flalign}
	&\sum_{t=0}^{T-1}\lmutauWt{f_i((s,a),\bar{\theta}_t) - Q^i_{\pi_{\tau_{t} W_{t}}}(s,a)}\nonumber\\
	& \leq C_2\left(\frac{T}{(1-\gamma)^{1.5}m^{1/8}} \log^{\frac{1}{4}}\left(\frac{2(1-\gamma)^2T\sqrt{m}}{\delta}\right) \right).\label{eq: a15}
	\end{flalign}
	Let the value of the tolerance $\eta$ be
	\begin{flalign}
	\eta&= \frac{4C_4\sE_{s\sim\nu^*} D_{\text{KL}}(\pi^*||\pi_{w_0})}{(1-\gamma)\sqrt{T}} + \frac{2C_3}{(1-\gamma)m^{1/4}} + \frac{m}{(1-\gamma)\sqrt{T}}  \nonumber\\
	&\quad  + 2C_5\left(\frac{\alpha T}{(1-\gamma)^{2.5}m^{1/8}} \log^{\frac{1}{4}}\left(\frac{(1-\gamma)^2T\sqrt{m}}{\delta}\right) \right) + \frac{2C_6}{\sqrt{T}},\label{eq: a16}
	\end{flalign}
    We have
    \begin{flalign*}
    	\frac{1}{2}\alpha(1-\gamma)\eta T&\geq \sE_{s\sim\nu^*} D_{\text{KL}}(\pi^*||\pi_{w_0}) + C_3\left(\frac{\alpha T}{m^{1/4}}\right) + C_4(\alpha^2mT)  \nonumber\\
    	&\quad  + C_5\left(\frac{\alpha T}{(1-\gamma)^{1.5}m^{1/8}} \log^{\frac{1}{4}}\left(\frac{(1-\gamma)^2T\sqrt{m}}{\delta}\right) \right) + C_6\left(\alpha(1-\gamma)\sqrt{T}\right),
    \end{flalign*}
which satisfies the requirement specified in \Cref{lemma_n2}.
	Combining \cref{eq: a14}, \cref{eq: a15} and \cref{eq: a16}, and using \Cref{lemma_n2}, we have with probability at least $1-\delta$ at least one of the following holds:
	\begin{flalign*}
		\sE[J_i(\pi_{w_{\text{out}}})]-d_i\leq 0,
	\end{flalign*}
or $\lone{\gN_0}\geq T/2$, which further implies
	\begin{flalign*}
	\sE[J_i(\pi_{w_{\text{out}}})]-d_i&\leq \frac{4C_4\sE_{s\sim\nu^*} D_{\text{KL}}(\pi^*||\pi_{w_0})}{(1-\gamma)\sqrt{T}} + \frac{2C_3}{(1-\gamma)m^{1/4}} + \frac{m}{(1-\gamma)\sqrt{T}}  \nonumber\\
	&\quad  + 2(C_2+C_5)\left(\frac{\alpha T}{(1-\gamma)^{2.5}m^{1/8}} \log^{\frac{1}{4}}\left(\frac{(1-\gamma)^2T\sqrt{m}}{\delta}\right) \right) + \frac{4C_6}{\sqrt{T}}.
	\end{flalign*}

\end{document}